\documentclass{article} 

\usepackage[table]{xcolor}
\usepackage{microtype}
\usepackage{graphicx}
\usepackage{subfigure}
\usepackage{booktabs} 

\usepackage{algorithm,algorithmic}

\usepackage{epsfig,amsmath,amssymb,amsfonts,amstext,amsthm,mathrsfs}
\usepackage{latexsym,graphics,epsf,epsfig,psfrag}
\usepackage{dsfont,color,epstopdf,fixmath}
\usepackage{enumitem}
\usepackage{dsfont}

\usepackage[utf8]{inputenc} 
\usepackage[T1]{fontenc}    

\usepackage{booktabs}       
\renewcommand{\arraystretch}{1.4}

\usepackage{booktabs, caption}

\usepackage{amsfonts}       
\usepackage{nicefrac}       
\usepackage{microtype}      
\usepackage{subfigure}
\usepackage{mathtools}
\usepackage{amssymb}
\usepackage{amsthm}

\usepackage[colorlinks=true,citecolor=blue]{hyperref}

\usepackage{multirow}

\newcommand{\tv}{\tilde v}
\newcommand{\tu}{\tilde u}
\newcommand{\tx}{\tilde x}
\newcommand{\ty}{\tilde y}
\newcommand{\tm}{\tilde m}
\newcommand{\wtDelta}{\widetilde \Delta}

\newcommand{\mcm}{\mathcal M}
\newcommand{\mcb}{\mathcal B}

\newcommand{\mE}{\mathbb E}

\newcommand{\mcs}{\mathcal S}

\newcommand{\mco}{\mathcal O}
\newcommand{\mf}{\mathcal F}
\newcommand{\mcv}{\mathcal V}

\newcommand{\ltwo}[1]{\left\|#1\right\|_2}
\newcommand{\nltwo}[1]{\|#1\|_2}
\newcommand{\Bltwo}[1]{\Bigg\|#1\Bigg\|_2}
\newcommand{\lone}[1]{\left|#1\right|}

\DeclareMathOperator*{\argmin}{argmin}

\newcommand{\mR}{\mathbb{R}}

\newtheorem{theorem}{Theorem}

\newtheorem{lemma}{Lemma}

\newtheorem{assumption}{Assumption}

\newtheorem{definition}{Definition}

\allowdisplaybreaks[4]

\usepackage{hyperref}
\usepackage{cleveref}
\usepackage{url}

\usepackage{threeparttable}
\usepackage{hhline}
\usepackage{makecell}

\usepackage[round]{natbib}

\usepackage{apalike}

\usepackage{geometry}
\hypersetup{
	colorlinks,linkcolor=red,anchorcolor=blue,citecolor=blue}
\usepackage{multirow}

\usepackage{authblk}

\makeatletter
\newcommand{\printfnsymbol}[1]{%
	\textsuperscript{\@fnsymbol{#1}}%
}
\makeatother

\begin{document}

\title{Gradient Free Minimax Optimization: Variance Reduction and Faster Convergence \footnote{This is an updated version to replace the previous arXiv post titled "Enhanced First and Zeroth Order Variance Reduced Algorithms for Min-Max Optimization" on 17 Jun, 2020}}

\author[1]{Tengyu Xu, Zhe Wang, Yingbin Liang}
\author[2]{H. Vincent Poor}
\affil[1]{Department of Electrical and Computer Engineering, The Ohio State University}
\affil[2]{Department of Electrical Engineering, Princeton University}
\affil[ ]{\{xu.3260, wang.10982, liang.889\}@osu.edu;\quad poor@princeton.edu}

\date{}
\maketitle

\begin{abstract}
Many important machine learning applications amount to solving minimax optimization problems, and in many cases there is no access to the gradient information, but only the function values. In this paper, we focus on such a gradient-free setting, and consider the nonconvex-strongly-concave minimax stochastic optimization problem. In the literature, various zeroth-order (i.e., gradient-free) minimax methods have been proposed, but none of them achieve the potentially feasible computational complexity of $\mathcal{O}(\epsilon^{-3})$ suggested by the stochastic nonconvex minimization theorem. In this paper, we adopt the variance reduction technique to design a novel zeroth-order variance reduced gradient descent ascent (ZO-VRGDA) algorithm. We show that the ZO-VRGDA algorithm achieves the best known query complexity of $\mathcal{O}(\kappa(d_1 + d_2)\epsilon^{-3})$, which outperforms all previous complexity bound by orders of magnitude, where $d_1$ and $d_2$ denote the dimensions of the optimization variables and $\kappa$ denotes the condition number. In particular, with a new analysis technique that we develop, our result does not rely on a diminishing or accuracy-dependent stepsize usually required in the existing methods. To our best knowledge, this is the first study of zeroth-order minimax optimization with variance reduction. Experimental results on the black-box distributional robust optimization problem demonstrates the advantageous performance of our new algorithm.

\end{abstract}

\section{Introduction}

Minimax optimization has attracted significant growth of attention in machine learning as it captures several important machine learning models and problems including generative adversarial networks (GANs) \cite{goodfellow2014generative}, robust adversarial machine learning \cite{madry2017towards}, 
imitation learning \cite{ho2016generative}, etc. Minimax optimization typically takes the following form
\begin{flalign}\label{eq: 1}
&\min_{x\in \mR^{d_1}}\max_{y\in \mR^{d_2}}f(x,y),\text{where}\, f(x,y)\triangleq \begin{cases} \mE[F(x,y;\xi)]  \quad &\text{(online case)}\\ \textstyle{\frac{1}{n}\sum_{i=1}^{n}F(x,y;\xi_i)} & \text{(finite-sum case)} \end{cases}
\end{flalign}
where $f(x,y)$ takes the expectation form if data samples $\xi$ are taken in an online fashion, and $f(x,y)$ takes the finite-sum form if a dataset of training samples $\xi_i$ for $i=1,\ldots,n$ are given in advance.

This paper focuses on the nonconvex-strongly-concave minimax problem, in which $f(x,y)$ is nonconvex with respect to $x$ for all $y\in \mR^{d_2}$, and $f(x,y)$ is $\mu$-strongly concave with respect to $y$ for all $x\in \mR^{d_1}$. The problem then takes the following equivalent form:
\begin{flalign}\label{eq: 2}
	\min_{x\in \mR^{d_1}}\Big\{ \Phi(x)\triangleq \max_{y\in \mR^{d_2}} f(x,y) \Big\}, 
\end{flalign}
where the objective function $\Phi(\cdot)$ in \cref{eq: 2} is nonconvex in general.

\renewcommand{\arraystretch}{1.5}
\begin{table*}[!t]
	\centering
	\caption{Comparison of gradient-free algorithms for nonconvex-strongly-concave minimax problems}\small
	\begin{threeparttable}
		\begin{tabular}{|c|c|c|c|c|}
			\hline
			 Algorithm & Estimator & Stepsize & Overall Complexity  \\ \hhline{|=====|}
			ZO-min-max \cite{liu2019min} & UniGE & $\mco(\kappa^{-1}\ell^{-1})$ & $\mathcal{O}((d\epsilon^{-6})$  \\ 
			ZO-SGDA \cite{wang2020zeroth} & GauGE & $\mco(\kappa^{-4}\ell^{-1})$ & $\mathcal{O}(d\kappa^5\epsilon^{-4})$  \\ 
			ZO-SGDMSA \cite{wang2020zeroth} & GauGE & $\mco(\kappa^{-1}\ell^{-1})$ & $\mathcal{O}(d\kappa^2\epsilon^{-4}\log( \frac{1}{\epsilon} ))$  \\ 
			\cellcolor{blue!15}{ZO-VRGDA} & \cellcolor{blue!15}{GauGE} & \cellcolor{blue!15}{$\mco(\textcolor{red}{\kappa^{-1}\ell^{-1}})$} & \cellcolor{blue!15}{$\mathcal{O}(\textcolor{red}{d\kappa^{3}\epsilon^{-3}})$} \\ \hline
		\end{tabular}\label{tab:comparison}
	\begin{tablenotes}
		  \item[1] "UniGE" and "GauGE" stand for "Uniform smoothing Gradient Estimator" and "Gaussian smoothing Gradient Estimator", respectively.
		\item[2] The complexity refers to the total number of queries of the function value.
		\item[3] We include only the complexity in the online case in the table, because many previous studies did not consider the finite-sum case. We comment on the finite-sum case in \Cref{sec: 4.2}.
		\item[4] We define $d=d_1+d_2$.
	\end{tablenotes}
	\end{threeparttable}
\end{table*}

In many machine learning scenarios, minimax optimization problems need to be solved without the access of the gradient information, but only the function values, e.g., in multi-agent reinforcement learning with bandit feedback \cite{wei2017online,zhang2019multi} and robotics \cite{wang2017max,bogunovic2018adversarially}. Such scenarios have motivated the design of {\bf gradient-free (i.e., zeroth-order)} algorithms, which solve the problem by querying the function values. For nonconvex-strongly-concave minimax optimization, stochastic gradient descent (SGD) type algorithms have been proposed, which use function values to form gradient estimators in order to iteratively find the solution. In particular, \cite{liu2019min} studied a constrained problem and proposed ZO-min-max algorithm that achieves an $\epsilon$-accurate solution with the function query complexity of $\mathcal{O}((d_1+d_2)\epsilon^{-6})$. \cite{wang2020zeroth} designed ZO-SGDA and ZO-SGDMSA, and between the two algorithms ZO-SGDMA achieves the better function query complexity of $\mathcal{O}((d_1+d_2)\kappa^2\epsilon^{-4}\log(1/\epsilon))$. 

Despite the previous progress, if we view the minimax problem as the nonconvex problem in \cref{eq: 2}, the lower bound on the computational complexity suggests that zeroth-order algorithms may potentially achieve the query complexity of $\mathcal{O}((d_1+d_2)\epsilon^{-3})$. But none of the previous algorithms in the literature achieves such a desirable rate. Thus, a fundamental question to ask here is as follows.
\begin{list}{$\bullet$}{\topsep=0.ex \leftmargin=0.15in \rightmargin=0.in \itemsep =0.01in}
\item {\em Can we design a better gradient-free algorithm that outperforms all existing stochastic algorithms by orders of magnitude, and can achieve the desired query complexity of $\mathcal{O}((d_1+d_2)\epsilon^{-3})$ suggested by the lower bound of gradient-based algorithms?}
\end{list}
This paper provides an affirmative answer to the above question together with the development of novel analysis tools.

\subsection{Main Contributions}

We propose the first zeroth-order variance reduced gradient descent ascent (ZO-VRGDA) algorithm for minimax optimization. ZO-VRGDA features gradient-free designs and adopts a nested-loop structure with the recursive variance reduction method incorporated for both the inner- and outer-loop updates. In particular, the outer loop adopts zeroth-order coordinate-wise estimators for accurate gradient estimation, and the inner loop adopts zeroth-order Gaussian smooth estimators for efficient gradient estimation. This is the first gradient-free variance reduced algorithm designed for minimax optimization.


We establish the convergence rate and the function query complexity for ZO-VRGDA for nonconvex-strongly-conconve minimax optimization. For the online case, we show that ZO-VRGDA achieves the best known query complexity of $\mathcal{O}((d_1+d_2)\kappa^{3}\epsilon^{-3})$, which outperforms the existing state-of-the-art (achieved by ZO-SGDMSA \cite{wang2020zeroth}) in the case with $\epsilon\leq \kappa^{-1}$. For the finite-sum case, we show that ZO-VRGDA achieves an overall query complexity of $\mco( (d_1+d_2)(\kappa^{2}\sqrt{n}\epsilon^{-2}+n) + d_2(\kappa^2+\kappa n) \log(\kappa))$ when $n\geq \kappa^2$, and $\mco( (d_1+d_2)(\kappa^2 + \kappa n)\kappa\epsilon^{-2} )$ when $n\leq \kappa^2$. Our work provides the first convergence analysis for gradient-free variance reduced algorithms for minimax optimization.

It is also instructive to compare our result with a concurrent work \cite{huang2020accelerated}, which proposed an accelerated zeroth-order momentum descent ascent (Acc-ZOMDA) method for minimax optimization. The performance difference between our ZO-VRGDA and their Acc-ZOMDA are two folds. (a) The query complexity of our ZO-VRGDA outperforms that of Acc-ZOMDA by a factor of $ (d_1+d_2)^{1/2}$, which can be significant in large dimensional problems such as the neural network training. (b) Rigorously speaking, our result characterizes the exact convergence to an $\epsilon$-accurate stationary point, whereas the convergence metric in \cite{huang2020accelerated} does not necessarily imply convergence to a stationary point. 

From the technical standpoint, differently from the previous approach (e.g., \cite{luo2020stochastic}), we develop a new analysis framework for analyzing the recursive variance reduced algorithms for minimax problems. Specifically, the main challenge for our analysis lies in bounding two inter-connected stochastic error processes: tracking error and gradient estimation error. The previous analysis forces those two error terms to be kept at $\epsilon$-level at the cost of inefficient initialization and $\epsilon$-level small stepsize. In contrast, we develop new tools to capture the coupling of the accumulative estimation error and tracking error over the entire algorithm execution, and then establish their relationships with the accumulative gradient estimators to derive the overall convergence bound. As a result, our ZO-VRGDA can adopt a more relaxed initialization and a large constant stepsize for fast running speed, and still enjoy the theoretical convergence guarantee.

\subsection{Related Work}

Due to the vast amount of studies on minimax optimization and the variance reduced algorithms, we include below only the studies that are highly relevant to this work.

Variance reduction methods for minimax optimization are inspired by those for conventional minimization problems, including SAGA \cite{defazio2014saga,reddi2016b}, SVRG \cite{johnson2013accelerating,allen2016variance,allen2017natasha}, SARAH \cite{nguyen2017sarah,nguyen2017stochastic,nguyen2018inexact}, SPIDER \cite{fang2018spider}, SpiderBoost \cite{wang2019spiderboost}, etc. But the convergence analysis for minimax optimization is much more challenging, and is typically quite different from their counterparts in minimization problems.

For {\em strongly-convex-strongly-concave minimax optimization}, \cite{palaniappan2016stochastic} applied SVRG and SAGA to the finite-sum case and established a linear convergence rate, and \cite{chavdarova2019reducing} proposed SVRE later to obtain a better bound. When the condition number of the problem is very large, \cite{luo2019stochastic} proposed a proximal point iteration algorithm to improve the performance of SAGA. For some special cases, \cite{du2017stochastic,du2019linear} showed that the linear convergence rate of SVRG can be maintained without the strongly-convex or strongly concave assumption. \cite{yang2020global} applied SVRG to study the minimax optimization under the two-sided Polyak-Lojasiewicz condition.


{\em Nonconvex-strongly-concave minimax optimization} is the focus of this paper. As we discuss at the beginning of the introduction, the SGD-type algorithms have been developed and studied, including SGDmax \cite{jin2019minmax}, PGSMD \cite{rafique2018non}, and SGDA \cite{lin2019gradient}. 
Several variance reduction methods have also been proposed to further improve the performance, including PGSVRG \cite{rafique2018non},
the SAGA-type algorithm for minimax optimization \cite{wai2019variance}, 
and SREDA \cite{luo2020stochastic}. Particularly, SREDA has been shown in \cite{luo2020stochastic} to achieve the optimal complexity dependence on $\epsilon$.

While SGD-type zeroth-order algorithms have been studied for minimax optimization, such as \cite{menickelly2020derivative,roy2019online} for convex-concave minimax problems and \cite{liu2019min,wang2020zeroth} for nonconvex-strongly-concave minimax problems, variance reduced algorithms have not been developed for {\em zeroth-order minimax optimization} so far. This paper proposes the first such an algorithm named ZO-VRGDA for nonconvex-strongly-concave minimax optimization, and established its complexity performance that outperforms the existing comparable algorithms 
(see \Cref{tab:comparison}).

\subsection{Notations}

In this paper, we use $\ltwo{\cdot}$ to denote the Euclidean norm of vectors. For a finite set $\mathcal{S}$, we denote its cardinality as $\lone{\mathcal{S}}$. For a positive integer $n$, we denote $[n]=\{1,\cdots,n\}$.

\section{Preliminaries}

We first introduce the gradient estimator that we use to design our gradient-free algorithm, and then describe the technical assumptions that we take in our analysis.

\subsection{Zeroth-order Gradient Estimator}\label{sec:gdest}

We consider the Gaussian smoothed function \cite{nesterov2017random,ghadimi2013stochastic} defined as:
\begin{flalign*}
	f_{\mu_1}(x,y)&\coloneqq \mE_{\nu,\xi}F(x+\mu_1 \nu,y,\xi),\\
	f_{\mu_2}(x,y)&\coloneqq \mE_{\omega,\xi}F(x,y+\mu_2 \omega,\xi),
\end{flalign*}
where $\nu_i\sim N(0,\mathbf{1}_{d_1})$, $\omega_i\sim N(0,\mathbf{1}_{d_2})$ with $\mathbf{1}_{d}$ denoting the identity matrices with sizes $d\times d$.
Then, in order to approximate the gradient of $f_{\mu_1}(x,y)$ and $f_{\mu_2}(x,y)$ with respect to $x$ and $y$ based on the function values, the zeroth-order stochastic gradient estimators can be constructed as
\begin{flalign}
&G_{\mu_1}(x,y,\nu_{\mcm_1},\xi_{\mcm})=\frac{1}{\lone{\mcm}}\sum_{i\in [\lone{\mcm}]} \frac{F(x+\mu_1 \nu_i,y,\xi_i)-F(x,y,\xi_i)}{\mu_1}\nu_i,\label{eq: 4}\\
&H_{\mu_2}(x,y,\omega_{\mcm_2},\xi_{\mcm})=\frac{1}{\lone{\mcm}}\sum_{i\in [\lone{\mcm}]} \frac{F(x,y+\mu_2 \omega_i,\xi_i)-F(x,y,\xi_i)}{\mu_2}\omega_i,\label{eq: 5}
\end{flalign}
 where $\lone{\mcm}=\lone{\mcm_1}=\lone{\mcm_2}$ denote the batchsize of samples. It can be shown that $G_{\mu_1}(x,y,\nu_{\mcm_1},\xi_{\mcm})$ and $H_{\mu_2}(x,y,\omega_{\mcm_2},\xi_{\mcm})$ are unbiased estimators of the true gradient of $f_{\mu_1}(x,y)$ and $f_{\mu_2}(x,y)$ with respect to $x$ and $y$ \cite{ghadimi2013stochastic}, respectively, i.e.,
\begin{flalign*}
\mE_{\nu_{\mcm_1},\xi_{\mcm}}G_{\mu_1}(x,y,\nu_{\mcm_1},\xi_{\mcm})&=\nabla_x f_{\mu_1}(x,y),\\ \mE_{\omega_{\mcm_2},\xi_{\mcm}}H_{\mu_2}(x,y,\omega_{\mcm_2},\xi_{\mcm})&=\nabla_y f_{\mu_2}(x,y).
\end{flalign*}
These zeroth-order gradient estimators are useful for us to design a gradient-free algorithm for minimax optimization.

\subsection{Technical Assumptions}

We take the following standard assumptions for the minimax problem in \cref{eq: 1} or \cref{eq: 2}, which have also been adopted in \cite{liu2019min,wang2020zeroth,huang2020accelerated,luo2020stochastic,lin2019gradient}. We slightly abuse the notation $\xi$ below to represent the random index in both the online and finite-sum cases, where in the finite-sum case, $\mE_{\xi}[\cdot]$ is with respect to the uniform distribution over $\{\xi_1,\cdots,\xi_n\}$. 

\begin{assumption}\label{ass1}
	The function $\Phi(\cdot)$ is lower bounded, i.e., we have $\Phi^*=\inf_{x\in \mR^{d_1}}\Phi(x) > -\infty$.
\end{assumption}
\begin{assumption}\label{ass2}
	The component function $F$ has an averaged $\ell$-Lipschitz gradient, i.e., for all $(x,y)$, $(x^\prime,y^\prime)\in \mR^{d_1}\times \mR^{d_2}$, we have 
	$\mE_\xi\big[\ltwo{\nabla F(x,y;\xi)-\nabla F(x^\prime,y^\prime;\xi)}^2\big]\leq \ell^2(\ltwo{x-x^\prime}^2+\ltwo{y-y^\prime}^2)$.
\end{assumption}
\begin{assumption}\label{ass3}
	The function $f$ is $\mu$-strongly-concave in $y$ for any $x\in\mR^{d_1}$, and the component function $F$ is concave in $y$, i.e., for any $x\in \mR^{d_1}$, $y,y^\prime\in \mR^{d_2}$ and $\xi$, we have
		\begin{flalign*}
		f(x,y)\leq f(x,y^\prime) + \langle \nabla_y f(x,y^\prime), y-y^\prime \rangle - \frac{\mu}{2}\ltwo{y-y^\prime},
		\end{flalign*}
		and
		\begin{flalign*}
		F(x,y;\xi)\leq F(x,y^\prime;\xi) + \big\langle \nabla_y F(x,y^\prime;\xi), y-y^\prime \big\rangle.
		\end{flalign*}
\end{assumption}

\begin{assumption}\label{ass5}
	The gradient of each component function $F(x,y;\xi)$ has a bounded variance, i.e., there exists a constant $\sigma>0$ such that for any $(x,y)\in \mR^{d_1 \times d_2}$, we have
	\begin{flalign*}
		\mE_\xi\big[\ltwo{\nabla F(x,y;\xi)-\nabla f(x,y)}^2\big]\leq \sigma^2.
	\end{flalign*}
\end{assumption}

Note that the above variance assumption is weaker than that of Acc-ZOMDA in \cite{huang2020accelerated}, because \cite{huang2020accelerated} directly requires the variance of the zeroth-order estimator to be bounded, which is not easy to verify. In contrast, we require such a condition to hold only for the original stochastic gradient estimator, which is standard in the optimization literature and can be satisfied easily in practice.


We define $\kappa\triangleq\ell/\mu$ as the condition number of the problem throughout the paper. The following structural lemma developed in \cite{lin2019gradient} provides further information about $\Phi$ for nonconvex-strongly-concave minimax optimization.
\begin{lemma}[Lemma 3.3 of \cite{lin2019gradient}]\label{lemma1}
	Under Assumption \ref{ass2} and \ref{ass3}, the function $\Phi(\cdot)=\max_{y\in \mR^{d_2}}f(\cdot,y)$ is $(\kappa+1)\ell$-gradient Lipschitz and $\nabla \Phi(x)=\nabla_x f(x, y^*(x))$ is $\kappa$-Lipschitz, where $y^*(\cdot)=\argmin_{y\in\mR^{d_2}}f(\cdot,y)$.
\end{lemma}
We let $L\triangleq (1+\kappa)\ell$ denote the Lipschitz constant of $\nabla\Phi(x)$. 
Since $\Phi$ is nonconvex in general, it is NP-hard to find its global minimum. Our goal here is to develop a gradient-free zeroth-order stochastic gradient algorithms that output an $\epsilon$-stationary point as defined below.
\begin{definition}
	The point $\bar{x}$ is called an $\epsilon$-stationary point of the differentiable function $\Phi$ if $\ltwo{\nabla \Phi(\bar{x})}\leq \epsilon$, where $\epsilon$ is a positive constant.
\end{definition}

\section{ZO-VRGDA: Zeroth-Order Variance Reduction Algorithm}\label{sc: zosreda}

In this section, we propose a new zeroth-order variance reduced gradient descent ascent (ZO-VRGDA) algorithm to solve the minimax problem in \cref{eq: 1} or \cref{eq: 2}. ZO-VRGDA (see \Cref{al:zosredaboost}) adopts a nested-loop structure, in which the parameters $x_t$ and $y_t$ are updated in a nested loop fashion: each update of $x_t$ in the outer-loop is followed by $(m+1)$ updates of $y_t$ over one entire inner loop. ZO-VRGDA incorporates the variance reduction method for both the inner-loop and outer-loop updates, and features gradient-free designs. We next describe the ZO-VRGDA algorithm in more detail as follows.

\begin{algorithm}[tb]
	\null
	\caption{ZO-VRGDA}
	\label{al:zosredaboost}
	\small
	\begin{algorithmic}[1]
		\STATE \textbf{Input:} $x_0$, initial accuracy $\zeta$, learning rate $\alpha=\Theta(\frac{1}{\kappa \ell})$, $\beta=\Theta(\frac{1}{\ell})$, batch size $\mcs_1$, $\mcs_2$ and periods $q,m$.
		\STATE \textbf{Initialization:} $y_0=\text{ZO-iSARAH}(-f(x_0,\cdot),\zeta)$ (detailed in \Cref{al:zoisarah})
		\FOR{$t=0, 1, ..., T-1$}
		\IF {$\text{mod}(k,q)=0$} 
		\STATE draw $S_1$ samples $\{ \xi_1,\cdots,\xi_{S_1} \}$
		\STATE $v_t=\frac{1}{S_1}\sum_{i=1}^{S_1}\sum_{j=1}^{d_1}\frac{F(x_t+\delta e_j, y_t, \xi_i) - F(x_t-\delta e_j, y_t, \xi_i)}{2\delta} e_j$
		\STATE $u_t=\frac{1}{S_1}\sum_{i=1}^{S_1}\sum_{j=1}^{d_2}\frac{F(x_t, y_t+\delta e_j, \xi_i) - F(x_t, y_t-\delta e_j, \xi_i)}{2\delta} e_j$
		\STATE where $e_j$ denotes the vector with $j$-th natural unit basis vector.
		\ELSE
		\STATE $v_t=\tilde{v}_{t-1,\bar{m}_{t-1}}$, $u_t=\tilde{u}_{t-1,\bar{m}_{t-1}}$
		\ENDIF
		\STATE $x_{t+1}=x_t - \alpha v_t$
		\STATE $y_{t+1}=\text{ZO-ConcaveMaximizer}(t,m,S_{2,x},S_{2,y})$ (detailed in \Cref{al:zoconcavemaximizer})
		\ENDFOR
		\STATE \textbf{Output:} $\hat{x}$ chosen uniformly at random from $\{ x_t\}_{t=0}^{T-1}$
	\end{algorithmic}
\end{algorithm}

\begin{algorithm}[tb]
	\null
	\caption{ZO-iSARAH}
	\label{al:zoisarah}
	\small
	\begin{algorithmic}[1]
		\STATE \textbf{Input:} $\tilde{w}_0$, learning rate $\gamma>0$, inner loop size $I$, batch size $B_1$ and $B_2$
		\FOR{$t=1, 2, ..., T$}
		\STATE $w_0=\tilde{w}_{t-1}$
		\STATE draw $B_1$ samples $\{ \xi_1,\cdots,\xi_{B_1} \}$
		\STATE $v_0=\frac{1}{B_1}\sum_{i=1}^{B_1}\sum_{j=1}^{d}\frac{P(w_0+\delta e_j, \xi_i) - P(w_0-\delta e_j, \xi_i)}{2\delta} e_j$
		\STATE where $e_j$ denotes the vector with $j$-th natural unit basis vector.
		\STATE $w_1=w_0 + \gamma v_0$
		\FOR{$k=1, 2, ..., I-1$}
		\STATE Draw minibatch sample $\mcm=\{ \xi_1,\cdots,\xi_{B_2}\}$ and $\mcm_1=\{\psi_{1},\cdots, \psi_{B_2}\}$
		\STATE $v_k=v_{k-1} + \Psi_\tau(w_k,\psi_{\mcm_1},\xi_{\mcm}) - \Psi_\tau(w_{k-1},\psi_{\mcm_1},\xi_{\mcm})$
		\STATE $w_{k+1}=w_k - \gamma v_k$
		\ENDFOR
		\STATE $\tilde{w}_t$ chosen uniformly at random from $\{ w_k\}_{k=0}^{I}$
		\ENDFOR
	\end{algorithmic}
\end{algorithm}

\begin{algorithm}[tb]
	\null
	\caption{$\text{ZO-ConcaveMaximizer}(t,m,S_{2,x},S_{2,y})$}
	\label{al:zoconcavemaximizer}
	\small
	\begin{algorithmic}[1]
		\STATE \textbf{Initialization:} $\tx_{t,-1}=x_t$, $\ty_{t,-1}=y_t$, $\tx_{t,0}=x_{t+1}$, $\ty_{t,0}=y_t$, $\tv_{t,-1}=v_t$, $\tu_{t,-1}=u_t$
		\STATE Draw minibatch sample $\mcm_x=\{ \xi_1,\cdots,\xi_{S_{2,x}}\}$, $\mcm_{1,x}=\{\nu_{1},\cdots, \nu_{S_{2,x}}\}$ and $\mcm_{2,x}=\{\omega_{1},\cdots,\omega_{S_{2,x}} \}$, and $\mcm_y=\{ \xi_1,\cdots,\xi_{S_{2,y}}\}$, $\mcm_{1,x}=\{\nu_{1},\cdots, \nu_{S_{2,y}}\}$ and $\mcm_{2,y}=\{\omega_{1},\cdots,\omega_{S_{2,y}} \}$
		\STATE $\tv_{t,0}=\tv_{t,-1} + G(\tx_{t,0},\ty_{t,0},\nu_{\mcm_{1,x}},\xi_{\mcm_x}) - G(\tx_{t,-1},\ty_{t,-1},\nu_{\mcm_{1,x}},\xi_{\mcm_x})$
		\STATE $\tu_{t,0}=\tu_{t,-1} + H(\tx_{t,0},\ty_{t,0},\omega_{\mcm_{2,y}},\xi_{\mcm_y}) - H(\tx_{t,-1},\ty_{t,-1},\omega_{\mcm_{2,y}},\xi_{\mcm_y})$
		\STATE $\tx_{t,1}=\tx_{t,0}$
		\STATE $\ty_{t,1}=\ty_{t,0} + \beta \tu_{t,0}$
		\FOR{$k=1, 2, ..., m+1$}
		\STATE Draw minibatch sample $\mcm_x=\{ \xi_1,\cdots,\xi_{S_{2,x}}\}$, $\mcm_{1,x}=\{\nu_{1},\cdots, \nu_{S_{2,x}}\}$ and $\mcm_{2,x}=\{\omega_{1},\cdots,\omega_{S_{2,x}} \}$, and $\mcm_y=\{ \xi_1,\cdots,\xi_{S_{2,y}}\}$, $\mcm_{1,y}=\{\nu_{1},\cdots, \nu_{S_{2,y}}\}$ and $\mcm_{2,y}=\{\omega_{1},\cdots,\omega_{S_{2,y}} \}$
		\STATE $\tv_{t,k}=\tv_{t,k-1} + G_{\mu_1}(\tx_{t,k},\ty_{t,k},\nu_{\mcm_{1,x}},\xi_{\mcm_x}) - G_{\mu_1}(\tx_{t,k-1},\ty_{t,k-1},\nu_{\mcm_{1,x}},\xi_{\mcm_x})$
		\STATE $\tu_{t,k}=\tu_{t,k-1} + H_{\mu_2}(\tx_{t,k},\ty_{t,k},\omega_{\mcm_{2,y}},\xi_{\mcm_y}) - H_{\mu_2}(\tx_{t,k-1},\ty_{t,k-1},\omega_{\mcm_{2,y}},\xi_{\mcm_y})$
		\STATE $\tx_{t,k+1}=\tx_{t,k}$
		\STATE $\ty_{t,k+1}=\ty_{t,k} + \beta \tu_{t,k}$
		\ENDFOR
		\OUTPUT $y_{t+1}=\ty_{t,\tm_t}$ with $\tm_t$ chosen uniformly at random from $\{ 0,1,\cdots, m\}$
	\end{algorithmic}
\end{algorithm}

(a) The initialization of ZO-VRGDA (line 2 of \Cref{al:zosredaboost}) utilizes a zeroth-order algorithm ZO-iSARAH (see \Cref{al:zoisarah}), which adopts a first-order algorithm iSARAH and incorporates the zeroth-order gradient estimators, to search an initialization $y_0$ with predefined accuracy $\mE[\ltwo{\nabla_yf(x_0,y_0)}^2]\leq\zeta$. In particular, ZO-iSARAH uses a small batch of sampled function values to construct Gaussian estimators for approximating gradients (line 10 of \Cref{al:zoisarah}), which is defined as
\vspace{1cm}
\begin{flalign}
\Psi_{\tau}(w,\psi_{\mcm_1},\xi_{\mcm})=\frac{1}{\lone{\mcm}}\sum_{i\in [\lone{\mcm}]} \frac{P(w+\tau \psi_i,\xi_i)-P(w,\xi_i)}{\tau}\psi_i,\label{eq: sa2}
\end{flalign}
where $\psi_i\sim N(0,\mathbf{1}_{d})$.

(b) The outer-loop updates of $x_t$ is divided into epochs for variance reduction. Consider a certain outer-loop epoch $t=\{(n_t-1)q,\cdots,n_tq-1\}$ ($1\leq n_t<\lceil T/q \rceil$ is a positive integer). At the beginning of such an epoch, ZO-VRGDA utilizes a large batch $S_1$ of the sampled function values to construct gradient-free coordinate-wise estimators for gradient $\nabla_xf(x,y)$ and $\nabla_y f(x,y)$ (see lines 6 and 7 in \Cref{al:zosredaboost}).
Note that the coordinate-wise gradient estimator is commonly taken in the zeroth-order variance reduced algorithms such as in \cite{ji2019improved,fang2018spider} for minimization problems. The batch size $S_1$ is set to be large so that gradient estimators that recursively updated in each epoch can build on an accurate estimators ($v_t$ and $u_t$). In this way, the estimators recursively updated over the entire epoch will not deviate too much from the exact gradients.

(c) For each outer-loop iteration, an inner loop of ZO-ConcaveMaximizer (see \Cref{al:zoconcavemaximizer}) (line 13 of ZO-VRGDA) uses the small batch $S_{2,x}$ and $S_{2,y}$ of sampled function values to construct a variance reduced estimators for $\nabla_xf_{\mu_1}(x,y)$ and $\nabla_y f_{\mu_2}(x,y)$, respectively, as follows
\begin{align*}
\tv_{t,k}&=\tv_{t,k-1} + G_{\mu_1}(\tx_{t,k},\ty_{t,k},\nu_{\mcm_{1,x}},\xi_{\mcm_x}) - G_{\mu_1}(\tx_{t,k-1},\ty_{t,k-1},\nu_{\mcm_{1,x}},\xi_{\mcm_x})\\
\tu_{t,k}&=\tu_{t,k-1}+ H_{\mu_2}(\tx_{t,k},\ty_{t,k},\omega_{\mcm_{2,y}},\xi_{\mcm_y})- H_{\mu_2}(\tx_{t,k-1},\ty_{t,k-1},\omega_{\mcm_{2,y}},\xi_{\mcm_y}).
\end{align*}
where the estimators $G_{\mu}(\cdot)$ and $H_{\mu}(\cdot)$ are defined in \Cref{sec:gdest}.
These zeroth-order gradient estimators are then recursively updated through the inner loop. The batch size $S_2$ is set at the same scale as epoch length $q$, so that the accumulated error of the recursively updated estimators $\tv_{t,k}$ and $\tu_{t,k}$ can be kept at a relatively low level.

In addition to the above major gradient-free designs, ZO-VRGDA also features the following enhancements over its first-order counterpart SREDA \cite{luo2020stochastic}. (a) ZO-VRGDA relaxes the initialization requirement to be $\mE[\ltwo{\nabla_y f(x_0,y_0)}^2]\leq\kappa^{-1}$, which requires only $\mathcal{O}(\kappa\log\kappa)$ gradient estimations. This improves the computational cost by a factor of $\mathcal{\tilde{O}}(\kappa\epsilon^{-2})$. (b) ZO-VRGDA adopts a much larger and $\epsilon$-{\bf in}dependent stepsize $\alpha_t=\alpha =\mathcal{O}(1/(\kappa\ell))$ for $x_t$ so that each outer-loop update can make much bigger progress. 

\section{Convergence Analysis of ZO-VRGDA} \label{sec: 4.2}

In this section, we first present our convergence results for ZO-VRGDA and then provide a proof sketch for our analysis.

\subsection{Main Results}



In order to analyze the convergence of ZO-VRGDA, we first provide the complexity analysis for the initialization algorithm ZO-iSARAH. Since the initialization is applied to the variable $y$, with respect to which the objective function is strongly concave. Hence, the initialization is equivalent to the following standard optimization problem:
\begin{flalign}
\min_{w\in\mR^d}p(w)\triangleq \mE[P(w;\xi)], \label{eq: sa1}
\end{flalign}
where $P$ is average $\ell$-gradient Lipschitz and convex, $p$ is $\mu$-strongly convex, and $\xi$ is a random vector. 

It turns out that the convergence of the zeroth-order recursive variance reduced algorithm ZO-iSARAH has not been studied before for strongly convex optimization. We thus provide the first complexity result for ZO-iSARAH to solve the problem in \cref{eq: sa1} as follows.
\begin{theorem}\label{thm: zo-sarah}
Apply ZO-iSARAH in \Cref{al:zoisarah} to solve the strongly convex optimization problem in \cref{eq: sa1}. Set $\gamma=\Theta(1/\ell)$, $B_1=\Theta(1/\epsilon)$, $B_2=d$, $I=\Theta(\kappa)$, $T=\Theta(\log(1/\epsilon))$, $\delta=\Theta(\epsilon^{0.5}/\ell d^{0.6})$, and $\tau=\min\{ \frac{\epsilon^{0.5}}{3\ell(d+3)^{1.5}}, \sqrt{\frac{2\epsilon}{5\ell\mu d}} \}$. Then, the output of \Cref{al:zoisarah} satisfies
	\begin{flalign*}
	\mE[\ltwo{\nabla p_{\tau}(\tilde{w}_T)}^2]\leq \epsilon,
	\end{flalign*}
	with the total function query complexity given by 
	\begin{flalign*}
	T\cdot (I \cdot B_2 + d\cdot B_1) = \mathcal{O}\left( d  \left(\kappa + \frac{1}{\epsilon}\right) \log\left( \frac{1}{\epsilon} \right) \right).
	\end{flalign*}
\end{theorem}
Since we require the initialization accuracy in \Cref{al:zosredaboost} to be $\kappa^{-1}$, \Cref{thm: zo-sarah} indicates that the total function query complexity of performing ZO-iSARAH in \Cref{al:zosredaboost} is $\mathcal{O}(d_2\kappa\log(1/\kappa))$. Ignoring the dependence on the dimension caused by zeroth-order estimator, our initialization complexity improves upon its first-order counterpart SREDA \cite{luo2020stochastic} by a factor of $\mathcal{\tilde{O}}(\kappa\epsilon^{-2})$.

We next provide our main theorem as follows, which characterizes the query complexity of ZO-VRGDA for finding a first-order stationary point of $\Phi(\cdot)$ with $\epsilon$ accuracy.
\begin{theorem}\label{thm2}
	Apply ZO-VRGDA in \Cref{al:zosredaboost} to solve the online case of the problem \cref{eq: 1}. Suppose Assumptions \ref{ass1}-\ref{ass5} hold. Consider the following hyperparamter setting:
	$\zeta=\kappa^{-1}$, $\alpha=\mathcal{O}(\kappa^{-1}\ell^{-1})$, $\beta=\mathcal{O}(\ell^{-1})$, $q=\mathcal{O}(\epsilon^{-1})$, $m=\mathcal{O}(\kappa)$, $S_1=\mathcal{O}(\sigma^2\kappa^2\epsilon^{-2})$, $S_{2,x}=\mathcal{O}(d_1\kappa\epsilon^{-1})$, $S_{2,y}=\mathcal{O}(d_2\kappa\epsilon^{-1})$, $\delta=\mathcal{O}((d_1+d_2)^{0.5}\kappa^{-1}\ell^{-1}\epsilon)$, $\mu_1=\mathcal{O}(d_1^{-1.5}\kappa^{-2.5}\ell^{-1}\epsilon)$ and $\mu_2=\mathcal{O}(d_2^{-1.5}\kappa^{-2.5}\ell^{-1}\epsilon)$. Then for $T$ to be at least at the order of $\mathcal{O}(\kappa\epsilon^{-2})$, 
	\Cref{al:zosredaboost} outputs $\hat{x}$ such that
	$$
	\mE[\ltwo{\nabla\Phi(\hat{x})}]\leq \epsilon,
	$$
	with the overall function query complexity given by
	\begin{flalign}
	&T\cdot (S_{2,x} + S_{2,y}) \cdot m + \left\lceil\frac{T}{q}\right\rceil\cdot S_1 \cdot (d_1+d_2) + T_0 \nonumber\\
	&= \mco\left(\frac{\kappa}{\epsilon^2} \cdot \frac{(d_1+d_2)\kappa}{\epsilon} \cdot \kappa \right) + \mco\left( \frac{\kappa}{\epsilon} \cdot \frac{\kappa^2}{\epsilon^2} \cdot (d_1+d_2) \right) + \mco\left( d_2 \kappa \log(\kappa) \right)\nonumber\\
	&=\mco\left( (d_1+d_2)\kappa^{3}\epsilon^{-3}\right).
	\end{flalign}
\end{theorem}
Furthermore, ZO-VRGDA can also be applied to the finite-sum case of the problem \cref{eq: 1}, by replacing the large batch sample $S_1$ used in line 6 of \Cref{al:zosredaboost} with the full set of samples. Then the following result characterizes the query complexity in such a case.
\begin{theorem}\label{cor: 0thfinitesum}
	Apply ZO-VRGDA described above to solve the finite-sum case of the problem \cref{eq: 1}. Suppose Assumptions \ref{ass1}-\ref{ass5} hold. Under appropriate parameter settings given in \Cref{sc: 0stfinitesum}, the function query complexity to attain an $\epsilon$-stationary point is $\mco( (d_1+d_2)(\sqrt{n}\kappa^{2}\epsilon^{-2} + n) + d_2(\kappa^2 + \kappa n) \log(\kappa))$ for $n\geq \kappa^2$, and $\mco( (d_1+d_2)(\kappa^2 + \kappa n)\epsilon^{-2} )$ for $n\leq \kappa^2$.
\end{theorem}

\Cref{thm2} and \Cref{cor: 0thfinitesum} indicate that the query complexity of ZO-VRGDA matches the optimal dependence on $\epsilon$ of the first-order algorithm for nonconvex optimization in \cite{fang2018spider}. The dependence on $d_1$ and $d_2$ typically arises in zeroth-order algorithms due to the estimation of gradients with dimensions $d_1$ and $d_2$.
Furthermore, in the online case, ZO-VRGDA outperforms the  best known query complexity dependence on $\epsilon$ among the existing zeroth-order algorithms by a factor of $\mathcal{O}(1/\epsilon)$. Including the conditional number $\kappa$ into consideration, ZO-VRGDA outperforms the best known query complexity
achieved by ZO-SGDMA in the case with $\epsilon\leq \kappa^{-1}$ (see \Cref{tab:comparison}). 

\Cref{thm2} and \Cref{cor: 0thfinitesum} provide the first convergence analysis and the query complexity for the zeroth-order variance-reduced algorithms for minimax optimization. 
Furthermore, \Cref{cor: 0thfinitesum} provides the first query complexity for the finite-sum zeroth-order minimax problems.


\subsection{Outline of Technical Proof}

Our analysis has the following two major novel developments. (a) We develop new tools to analyze the zeorth-order estimator for variance reduced minimax algorithms. (b) More importantly, differently from the previous approach (e.g., \cite{luo2020stochastic}), we develop a new analysis framework for analyzing the recursive variance reduced algorithms for minimax problems.
At a high level, the previous analysis mainly focuses on bounding two inter-related errors: {\bf tracking error} $\delta_t=\mE[\ltwo{\nabla_y f(x_t,y_t)}^2]$ that captures how well $y_t$ approximates the optimal point $y^*(x_t)$ for a given $x_t$, and {\bf gradient estimation error} $\Delta_t=\mE[\ltwo{v_t-\nabla_x f(x_t,y_t)}^2 + \ltwo{u_t-\nabla_y f(x_t,y_t)}^2]$ that captures how well the stochastic gradient estimators approximate the true gradients. In the previous analysis, those two error terms are forced to be at $\epsilon$-level at the cost of inefficient initialization and $\epsilon$-level stepsize.
In contrast, we develop tools to capture the coupling of the accumulative estimation and tracking errors over the entire algorithm execution, and then establish their relationships with the accumulative gradient estimators to derive the overall convergence bound. As a result, our ZO-VRGDA can adopt a more relaxed initialization and a large constant stepsize for fast running speed, and still enjoy the theoretical convergence guarantee.

\vspace{-2mm}
\begin{proof}[\bf Proof Sketch of \Cref{thm2}]
The proof of \Cref{thm2} consists of the following three steps. 

\noindent\textbf{Step 1:} We start from the estimation error $\Delta^\prime_t$ and tracking error $\delta^\prime_t$ defined with respect to the Gaussian smooth objective functions: $\Delta^\prime_t=\mE[\ltwo{\nabla_x f_{\mu_1}(x_t,y_t)-v_t}^2] + \mE[\ltwo{\nabla_y f_{\mu_2}(x_t,y_t)-u_t}^2]$ and $\delta^\prime_t=\mE[\ltwo{\nabla_y f_{\mu_2}(x_t,y_t)}^2]$. 
which is connected with $\Delta_t$ and $\delta_t$ via the following inequalities:
 \begin{flalign*}
 	\Delta_t &\leq 2\Delta^\prime_t + \frac{\mu^2_1}{2}\ell^2(d_1+3)^3 + \frac{\mu^2_2}{2}\ell^2(d_2+3)^3,\nonumber\\
 	\delta_t &\leq 2\delta^\prime_t + \frac{\mu^2_2}{2}\ell^2(d_2+3)^3.
\end{flalign*}
We establish the relationship between $\Delta^\prime_{t}$ and $\Delta^\prime_{t-1}$ as well as that between $\delta^\prime_{t}$ and $\delta^\prime_{t-1}$ as follows
\begin{flalign}
	\Delta^\prime_t &\leq \left(1 + \Theta(\epsilon)\right)\Delta^\prime_{t-1} + \Theta(\epsilon)\delta^\prime_{t-1} + \Theta(\kappa^{-2}\epsilon)\mE[\ltwo{v_{t-1}}^2]  +  \Theta(\kappa^{-2}\epsilon^2),\label{eq: pf1}\\
	\delta^\prime_t&\leq \frac{1}{2}\delta^\prime_{t-1} + \Theta(1)\Delta^\prime_{t-1} +  \Theta(\kappa^{-2})\mE[\ltwo{v_{t-1}}^2] +  \Theta(\kappa^{-5}\epsilon^2).\label{eq: pf2}
\end{flalign}

\noindent\textbf{Step 2:} {\bf Step 1} indicates that $\Delta^\prime_t$ and $\delta^\prime_t$ are strongly coupled with each other at each iteration. Then, we need to decouple them so that we can characterize the effect of $\Delta^\prime_t$ and $\delta^\prime_t$ on the overall convergence separately. 

We first consider the accumulation of $\Delta^\prime_t$ over one epoch. Although the value of $\Delta^\prime_t$ increases within each epoch (indicated by \cref{eq: pf1}), the accumulation of this error can still be controlled via adjusting the mini-batch sizes $S_1$, $S_2$ and epoch length $q$. Under an appropriate parameter setting, we can obtain the following bound
\begin{flalign*}
	\Delta^\prime_t &\leq 2\Delta^\prime_{\lfloor t\rfloor q} + \Theta(\epsilon)\sum_{p=\lfloor t\rfloor q}^{t-1}\delta^\prime_{t-1} + \Theta(\kappa^{-2}\epsilon)\sum_{p=\lfloor t\rfloor q}^{t-1}\mE[\ltwo{v_{t-1}}^2]  +  \Theta(\kappa^{-2}\epsilon).
\end{flalign*}
Note that $\Delta^\prime_{\lfloor t\rfloor q}$ is the estimation error of coordinate-wise estimator obtained at the beginning of each epoch, which diminishes as the batch size $S_1$ increases. Letting $S_1 = \Theta(\kappa^2/\epsilon^2)$ as specified in \Cref{thm2}, we can bound the accumulation of $\Delta^\prime_t$ over the all iterations as
\begin{flalign}
	\sum_{t=0}^{T-1}\Delta^\prime_t&\leq \Theta\left(\frac{1}{\kappa}\right) + \Theta(1)\sum_{t=0}^{T-1} \delta^\prime_t + \Theta\left(\frac{1}{\kappa^2}\right) \sum_{t=0}^{T-1} \mE[\ltwo{v_t}^2]  + \Theta(\kappa^{-1}).\label{eq: pf3}
\end{flalign}
Moreover, based on the contraction property of $\delta^\prime_t$ provided in \cref{eq: pf2}, we derive the following bound for the accumulation of $\delta^\prime_t$:
\begin{flalign}
	\sum_{t=0}^{T-1}\delta^\prime_t&\leq 2\delta^\prime_0 + \Theta(1)\sum_{t=0}^{T-1}\Delta^\prime_t + \Theta\left( \frac{1}{\kappa^2}\right)\sum_{t=0}^{T-1} \mE[\ltwo{v_t}^2] +\Theta(\kappa^{-4}).\label{eq: pf4}
\end{flalign}
Combining \cref{eq: pf3} and \cref{eq: pf4}, the upper bounds for $\sum_{t=0}^{T-1}\Delta^\prime_t$ and $\sum_{t=0}^{T-1}\delta^\prime_t$ can then be derived separately as
\begin{flalign}
	\sum_{t=0}^{T-1}\Delta^\prime_t \leq\Theta\left(\frac{1}{\kappa}\right) + \Theta(1)\delta^\prime_0 + \Theta\left(\frac{1}{\kappa^2}\right)\sum_{t=0}^{T-1}\mE[\ltwo{v_t}^2],\label{eq: pf5}\\
	\sum_{t=0}^{T-1}\delta^\prime_t \leq \Theta\left(\frac{1}{\kappa}\right) + \Theta(1)\delta^\prime_0 + \Theta\left(\frac{1}{\kappa^2}\right)\sum_{t=0}^{T-1}\mE[\ltwo{v_t}^2].\label{eq: pf6}
\end{flalign}

\noindent\textbf{Step 3:} Note that \cref{eq: pf5} and \cref{eq: pf6} alone are not sufficient to guarantee the boundness of accumulation errors $\sum_{t=0}^{T-1}\Delta^\prime_t$ and $\sum_{t=0}^{T-1}\delta^\prime_t$, as the upper bounds in \cref{eq: pf5} and \cref{eq: pf6} depend on an unknown error term $\sum_{t=0}^{T-1}\mE[\ltwo{v_t}^2]$. To handle this issue, we utilize the Lipschitz property of $\Phi(x)$ given in Assumption \ref{ass2} to obtain the following bound
\begin{flalign}
\left(\frac{\alpha}{2}-\frac{L\alpha^2}{2}\right)\sum_{t=0}^{T-1}\mE[\ltwo{v_t}^2] \leq \Phi(x_0) - \mE[\Phi(x_T)] + 2\alpha\kappa^2\sum_{t=0}^{T-1}\delta^\prime_t + 2\alpha\sum_{t=0}^{T-1}\Delta^\prime_t + T \Theta(\epsilon^{-2}\kappa^4).\label{eq: pf7}
\end{flalign}
Substituting \cref{eq: pf5} and \cref{eq: pf6} into \cref{eq: pf7} and subtracting the residual terms on both sides yield the following bound
\begin{flalign}
	\sum_{t=0}^{T-1}\mE[\ltwo{v_t}^2]\leq \Theta(L(\Phi(x_0)-\Phi^*)) + \Theta(\kappa).\label{eq: pf9}
\end{flalign}
The upper bounds of $\sum_{t=0}^{T-1}\Delta^\prime_t$ and $\sum_{t=0}^{T-1}\delta^\prime_t$ can then be obtained by substituting \cref{eq: pf9} into \cref{eq: pf5} and \cref{eq: pf6}.
 
To establish the convergence rate for $\mE[\ltwo{\nabla\Phi(\hat{x})}^2]=\frac{1}{T}\sum_{t=0}^{T}\mE[\ltwo{\nabla\Phi({x}_t)}^2]$, we note that
\begin{flalign}
\sum_{t=0}^{T-1}\mE[\ltwo{\nabla\Phi(x_t)}^2] \leq 6\kappa^2\sum_{t=0}^{T-1}\delta^\prime_t + 6\sum_{t=0}^{T-1}\Delta^\prime_t + 3\sum_{t=0}^{T-1}\mE[\ltwo{v_t}^2] + \Theta(\kappa^{-2}).\label{eq: pf8}
\end{flalign}
Substituting the bounds on $\sum_{t=0}^{T-1}\mE[\ltwo{v_t}^2]$, $\sum_{t=0}^{T-1}\Delta^\prime_t$ and $\sum_{t=0}^{T-1}\delta^\prime_t$ into \cref{eq: pf8}, we obtain the convergence rate for ZO-VRGDA.
\end{proof}


%

\section{Experiments}\label{sc: experiment}
 Our experiments focus on two types of comparisons: (a) we compare our ZO-VRGDA with other existing zeroth-order stochastic algorithms and demonstrate the superior performance of ZO-VRGDA; (b) we compare the performance of ZO-VRGDA with different inner-loop lengths.
 
Our experiments solve a distributionally robust optimization problem, which is commonly used for studying  minimax optimization \cite{lin2019gradient,rafique2018non}. We conduct the experiments on three datasets from LIBSVM \cite{Chang_2011}. The details of the problem and the datasets  are provided in \Cref{exp_in_appendix}.

\begin{figure*}[ht]  
	\centering 
	\subfigure[Dataset: a9a]{\includegraphics[width=50mm]{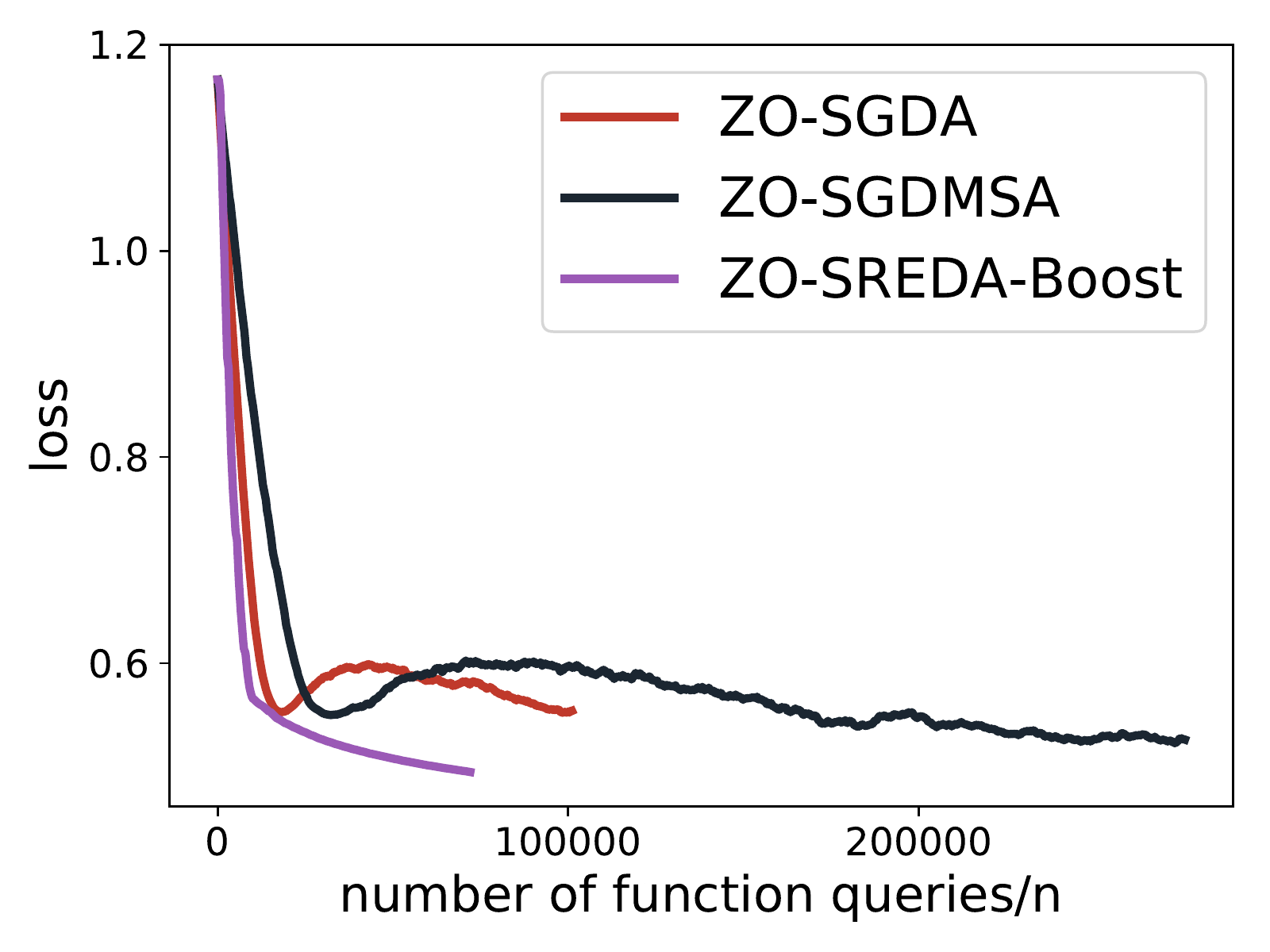}}
	\subfigure[Dataset: w8a]{\includegraphics[width=50mm]{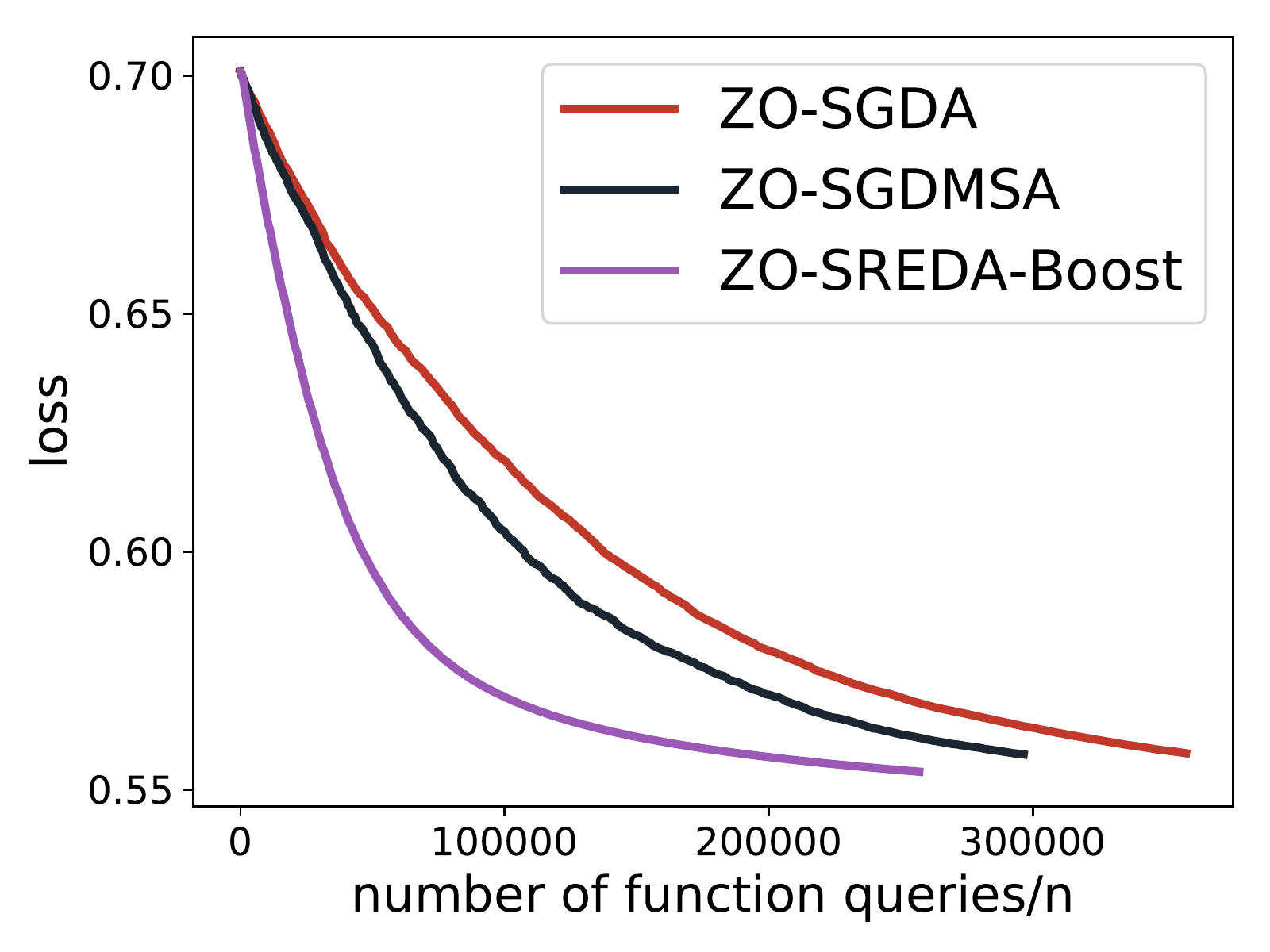}}
	\subfigure[Dataset: mushrooms]{\includegraphics[width=50mm]{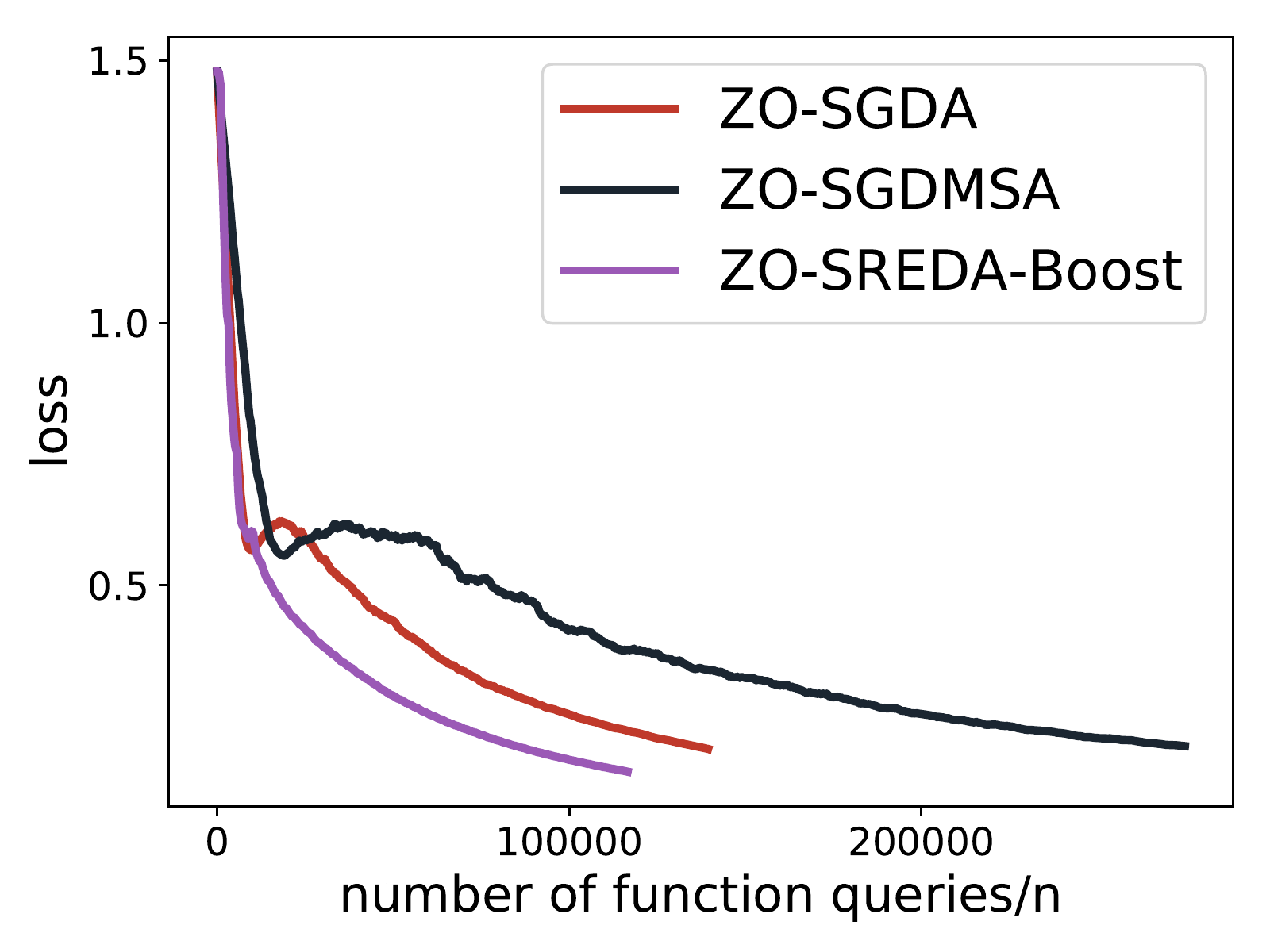}}
	\caption{\small  Comparison of function query complexity among three algorithms.}   \label{Exp_3}
\end{figure*}

{\bf Comparison among zeroth-order algorithms:} We compare the performance of our proposed ZO-VRGDA with that of two existing zeroth-order algorithms ZO-SGDA \cite{wang2020zeroth} and ZO-SGDMSA \cite{wang2020zeroth} designed for nonconvex-strongly-concave minimax problems. For ZO-SGDA and ZO-SGDMSA, as suggested by the corresponding theory, we set the mini-batch size $B = C d_1/\epsilon^2$ and $B = C d_2/\epsilon^2$ for updating the variables $x$ and $y$, respectively. For ZO-VRGDA, based on our theory, we set the mini-batch size $B = C d_1/\epsilon$ and $B = C d_2/\epsilon$ for updating the variables $x$ and $y$, and set $S_1 = n$ for the large batch, where $n$ is the number of data samples in the dataset. We set $C = 0.1$ and $\epsilon = 0.1$ for all algorithms. We further set the stepsize $\eta = 0.01$ for ZO-VRGDA and ZO-SGDMSA. Since ZO-SGDA is a two time-scale algorithm, we set $\eta = 0.01$ as the stepsize for the fast time scale, and $\eta/{\kappa^3}$ as the stepsize for slow time scale (based on the theory) where $\kappa^3=10$. 
It can be seen in  \Cref{Exp_3} that ZO-VRGDA substantially outperforms the other two algorithms in terms of the function query complexity (i.e., the running time).

\begin{figure*}[ht]  
	\centering 
	\subfigure[Dataset: a9a]{\includegraphics[width=50mm]{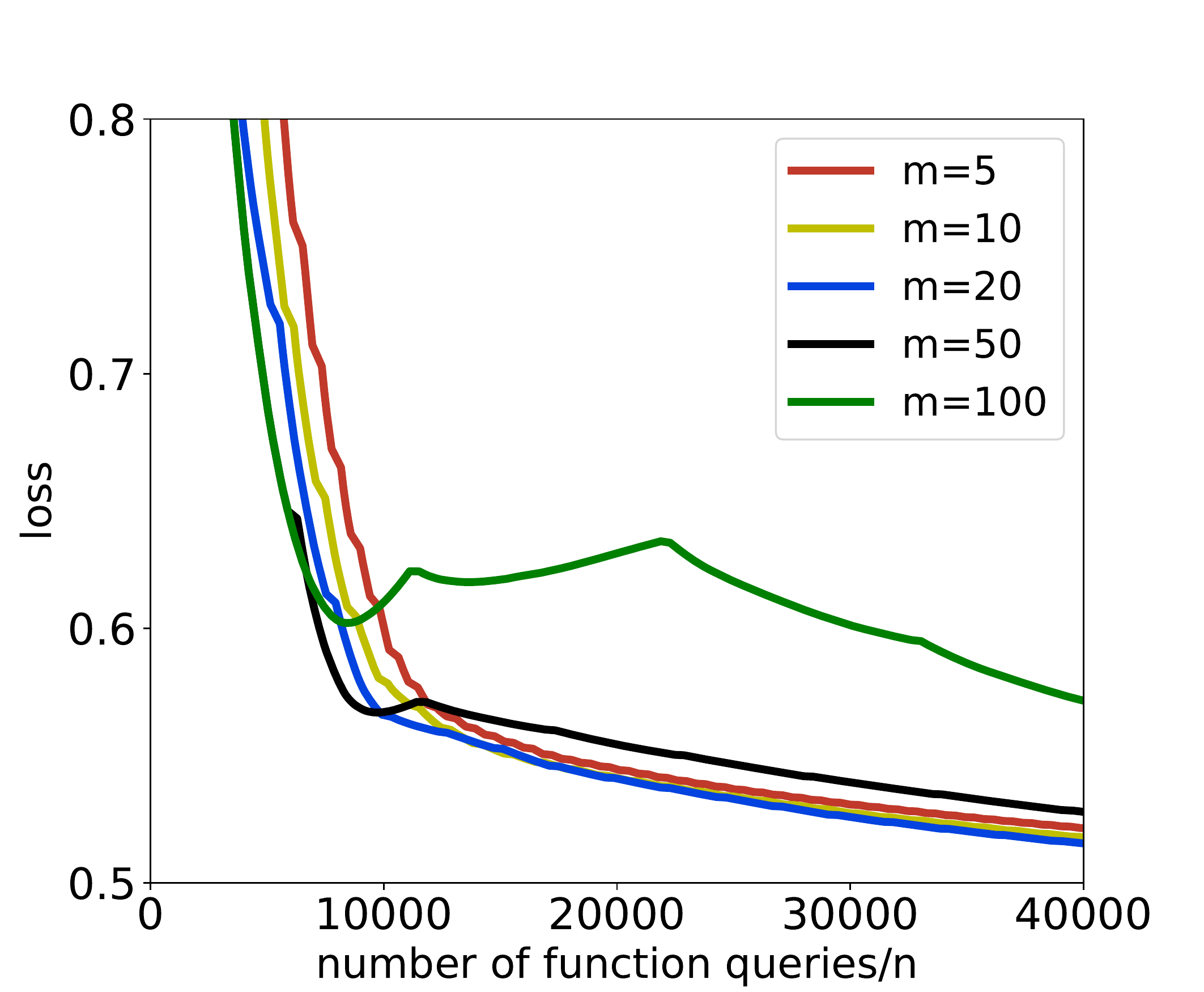}}
	\subfigure[Dataset: w8a]{\includegraphics[width=50mm]{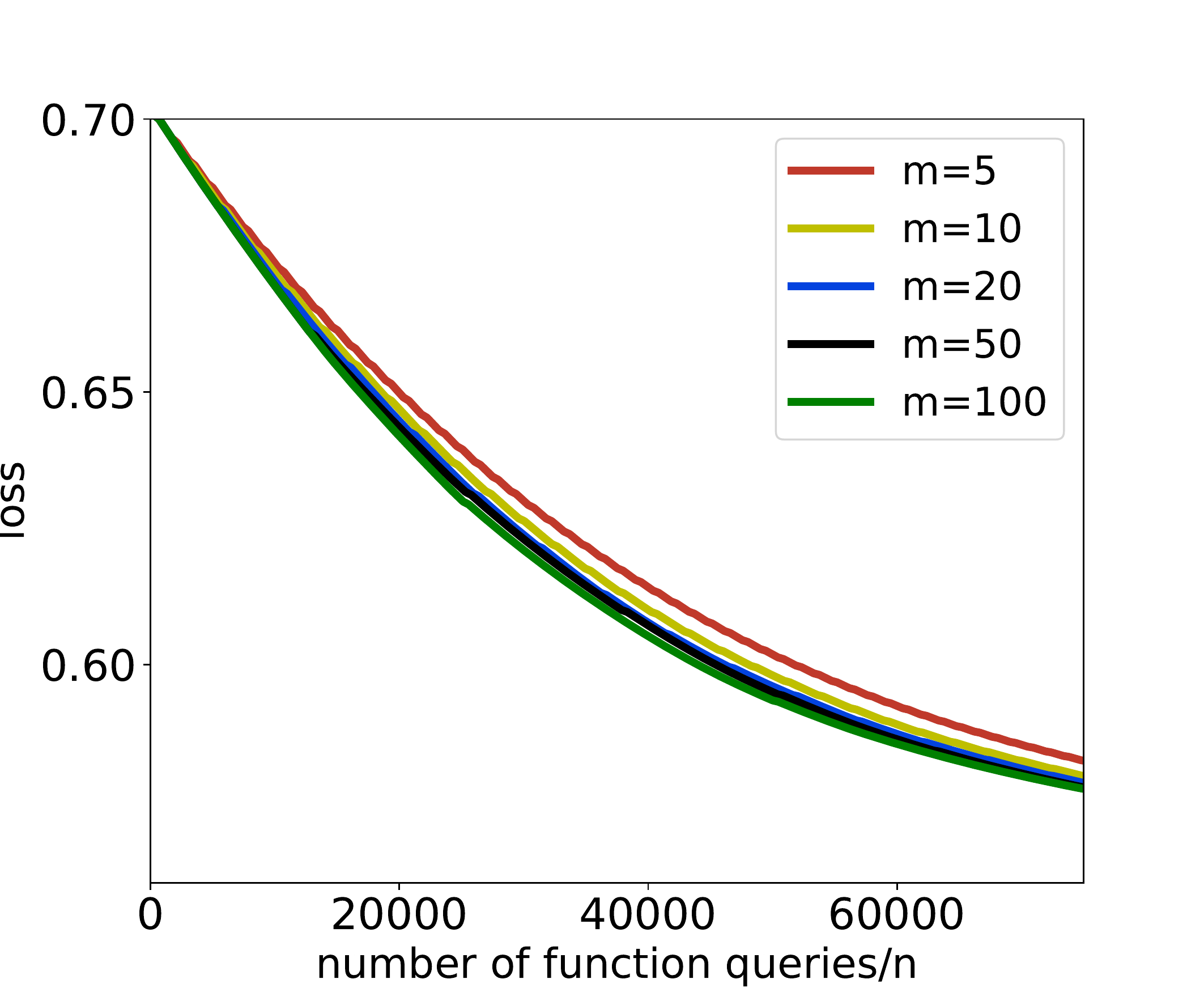}}
	\subfigure[Dataset: mushrooms]{\includegraphics[width=50mm]{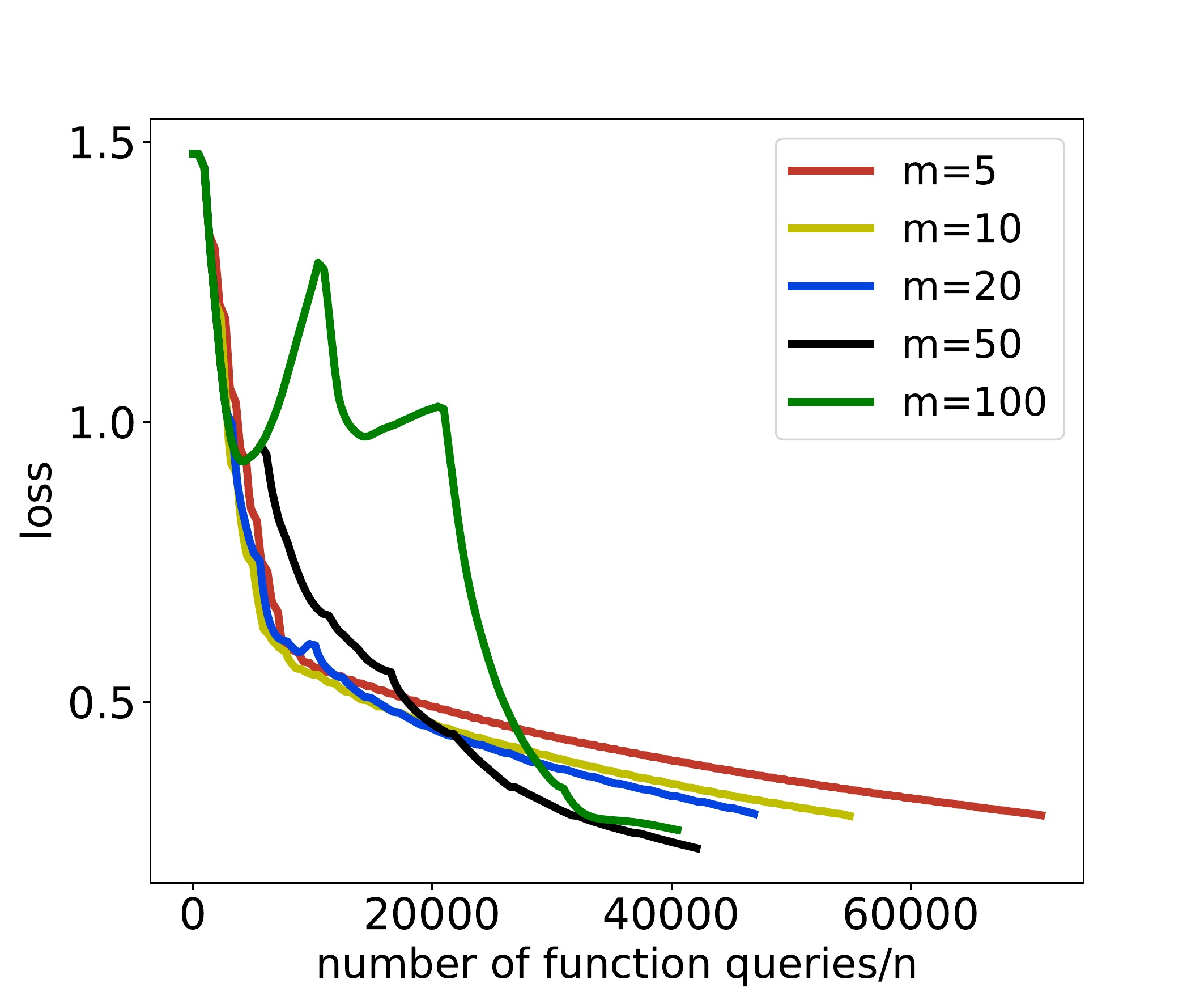}}
	\caption{\small  Comparison of ZO-VRGDA with different inner-loop lengths.}   \label{Exp_2}
\end{figure*}

{\bf Comparison among different inner-loop length: }We investigate how the inner-loop length affects the overall convergence of ZO-VRGDA. We consider the following inner loop lengths $\{5, 10, 20, 50, 100\}$.  It can be seen in \Cref{Exp_2} that ZO-VRGDA converges faster as we increase the inner-loop length $m$ initially, and then the convergence slows down as we further enlarge $m$ beyond a certain threshold. This verifies the tradeoff role that $m$ plays, i.e., larger $m$ attains a better optimized $y$ but causes more queries. \Cref{Exp_2} also illustrates that the performance of ZO-VRGDA is fairly robust to the inner-loop length as long as $m$ is not too large.


\section{Conclusion}
In this work, we have proposed the first zeroth-order variance reduced algorithm ZO-VRGDA for solving nonconvex-strongly-concave minimax optimization problems. The function query complexity of ZO-VRGDA achieves the best dependence on the target accuracy compared to previously designed gradient-free algorithms. We have also developed a novel analysis framework to characterize the convergence rate and the complexity, which we expect to be also useful for studying various other stochastic minimax problems such as proximal, momentum, and manifold optimization.

\section*{Acknowledgments} 
The work was supported in part by the U.S. National Science Foundation under the grants CCF-1761506, CCF-1801846, CCF-1801855, and CCF-190929.

\newpage
\bibliography{ref}
\bibliographystyle{apalike}

\onecolumn
\newpage
\appendix
\noindent {\Large \textbf{Supplementary Materials}}
\section{Specifications of Experiments} \label{exp_in_appendix} 
The distributionally robust optimization problem is formulated as follows:
\begin{align*}
\min_{x \in \mathcal{X}} 	\max_{y \in \mathcal{Y}} \sum_{i=1}^{n} y_i f_i(x) - r(y),
\end{align*}
where  $\mathcal{X} = \{x \in \mathbb{R}^d\}$, $\mathcal{Y} = \{y \in \mathbb{R}^n| \sum_{i=1}^{n} y_i = 1, y_i \geq 0, i = 1, \cdots n\}$, $r(y) = 10 \sum_{i=1}^{n} (y_i - 1/n)^2$, $f_i(x) = \phi(l(x))$ where $\phi(\theta) = 2\log\left(1 + \frac{\theta}{2}\right)$, $l(x; s, z) = \log(1 + \exp(-zx^\top s))$, and $(s, z)$ are the feature and label pair of a data sample. It can be seen that the problem is a minimax problem with  $d_1 = d$ and $d_2 = n$.  Since the distributionally robust optimization aims at an unbalanced dataset, we pick the samples from the original dataset and set the ratio between the number of negative  labeled samples and the number of positive labeled samples to be $1:4$. Since the maximization over $y$ is a constrained optimization problem, we incorporate a projection step after   updates of $y$ for all algorithms.


The details of the datasets used for zeroth-order algorithms are listed in \Cref{Exp_zeroth_order_dataset}. 

\begin{table}[H]
	\caption{Datasets used for zeroth-order algorithms} \label{Exp_zeroth_order_dataset}
	\centering
	\begin{tabular}{|l|l|l|l|}
		\toprule
		Datasets	& \# of samples & \# of features  & \# Pos: \# Neg \\	
		mushrooms& 200 &112  & 1:4  \\
		w8a	& 100 &300  &  1:4  \\ 
		a9a	& 150 &123  &  1:4  \\
		\bottomrule
	\end{tabular}
\end{table}

\section{Technical Lemmas}

\subsection{Preliminary Lemmas}
We first provide useful inequalities in convex optimization \cite{nesterov2013introductory,polyak1963gradient} and auxiliary lemmas from \cite{fang2018spider,luo2020stochastic}.
\begin{lemma}[\cite{nesterov2013introductory},\cite{polyak1963gradient}]\label{lemma12}
	Suppose $h(\cdot)$ is convex and has $\ell$-Lipschitz gradient. Then, we have
	\begin{flalign}
	\langle \nabla h(w) - \nabla h(w^\prime), w-w^\prime \rangle &\geq \frac{1}{\ell}\ltwo{\nabla h(w)-\nabla h(w^\prime)}^2.\label{eq: 6}
	\end{flalign}
\end{lemma}
\begin{lemma}[\cite{nesterov2013introductory},\cite{polyak1963gradient}]\label{lemma2}
	Suppose $h(\cdot)$ is $\mu$-strongly convex and has $\ell$-Lipschitz gradient. Let $w^*$ be the minimizer of $h$. Then for any $w$ and $w^\prime$, we have the following inequalities hold.
	\begin{flalign}
	\langle \nabla h(w) - \nabla h(w^\prime), w-w^\prime \rangle &\geq \frac{\mu \ell}{\mu + \ell}\ltwo{w-w^\prime}^2 + \frac{1}{\mu+\ell}\ltwo{\nabla h(w) - \nabla h(w^\prime)}^2,\label{eq: 7}\\
	\ltwo{\nabla h(w)-\nabla h(w^\prime)}&\geq \mu \ltwo{w-w^\prime},\label{eq: 15}\\
	2\mu(h(w)-h(w^\prime))&\leq \ltwo{\nabla h(w)}^2.\label{eq: 8}
	\end{flalign}
\end{lemma}
\begin{lemma}[\cite{fang2018spider}, Lemma 2]\label{lemma5}
	Suppose Assumption \ref{ass5} hold. For any $(x,y)\in \mR^{d_1}\times\mR^{d_2}$ and sample batch $\{ \xi_1,\cdots,\xi_S\}$, let $v=\frac{1}{S}\sum_{i=1}^{S}\nabla_xF(x,y,\xi_i)$ and $u=\frac{1}{S}\sum_{i=1}^{\mcs}\nabla_yF(x,y,\xi_i)$. We have
	\begin{flalign*}
	\mE[\ltwo{v-\nabla_x f(x,y)}^2] + \mE[\ltwo{u-\nabla_y f(x,y)}^2] \leq \frac{\sigma^2}{S}.
	\end{flalign*}
\end{lemma}
\begin{lemma}[\cite{fang2018spider}, Lemma 1]\label{lemma3}
	Let $\mcv_t$ be an estimator of $\mcb(z_t)$ as
	\begin{flalign*}
	\mcv_t=\mcb_{\mcs_*}(z_t) - \mcb_{\mcs_*}(z_{t-1}) + \mcv_{t-1},
	\end{flalign*}
	where $\mcb_{\mcs_*}=\frac{1}{\lone{\mcs_*}}\sum_{\mcb_i\in \mcs_*}\mcb_i$ satisfies
	\begin{flalign*}
	\mE[\mcb_i(z_t)-\mcb_i(z_{t-1}) | z_0,\cdots,z_{t-1} ]=\mE[\mcv_t-\mcv_{t-1}|z_0,\cdots,z_{t-1}].
	\end{flalign*}
	For all $k=1,\cdots,K$, we have
	\begin{flalign*}
	\mE[\ltwo{\mcv_t-\mcv_{t-1} - (\mcb_{\mcs_*}(z_t) - \mcb_{\mcs_*}(z_{t-1})) }^2]\leq \frac{1}{\mcs_*}\mE[\ltwo{\mcb_{i}(z_t) - \mcb_{i}(z_{t-1})}^2|z_0,\cdots,z_{t-1}],
	\end{flalign*}
	and
	\begin{flalign*}
	\mE[\ltwo{\mcv_t-\mcb(z_t)|z_0,\cdots,z_{t-1}}^2]\leq \ltwo{\mcv_{t-1}-\mcb(z_{t-1})}^2+\frac{1}{\lone{\mcs_*}}\mE[\ltwo{\mcb_{i}(z_t) - \mcb_{i}(z_{t-1})}^2|z_0,\cdots,z_{t-1}].
	\end{flalign*}
	Furthermore, if $\mcb_i$ is L-Lipschitz continuous in expectation, we have
	\begin{flalign*}
	\mE[\ltwo{\mcv_t-\mcb(z_t)|z_0,\cdots,z_{t-1}}^2]\leq \ltwo{\mcv_{t-1}-\mcb(z_{t-1})}^2+\frac{L^2}{\lone{\mcs_*}}\mE[\ltwo{z_{t}-z_{t-1}}^2|z_0,\cdots,z_{t-1}].
	\end{flalign*}
\end{lemma}

We provide the following lemmas to characterize the properties of Gaussian smoothed function and zeroth-order Gaussian gradient estimator. Consider a function $h(\cdot)$: $\mR^d\rightarrow \mR$. Let $\nu$ be a $d$-dimensional standard Gaussian random vector and $\mu>0$ be the smoothing parameter. Then a smooth approximation of $h(\cdot)$ is defined as $h_\tau(x)=\mE_\nu[h(x+\tau\nu)]$. We have the following lemmas.
\begin{lemma}[\cite{nesterov2017random}, Section 2]\label{lemma8}
	If $h(\cdot)$ is convex, then $h_\mu(\cdot)$ is also a convex function.
\end{lemma}
\begin{lemma}[\cite{ghadimi2013stochastic}, Section 3.1]\label{lemma9}
	If $h(\cdot)$ has $\ell$-Lipschitz gradient, then $h_\mu(\cdot)$ also has $\ell$-Lipschitz gradient.
\end{lemma}
\begin{lemma}[\cite{nesterov2017random}, Theorem 1]\label{lemma10}
	If $h(\cdot)$ has $\ell$-Lipschitz gradient, then for all $x\in\mR^d$, we have $\lone{h(x)-h_\tau(x)}\leq \frac{\tau^2}{2}\ell d$.
\end{lemma}
\begin{lemma}[\cite{nesterov2017random}, Lemma 3]\label{lemma11}
	If $h(\cdot)$ has $\ell$-Lipschitz gradient, then $\ltwo{\nabla_xh_\tau(x)-\nabla_xh(x)}^2\leq \frac{\tau^2}{4}\ell^2(d+3)^3$.
\end{lemma}
The following lemma characterizes the estimation error of a zeroth-order coordinate-wise estimator with batch size $S_1$ in lines 6 and 7 in \Cref{al:zosredaboost}.
\begin{lemma}\label{lemma19}
	Suppose Assumption \ref{ass2} and \ref{ass5} hold. Suppose $\text{mod}(t,q)=0$, and let $\epsilon(S_1,\delta)=\mE[\ltwo{v_t-\nabla_xf_{\mu_1}(x_t,y_t)}^2] + \mE[\ltwo{u_t-\nabla_yf_{\mu_2}(x_t,y_t)}^2]$. Then, we have
	\begin{flalign*}
		\epsilon(S_1,\delta) \leq \frac{(d_1+d_2)\ell^2\delta^2}{2} + \frac{4\sigma^2}{S_1} + \frac{\mu_1^2}{2}\ell^2(d_1+3)^3 + \frac{\mu_2^2}{2}\ell^2(d_2+3)^3.
	\end{flalign*}
\end{lemma}
\begin{proof}
	(B.56) and (B.57) in \cite{fang2018spider} imply that
	\begin{flalign}
		\mE[\ltwo{v_t-\nabla_xf(x_t,y_t)}^2] \leq \frac{d_1\ell^2 \delta^2 }{2}+ \frac{2\sigma^2}{S_1},\label{eq: 69}
	\end{flalign}
	and
	\begin{flalign}
		\mE[\ltwo{u_t-\nabla_yf(x_t,y_t)}^2] \leq \frac{d_2\ell^2 \delta^2 }{2}+ \frac{2\sigma^2}{S_1}. \label{eq: 70}
	\end{flalign}
	Then we proceed as follows:
	\begin{flalign}
		&\mE[\ltwo{v_t-\nabla_xf_{\mu_1}(x_t,y_t)}^2] + \mE[\ltwo{u_t-\nabla_yf_{\mu_2}(x_t,y_t)}^2] \nonumber\\
		&\leq 2\mE[\ltwo{v_t-\nabla_xf(x_t,y_t)}^2] + 2\mE[\ltwo{u_t-\nabla_yf(x_t,y_t)}^2] \nonumber\\
		&\quad + 2\mE[\ltwo{\nabla_xf_{\mu_1}(x_t,y_t)-\nabla_xf(x_t,y_t)}^2] + 2\mE[\ltwo{\nabla_xf_{\mu_2}(x_t,y_t)-\nabla_yf(x_t,y_t)}^2] \nonumber\\
		&\overset{(i)}{\leq} 2\mE[\ltwo{v_t-\nabla_xf(x_t,y_t)}^2] + 2\mE[\ltwo{u_t-\nabla_yf(x_t,y_t)}^2] + \frac{\mu_1^2}{2}\ell^2(d_1+3)^3 + \frac{\mu_2^2}{2}\ell^2(d_2+3)^3 \nonumber\\
		&\overset{(ii)}{\leq} (d_1+d_2)\ell^2\delta^2 + \frac{8\sigma^2}{S_1} + \frac{\mu_1^2}{2}\ell^2(d_1+3)^3 + \frac{\mu_2^2}{2}\ell^2(d_2+3)^3,\nonumber
	\end{flalign}
	where $(i)$ follows from \Cref{lemma11}, and $(ii)$ follows from \cref{eq: 69} and \cref{eq: 70}.
\end{proof}
We denote 
\begin{flalign*}
	G_{\mu_1}(x,y,\nu_i,\xi_i)=\frac{F(x+\mu_1 \nu_i,y,\xi_i)-F(x,y,\xi_i)}{\mu_1}\nu_i
\end{flalign*}
and
\begin{flalign*}
	H_{\mu_2}(x,y,\omega_i,\xi_i)=\frac{F(x,y+\mu_2 \omega_i,\xi_i)-F(x,y,\xi_i)}{\mu_2}\omega_i
\end{flalign*}
as unbiased estimators of $\nabla_x f_{\mu_1}(x,y)$ and $\nabla_y f_{\mu_2}(x,y)$, respectively. Then we have the following lemma.
\begin{lemma}\label{lemma6}
	Suppose Assumption \ref{ass2} holds, and suppose $u_1$ and $u_2$ are standard Gaussian random vector, i.e., $\nu_i\sim N(0,\mathbf{1}_{d_1})$ and $\omega_i\sim N(0,\mathbf{1}_{d_2})$. Then, we have
	\begin{flalign*}
	\mE\left[ \ltwo{G_{\mu_1}(x,y,\nu_i,\xi_i) - G_{\mu_1}(x^\prime,y,\nu_i,\xi_i) }^2 \right]\leq 2(d_1+4)\ell^2\ltwo{x-x^\prime}^2 + 2\mu^2_1(d_1+6)^3\ell^2,\\
	\mE\left[ \ltwo{G_{\mu_1}(x,y,\nu_i,\xi_i) - G_{\mu_1}(x,y^\prime,\nu_i,\xi_i) }^2 \right]\leq 2(d_1+4)\ell^2\ltwo{y-y^\prime}^2 + 2\mu^2_1(d_1+6)^3\ell^2,
	\end{flalign*}
	and
	\begin{flalign*}
	\mE\left[ \ltwo{H_{\mu_2}(x,y,\nu_i,\xi_i) - H_{\mu_2}(x^\prime,y,\nu_i,\xi_i) }^2 \right]\leq 2(d_2+4)\ell^2\ltwo{x-x^\prime}^2 + 2\mu^2_2(d_2+6)^3\ell^2,\\
	\mE\left[ \ltwo{H_{\mu_2}(x,y,\nu_i,\xi_i) - H_{\mu_2}(x,y^\prime,\nu_i,\xi_i) }^2 \right]\leq 2(d_2+4)\ell^2\ltwo{y-y^\prime}^2 + 2\mu^2_2(d_2+6)^3\ell^2.
	\end{flalign*}
\end{lemma}
\begin{proof}
	The proof is similar to that of Lemma 3 in \cite{fang2018spider}. Here we provide the proof for completeness.
	We will show how to upper bound the term $\mE\left[ \ltwo{G_{\mu_1}(x,y,\nu_1,\xi) - G_{\mu_1}(x^\prime,y,\nu_1,\xi) }^2 \right]$ here. Then, the upper bounds on the remaining three terms can be obtained by following similar steps. We proceed the bound as follows.
	\begin{flalign}
	&\mE\left[ \ltwo{G_{\mu_1}(x,y,\nu_i,\xi_i) - G_{\mu_1}(x,y^\prime,\nu_i,\xi_i) }^2 \right] \nonumber\\
	&=\mE\left[ \ltwo{\frac{F(x+\mu_1\nu_i,y,\xi_i) - F(x,y,\xi_i) }{\mu_1}\nu_1 - \frac{F(x+\mu_1\nu_i,y^\prime,\xi_i) - F(x,y^\prime,\xi_i)}{\mu_1}\nu_1}^2 \right] \nonumber\\
	&= \mE\Bigg[\Bltwo{\frac{F(x+\mu_1\nu_i,y,\xi_i) - F(x,y,\xi_i) - \langle \nabla_x F(x,y,\xi_i), \mu_1 \nu_i \rangle }{\mu_1}\nu_i \nonumber\\
		&\quad- \frac{F(x+\mu_1\nu_i,y^\prime,\xi_i) - F(x,y^\prime,\xi_i)- \langle \nabla_x F(x,y^\prime,\xi_i), \mu_1 \nu_i \rangle }{\mu_1}\nu_i \nonumber\\
		&\quad + \langle \nabla_x F(x,y,\xi_i) - \nabla_x F(x,y^\prime,\xi_i), \nu_i \rangle \nu_i }^2\Bigg] \nonumber\\
	&\leq 2\mE\Bigg[\Bltwo{\frac{F(x+\mu_1\nu_i,y,\xi_i) - F(x,y,\xi_i) - \langle \nabla_x F(x,y,\xi_i), \mu_1 \nu_i \rangle }{\mu_1}\nu_i \nonumber\\
		&\quad- \frac{F(x+\mu_1\nu_i,y^\prime,\xi_i) - F(x,y^\prime,\xi_i)- \langle \nabla_x F(x,y^\prime,\xi_i), \mu_1 \nu_i \rangle }{\mu_1}\nu_i}^2\Bigg] \nonumber\\
	&\quad + 2\mE\big[\ltwo{\langle \nabla_x F(x,y,\xi_i) - \nabla_x F(x,y^\prime,\xi_i), \nu_i \rangle \nu_i}^2 \big]\nonumber\\
	&\leq 4\mE\Bigg[\Bltwo{\frac{F(x+\mu_1\nu_i,y,\xi_i) - F(x,y,\xi_i) - \langle \nabla_x F(x,y,\xi_i), \mu_1 \nu_i \rangle }{\mu_1}\nu_i}^2 \Bigg] \nonumber\\
	&\quad + 4\mE\Bigg[\Bltwo{\frac{F(x+\mu_1\nu_i,y^\prime,\xi_i) - F(x,y^\prime,\xi_i)- \langle \nabla_x F(x,y^\prime,\xi_i), \mu_1 \nu_i \rangle }{\mu_1}\nu_i}^2\Bigg] \nonumber\\
	&\quad + 2\mE\big[\ltwo{\langle \nabla_x F(x,y,\xi_i) - \nabla_x F(x,y^\prime,\xi_i), \nu_i \rangle \nu_i}^2 \big] \nonumber\\
	&\leq 4\mE\Bigg[\lone{\frac{F(x+\mu_1\nu_i,y,\xi_i) - F(x,y,\xi_i) - \langle \nabla_x F(x,y,\xi_i), \mu_1 \nu_i \rangle }{\mu_1}}^2 \ltwo{\nu_i}^2 \Bigg] \nonumber\\
	&\quad + 4\mE\Bigg[\lone{\frac{F(x+\mu_1\nu_i,y^\prime,\xi_i) - F(x,y^\prime,\xi_i)- \langle \nabla_x F(x,y^\prime,\xi_i), \mu_1 \nu_i \rangle }{\mu_1}}^2 \ltwo{\nu_i}^2 \Bigg] \nonumber\\
	&\quad + 2\mE\big[\ltwo{\langle \nabla_x F(x,y,\xi_i) - \nabla_x F(x,y^\prime,\xi_i), \nu_i \rangle \nu_i }^2 \big] \nonumber\\
	&\overset{(i)}{\leq} 2\mu^2_1\ell^2\mE[\ltwo{\nu_i}^2] + 2\mE\big[\ltwo{\langle \nabla_x F(x,y,\xi_i) - \nabla_x F(x,y^\prime,\xi_i), \nu_i \rangle \nu_i }^2 \big] \nonumber\\
	&\overset{(ii)}{\leq} 2\mu^2_1\ell^2\mE[\ltwo{\nu_i}^2] + 2(d_1+4)\mE\big[\ltwo{ \nabla_x F(x,y,\xi_i) - \nabla_x F(x,y^\prime,\xi_i)  }^2 \big] \nonumber\\
	&\overset{(iii)}{\leq} 2\mu^2_1(d_1+6)^3\ell^2 + 2(d_1+4)\ell^2\mE\big[\ltwo{ y - y^\prime }^2 \big],\nonumber
	\end{flalign}
	where $(i)$ follows from the fact that for any $a,a^\prime\in \mR^{d_1}$ and $b\in \mR^{d_2}$, we have
	\begin{flalign*}
	\lone{F(a,b,\xi_i) - F(a^\prime,b,\xi_i)- \langle \nabla_x F(a,b,\xi_i), a-a^\prime \rangle}\leq \frac{\ell}{2}\ltwo{a-a^\prime}^2,
	\end{flalign*}
	because $F(a,b,\xi)$ has $\ell$-Lipschitz continuous gradient; $(ii)$ follows because
	\begin{flalign*}
	\mE[\ltwo{\langle a, \nu_i \rangle \nu_i}^2]\leq (d_1+4)\ltwo{a}^2,
	\end{flalign*}
	obtained from $(33)$ in \cite{nesterov2017random}, and $(iii)$ follows because $\mE[\ltwo{\nu_i}^2]\leq (d_1+6)^3$ in $(17)$ of \cite{nesterov2017random}.
\end{proof}

\subsection{Useful Properties for Zeroth-Order Concave Maximizer}\label{sc: resultofzoconcavemaximizer}
In this section, we show some properties for the zeroth-order concave maximizer in \Cref{al:zoconcavemaximizer}. For simplicity, for any given $t\geq 0$, we define $g_t(y)=-f(x_{t+1},y)$ and $g_{t,\mu_2}(y)=-f_{\mu_2}(x_{t+1},y)$. \Cref{lemma8} and \Cref{lemma9} imply that $g_t(\cdot)$ is $\mu$-strongly convex and has $\ell$-Lipschitz gradient, and $g_{t,\mu_2}(\cdot)$ is convex and has $\ell$-Lipschitz gradient. We also define $\ty^*_t=\argmin_y g_t(y)$. We can obtain the following two lemmas by following the same steps in \cite{luo2020stochastic}
\begin{lemma}[Lemma 9 of \cite{luo2020stochastic}]\label{lemma15}
	Consider \Cref{al:zoconcavemaximizer}. We have
	\begin{flalign*}
		\sum_{k=0}^{m}\mE[\ltwo{\nabla g_{t,\mu_2}(\ty_{t,k})}^2] \leq \frac{2}{\beta}\mE[g_{t,\mu_2}(\ty_{t,0}) - g_{t,\mu_2}(\ty_{t,m+1}) ] + \sum_{k=0}^{m}\mE[\ltwo{\nabla g_{t,\mu_2}(\ty_{t,k}) - \tu_{t,k} }^2].
	\end{flalign*}
\end{lemma}
\begin{lemma}[Lemma 11 of \cite{luo2020stochastic}]\label{lemma16}
	Consider \Cref{al:zoconcavemaximizer} with any $\beta\leq \frac{2}{\ell}$ and $k\geq 1$. We have
	\begin{flalign*}
		\mE[\ltwo{\nabla g_{t,\mu_2}(\ty_{t,k}) - \tu_{t,k} }^2]\leq \mE[\ltwo{\nabla g_{t,\mu_2}(\ty_{t,0}) - \tu_{t,0} }^2] + \frac{\ell\beta}{2-\ell\beta}\mE[\ltwo{\tu_{t,0}}^2].
	\end{flalign*}
\end{lemma}
The following lemma characterizes the recursion of $\mE[\ltwo{\nabla g_{t,\mu_2}(\ty_{t,\tm_t})}^2]$ within each inner loop.
\begin{lemma}\label{lemma17}
	Consider \Cref{al:zoconcavemaximizer}. For any $k\geq 1$ and $\beta\leq \frac{1}{\ell}$, we have
	\begin{flalign*}
	\mE[\ltwo{\nabla g_{t,\mu_2}(\ty_{t,\tm_t})}^2] &\leq \frac{2}{\beta\mu(m+1)}\mE[\ltwo{\nabla g_{t,\mu_2}(\ty_{t,0})}^2] + \mE[\ltwo{\nabla g_{t,\mu_2}(\ty_{t,0}) - \tu_{t,0} }^2] + \frac{\ell\beta}{2-\ell\beta}\mE[\ltwo{\tu_{t,0}}^2] \\
	&\quad  + \frac{2}{\beta(m+1)}\left(\frac{\mu_2^2}{4\mu}\ell^2(d_2+3)^3  + \mu_2^2\ell d_2 \right).
	\end{flalign*}
\end{lemma}
\begin{proof}
	Taking summation of the result of \Cref{lemma16} over $t=\{0,\cdots m\}$ yields
	\begin{flalign}
		\sum_{k=0}^{m}\mE[\ltwo{\nabla g_{t,\mu_2}(\ty_{t,k}) - \tu_{t,k}}^2] \leq (m+1)\mE[\ltwo{\nabla g_{t,\mu_2}(\ty_{t,0}) - \tu_{t,0} }^2] + \frac{\ell\beta(m+1)}{2-\ell\beta}\mE[\ltwo{\tu_{t,0}}^2]. \label{eq: 44}
	\end{flalign}
	Combining \cref{eq: 44} with \Cref{lemma15} yields
	\begin{flalign}
		\sum_{k=0}^{m}\mE[\ltwo{\nabla g_{t,\mu_2}(\ty_{t,k})}^2] &\leq \frac{2}{\beta}\mE[g_{t,\mu_2}(\ty_{t,0}) - g_{t,\mu_2}(\ty_{t,m+1}) ] + (m+1)\mE[\ltwo{\nabla g_{t,\mu_2}(\ty_{t,0}) - \tu_{t,0} }^2] \nonumber\\
		&\quad + \frac{\ell\beta(m+1)}{2-\ell\beta}\mE[\ltwo{\tu_{t,0}}^2]. \label{eq: 45}
	\end{flalign}
	Dividing both sides of \cref{eq: 45} by $m$ and recalling the definition of $\tm_t$ in the output of \Cref{al:zoconcavemaximizer} yield
	\begin{flalign}
		\mE[\ltwo{\nabla g_{t,\mu_2}(\ty_{t,\tm_t})}^2]&\leq \frac{2}{\beta(m+1)}\mE[g_{t,\mu_2}(\ty_{t,0}) - g_{t,\mu_2}(\ty_{t,m+1}) ] + \mE[\ltwo{\nabla g_{t,\mu_2}(\ty_{t,0}) - \tu_{t,0} }^2] \nonumber\\
		&\quad + \frac{\ell\beta}{2-\ell\beta}\mE[\ltwo{\tu_{t,0}}^2]. \label{eq: 46}
	\end{flalign}
	We then bound the term $\mE[g_{t,\mu_2}(\ty_{t,0}) - g_{t,\mu_2}(\ty_{t,m+1}) ]$ as follows:
	\begin{flalign}
		&\mE[g_{t,\mu_2}(\ty_{t,0}) - g_{t,\mu_2}(\ty_{t,m+1}) ] \nonumber\\
		&=\mE[g_t(\ty_{t,0}) - g_t(\ty_{t,m+1})] + \mE[g_{t,\mu_2}(\ty_{t,0}) - g_t(\ty_{t,0}) ] + \mE[g_t(\ty_{t,m+1}) - g_{t,\mu_2}(\ty_{t,m+1}) ] \nonumber\\
		&\leq \mE[g_t(\ty_{t,0}) - g_t(\ty_{t,m+1})] + \mE[\lone{g_{t,\mu_2}(\ty_{t,0}) - g_t(\ty_{t,0})} ] + \mE[\lone{ g_{t,\mu_2}(\ty_{t,m+1}) - g_t(\ty_{t,m+1})} ] \nonumber\\
		&\overset{(i)}{\leq} \mE[g_t(\ty_{t,0}) - g_t(\ty_{t,m+1})] + \mu_2^2\ell d_2\nonumber\\
		&\leq \mE[g_t(\ty_{t,0}) - g_t(\ty^*_t)] + \mu_2^2\ell d_2\nonumber\\
		&\overset{(ii)}{\leq} \frac{1}{2\mu}\mE[\ltwo{\nabla g_t(\ty_{t,0})}^2]  + \mu_2^2\ell d_2 \nonumber\\
		&\leq \frac{1}{\mu}\mE[\ltwo{\nabla g_{t,\mu_2}(\ty_{t,0})}^2] + \frac{1}{\mu}\mE[\ltwo{\nabla g_{t,\mu_2}(\ty_{t,0}) - \nabla g_t(\ty_{t,0}) }^2]  + \mu_2^2\ell d_2 \nonumber\\
		&\overset{(iii)}{\leq} \frac{1}{\mu}\mE[\ltwo{\nabla g_{t,\mu_2}(\ty_{t,0})}^2] + \frac{\mu_2^2}{4\mu}\ell^2(d_2+3)^3  + \mu_2^2\ell d_2,\label{eq: 47}
	\end{flalign}
	where $(i)$ follows from \Cref{lemma10}, $(ii)$ follows from \cref{eq: 8} in \Cref{lemma2}, and $(iii)$ follows from \Cref{lemma11}. Substituting \cref{eq: 47} into \cref{eq: 46} yields
	\begin{flalign*}
		\mE[\ltwo{\nabla g_{t,\mu_2}(\ty_{t,\tm_t})}^2] &\leq \frac{2}{\beta\mu(m+1)}\mE[\ltwo{\nabla g_{t,\mu_2}(\ty_{t,0})}^2] + \mE[\ltwo{\nabla g_{t,\mu_2}(\ty_{t,0}) - \tu_{t,0} }^2] + \frac{\ell\beta}{2-\ell\beta}\mE[\ltwo{\tu_{t,0}}^2] \\
		&\quad  + \frac{2}{\beta(m+1)}\left(\frac{\mu_2^2}{4\mu}\ell^2(d_2+3)^3  + \mu_2^2\ell d_2 \right),
	\end{flalign*}
	which completes the proof.
\end{proof}

\begin{lemma}\label{lemma14}
	Consider \Cref{al:zoconcavemaximizer}. Let $S_{2,y}\geq 16\kappa(d_2+4)\ell\beta$ and $\beta\leq\frac{1}{6\ell}$. For any $t>0$, we have
	\begin{flalign*}
		\sum_{k=0}^{m} 	\mE[\ltwo{\tu_{t,k}}^2] \leq \frac{1}{1-b} \mE[\ltwo{\tu_{t,0}}^2] + \frac{m+1}{1-b}\left[ \frac{2\mu^2_2\ell\kappa}{\beta}(d_2+3)^3 + 7 \mu^2_2(d_2+6)^3\ell^2 \right],
	\end{flalign*}
	where $b=1 - \frac{\beta\mu\ell}{2(\mu+\ell)}$.
\end{lemma}
\begin{proof}
	The update of \Cref{al:zoconcavemaximizer} implies that
		\begin{flalign}
	&\mE[\ltwo{\tu_{t,k}}^2|\mf_{t,k}] \nonumber\\
	&= \ltwo{\tu_{t,k-1}}^2 + 2\mE[\langle \tu_{t,k-1}, H_{\mu_2}(\tx_{t,k},\ty_{t,k},\omega_{\mcm_2},\xi_{\mcm}) - H_{\mu_2}(\tx_{t,k-1},\ty_{t,k-1},\omega_{\mcm_2},\xi_{\mcm})  \rangle|\mf_{t,k}] \nonumber\\
	&\quad + \mE[\ltwo{H_{\mu_2}(\tx_{t,k},\ty_{t,k},\omega_{\mcm_2},\xi_{\mcm}) - H_{\mu_2}(\tx_{t,k-1},\ty_{t,k-1},\omega_{\mcm_2},\xi_{\mcm})}^2|\mf_{t,k}]\nonumber\\
	&= \ltwo{\tu_{t,k-1}}^2 + \frac{2}{\beta}\langle \ty_{t,k}-\ty_{t,k-1}, \nabla_y f_{\mu_2}(\tx_{t,k},\ty_{t,k}) - \nabla_y f_{\mu_2}(\tx_{t,k-1},\ty_{t,k-1})  \rangle \nonumber\\
	&\quad + \mE[\ltwo{H_{\mu_2}(\tx_{t,k},\ty_{t,k},\omega_{\mcm_2},\xi_{\mcm}) - H_{\mu_2}(\tx_{t,k-1},\ty_{t,k-1},\omega_{\mcm_2},\xi_{\mcm})}^2|\mf_{t,k}]\nonumber\\
	&= \ltwo{\tu_{t,k-1}}^2 + \frac{2}{\beta}\langle \ty_{t,k}-\ty_{t,k-1}, \nabla_y f(\tx_{t,k},\ty_{t,k}) - \nabla_y f(\tx_{t,k-1},\ty_{t,k-1})  \rangle \nonumber\\
	&\quad + \frac{2}{\beta}\langle \ty_{t,k}-\ty_{t,k-1}, \nabla_y f_{\mu_2}(\tx_{t,k},\ty_{t,k}) - \nabla_y f(\tx_{t,k},\ty_{t,k})  \rangle \nonumber\\
	&\quad + \frac{2}{\beta}\langle \ty_{t,k}-\ty_{t,k-1}, \nabla_y f(\tx_{t,k-1},\ty_{t,k-1}) - \nabla_y f_{\mu_2}(\tx_{t,k-1},\ty_{t,k-1})  \rangle \nonumber\\
	&\quad + \mE[\ltwo{H_{\mu_2}(\tx_{t,k},\ty_{t,k},\omega_{\mcm_2},\xi_{\mcm}) - H_{\mu_2}(\tx_{t,k-1},\ty_{t,k-1},\omega_{\mcm_2},\xi_{\mcm})}^2|\mf_{t,k}]\nonumber\\
	&\overset{(i)}{\leq} \ltwo{\tu_{t,k-1}}^2 - \frac{2}{\beta}\left( \frac{\mu\ell}{\mu+\ell} \ltwo{\ty_{t,k}-\ty_{t,k-1}}^2 + \frac{1}{\mu+\ell} \ltwo{\nabla_y f(\tx_{t,k},\ty_{t,k}) - \nabla_y f(\tx_{t,k-1},\ty_{t,k-1})}^2\right)\nonumber\\
	&\quad + \frac{2}{\beta}\left( \frac{\mu\ell}{4(\mu+\ell)} \ltwo{\ty_{t,k}-\ty_{t,k-1}}^2 + \frac{\mu+\ell}{\mu\ell} \ltwo{\nabla_y f_{\mu_2}(\tx_{t,k},\ty_{t,k}) - \nabla_y f(\tx_{t,k},\ty_{t,k}) }^2  \right) \nonumber\\
	&\quad + \frac{2}{\beta}\left( \frac{\mu\ell}{4(\mu+\ell)} \ltwo{\ty_{t,k}-\ty_{t,k-1}}^2 + \frac{\mu+\ell}{\mu\ell} \ltwo{\nabla_y f_{\mu_2}(\tx_{t,k-1},\ty_{t,k-1}) - \nabla_y f(\tx_{t,k-1},\ty_{t,k-1}) }^2  \right) \nonumber\\
	&\quad + \mE[\ltwo{H_{\mu_2}(\tx_{t,k},\ty_{t,k},\omega_{\mcm_2},\xi_{\mcm}) - H_{\mu_2}(\tx_{t,k-1},\ty_{t,k-1},\omega_{\mcm_2},\xi_{\mcm})}^2|\mf_{t,k}]\nonumber\\
	&\overset{(ii)}{\leq} \ltwo{\tu_{t,k-1}}^2 - \frac{\mu\ell}{\beta(\mu+\ell)} \ltwo{\ty_{t,k}-\ty_{t,k-1}}^2 - \frac{2}{\beta(\mu+\ell)} \ltwo{\nabla_y f(\tx_{t,k},\ty_{t,k}) - \nabla_y f(\tx_{t,k-1},\ty_{t,k-1})}^2  \nonumber\\
	&\quad + \frac{\mu^2_2\ell(\mu+\ell)}{\beta\mu} (d_2+3)^3 + \mE[\ltwo{H_{\mu_2}(\tx_{t,k},\ty_{t,k},\omega_{\mcm_2},\xi_{\mcm}) - H_{\mu_2}(\tx_{t,k-1},\ty_{t,k-1},\omega_{\mcm_2},\xi_{\mcm})}^2|\mf_{t,k}]\nonumber\\
	&\leq \left(1 - \frac{\beta\mu\ell}{\mu+\ell}\right) \ltwo{\tu_{t,k-1}}^2 - \frac{2}{\beta(\mu+\ell)} \ltwo{\nabla_y f(\tx_{t,k},\ty_{t,k}) - \nabla_y f(\tx_{t,k-1},\ty_{t,k-1})}^2  \nonumber\\
	&\quad + 2\mE[\nltwo{H_{\mu_2}(\tx_{t,k},\ty_{t,k},\omega_{\mcm_2},\xi_{\mcm}) - H_{\mu_2}(\tx_{t,k-1},\ty_{t,k-1},\omega_{\mcm_2},\xi_{\mcm}) \nonumber\\
		&\qquad\qquad - (\nabla_y f_{\mu_2}(\tx_{t,k},\ty_{t,k}) - \nabla_y f_{\mu_2}(\tx_{t,k-1},\ty_{t,k-1})) }^2|\mf_{t,k}]\nonumber\\
	&\quad + 2\mE[\ltwo{\nabla_y f_{\mu_2}(\tx_{t,k},\ty_{t,k}) - \nabla_y f_{\mu_2}(\tx_{t,k-1},\ty_{t,k-1})}^2|\mf_{t,k}] + \frac{\mu^2_2\ell(\mu+\ell)}{\beta\mu} (d_2+3)^3 \nonumber\\
	&\leq \left(1 - \frac{\beta\mu\ell}{\mu+\ell}\right) \ltwo{\tu_{t,k-1}}^2 - \frac{2}{\beta(\mu+\ell)} \ltwo{\nabla_y f(\tx_{t,k},\ty_{t,k}) - \nabla_y f(\tx_{t,k-1},\ty_{t,k-1})}^2  \nonumber\\
	&\quad + 2\mE[\nltwo{H_{\mu_2}(\tx_{t,k},\ty_{t,k},\omega_{\mcm_2},\xi_{\mcm}) - H_{\mu_2}(\tx_{t,k-1},\ty_{t,k-1},\omega_{\mcm_2},\xi_{\mcm}) \nonumber\\
		&\qquad\qquad - (\nabla_y f_{\mu_2}(\tx_{t,k},\ty_{t,k}) - \nabla_y f_{\mu_2}(\tx_{t,k-1},\ty_{t,k-1})) }^2|\mf_{t,k}]\nonumber\\
	&\quad + 6\mE[\ltwo{\nabla_y f(\tx_{t,k},\ty_{t,k}) - \nabla_y f(\tx_{t,k-1},\ty_{t,k-1})}^2|\mf_{t,k}]\nonumber\\
	&\quad + 6\mE[\ltwo{\nabla_y f(\tx_{t,k-1},\ty_{t,k-1}) - \nabla_y f_{\mu_2}(\tx_{t,k-1},\ty_{t,k-1})}^2|\mf_{t,k}] \nonumber\\ 
	&\quad + 6\mE[\ltwo{\nabla_y f_{\mu_2}(\tx_{t,k},\ty_{t,k}) - \nabla_y f(\tx_{t,k},\ty_{t,k})}^2|\mf_{t,k}] + \frac{\mu^2_2\ell(\mu+\ell)}{\beta\mu} (d_2+3)^3 \nonumber\\
	&\leq \left(1 - \frac{\beta\mu\ell}{\mu+\ell}\right) \ltwo{\tu_{t,k-1}}^2 - \left(\frac{2}{\beta(\mu+\ell)} - 6 \right)\ltwo{\nabla_y f(\tx_{t,k},\ty_{t,k}) - \nabla_y f(\tx_{t,k-1},\ty_{t,k-1})}^2  \nonumber\\
	&\quad + \frac{2}{S_{2,y}}\mE[\ltwo{H_{\mu_2}(\tx_{t,k},\ty_{t,k},\omega_{i},\xi_{i}) - H_{\mu_2}(\tx_{t,k-1},\ty_{t,k-1},\omega_{i},\xi_{i})}^2|\mf_{t,k}] \nonumber\\
	&\quad + 3\mu^2_2\ell^2(d_2+3)^3 + \frac{\mu^2_2\ell(\mu+\ell)}{\beta\mu} (d_2+3)^3 \nonumber\\
	&\overset{(iv)}{\leq} \left(1 - \frac{\beta\mu\ell}{\mu+\ell}\right) \ltwo{\tu_{t,k-1}}^2 + \frac{2}{S_{2,y}}\mE[\ltwo{H_{\mu_2}(\tx_{t,k},\ty_{t,k},\omega_{i},\xi_{i}) - H_{\mu_2}(\tx_{t,k-1},\ty_{t,k-1},\omega_{i},\xi_{i})}^2|\mf_{t,k}] \nonumber\\
	&\quad + 3\mu^2_2\ell^2(d_2+3)^3 + \frac{\mu^2_2\ell(\mu+\ell)}{\beta\mu} (d_2+3)^3 \nonumber\\
	&\overset{(v)}{\leq} \left(1 - \frac{\beta\mu\ell}{\mu+\ell}\right) \ltwo{\tu_{t,k-1}}^2 + \frac{2}{S_{2,y}} \left[ 2(d_2+4)\ell^2\beta^2\ltwo{\tu_{t,k-1}}^2 + 2\mu^2_2(d_2+6)^3\ell^2 \right] \nonumber\\
	&\quad + 3\mu^2_2\ell^2(d_2+3)^3 + \frac{\mu^2_2\ell(\mu+\ell)}{\beta\mu} (d_2+3)^3 \nonumber\\
	&= \left(1 - \frac{\beta\mu\ell}{\mu+\ell} + \frac{4}{S_{2,y}}(d_2+4)\ell^2\beta^2 \right) \ltwo{\tu_{t,k-1}}^2 \nonumber\\
	&\quad + \frac{4}{S_{2,y}} \mu^2_2(d_2+6)^3\ell^2 + 3\mu^2_2\ell^2(d_2+3)^3 + \frac{\mu^2_2\ell(\mu+\ell)}{\beta\mu} (d_2+3)^3 \nonumber\\
	&\overset{(vi)}{\leq} \left(1 - \frac{\beta\mu\ell}{2(\mu+\ell)} \right) \ltwo{\tu_{t,k-1}}^2 + \frac{\mu^2_2\ell(1+\kappa)}{\beta} (d_2+3)^3 + 7 \mu^2_2(d_2+6)^3\ell^2.\label{eq: 41}
	\end{flalign}
	where $(i)$ follows from \cref{eq: 7} in \Cref{lemma2} and Young's inequality, $(ii)$ follows from \Cref{lemma11}, $(iii)$ follows from Lemma 1 in \cite{fang2018spider}, $(iv)$ follows from the fact that $\frac{2}{\beta(\mu+\ell)}-6>0$, 
	$(v)$ follows from \Cref{lemma6}, and $(vi)$ follows from the fact that $\frac{4}{S_{2,y}}(d_2+4)\ell^2\beta^2\leq \frac{\beta\mu\ell}{2(\mu+\ell)}$. Taking expectation on both sides of \cref{eq: 41} and applying \cref{eq: 41} iteratively yield
	\begin{flalign}
		\mE[\ltwo{\tu_{t,k}}^2] &\leq b^k\mE[\ltwo{\tu_{t,0}}^2] + \left[ \frac{2\mu^2_2\ell\kappa}{\beta}(d_2+3)^3 + 7 \mu^2_2(d_2+6)^3\ell^2 \right]\sum_{j=0}^{k-1}b^j. \label{eq: 42}
	\end{flalign}
	Taking summation of \cref{eq: 42} over $k=\{0,\cdots m\}$ yields
	\begin{flalign*}
		\sum_{k=0}^{m} 	\mE[\ltwo{\tu_{t,k}}^2] &\leq \mE[\ltwo{\tu_{t,0}}^2] \sum_{k=0}^{m}b^k +  \left[ \frac{2\mu^2_2\ell\kappa}{\beta}(d_2+3)^3 + 7 \mu^2_2(d_2+6)^3\ell^2 \right]\sum_{k=0}^{m}\sum_{j=0}^{k-1}b^j\\
		&\leq \frac{1}{1-b} \mE[\ltwo{\tu_{t,0}}^2] + \frac{m+1}{1-b}\left[ \frac{2\mu^2_2\ell\kappa}{\beta}(d_2+3)^3 + 7 \mu^2_2(d_2+6)^3\ell^2 \right],
	\end{flalign*}
	which completes the proof.
\end{proof}

\section{Proof of \Cref{thm: zo-sarah}}\label{sc: zoisarah}

	Following steps similar to those in Lemmas \ref{lemma15}-\ref{lemma17}, at the $t$-th outer-loop iteration, we obtain the following convergence result for the inner loop:
	\begin{flalign}
	&\mE[\ltwo{\nabla p_{\tau}(\tilde{w}_t)}^2] \nonumber\\
	&\leq \frac{2}{\gamma\tau(I+1)}\mE[\ltwo{\nabla p_{\tau}(w_0)}^2] + \mE[\ltwo{\nabla p_{\tau}(w_0) - v_0 }^2] + \frac{\ell\gamma}{2-\ell\gamma}\mE[\ltwo{v_0}^2] \nonumber\\
	&\quad  + \frac{2}{\gamma(I+1)}\left(\frac{\tau^2}{4\mu}\ell^2(d+3)^3  + \tau^2\ell d \right)\nonumber\\
	&\leq \left( \frac{2}{\gamma\mu(I+1)} + \frac{2\ell\gamma}{2-\ell\gamma} \right)\mE[\ltwo{\nabla p_{\tau}(w_0)}^2] + \left( 1 + \frac{2\ell\gamma}{2-\ell\gamma} \right)\mE[\ltwo{\nabla p_{\tau}(w_0) - v_0 }^2]\nonumber\\
	&\quad + \frac{2}{\gamma(I+1)}\left(\frac{\tau^2}{4\mu}\ell^2(d+3)^3  + \tau^2\ell d \right).\label{eq: 74}
	\end{flalign}
	Then, following steps similar to those in \Cref{lemma19}, we can obtain
	\begin{flalign}
	\mE[\ltwo{\nabla p_{\tau}(w_0) - v_0 }^2] \leq \frac{2\sigma^2}{B_1} +  \frac{d\ell^2\delta^2}{2} + \frac{\tau^2}{2}\ell^2(d+3)^3.\label{eq: 75}
	\end{flalign}
	Letting $\gamma=\frac{2}{9\ell}$, $I=36\kappa-1$, substituting \cref{eq: 75} into \cref{eq: 74}, and recalling the fact that $w_I=\tilde{w}_t$ and $w_0=\tilde{w}_{t-1}$ yield
	\begin{flalign}
		\mE[\ltwo{\nabla p_{\tau}(\tilde{w}_t)}^2]\leq \frac{1}{2}\mE[\ltwo{\nabla p_{\tau}(\tilde{w}_{t-1})}^2] + \frac{5\sigma^2}{2B_1} +  \frac{5d\ell^2\delta^2}{8} + \frac{11\tau^2}{16}\ell^2(d+3)^3 + \frac{\tau^2}{4}\ell\mu d.\label{eq: 76}
	\end{flalign}
	Applying \cref{eq: 76} iteratively from $t=T$ to $0$ yields
	\begin{flalign}
			\mE[\ltwo{\nabla p_{\tau}(\tilde{w}_T)}^2] &\leq \frac{1}{2^T}\ltwo{\nabla p_{\tau}(\tilde{w}_0)}^2 +  \frac{5\sigma^2}{2B_1} \sum_{t=0}^{T-1}\frac{1}{2^t} \nonumber\\
			&\quad + \left(\frac{5d\ell^2\delta^2}{8} + \frac{11\tau^2}{16}\ell^2(d+3)^3 + \frac{\tau^2}{4}\ell\mu d\right)\sum_{t=0}^{T-1}\frac{1}{2^t}\nonumber\\
			&\leq \frac{1}{2^T}\ltwo{\nabla p_{\tau}(\tilde{w}_0)}^2 +  \frac{5\sigma^2}{B_1} + \frac{5d\ell^2\delta^2}{4} + \frac{11\tau^2}{8}\ell^2(d+3)^3 + \frac{\tau^2}{2}\ell\mu d.\label{eq: 85}
	\end{flalign}
	Letting $T=\log_2\frac{5\ltwo{\nabla p_{\tau}(\tilde{w}_0)}^2}{\epsilon}$, $B_1=\frac{25\sigma^2}{\epsilon}$, $\delta=\frac{2\epsilon^{0.5}}{5\ell d^{0.5}}$, and $\tau=\min\{ \frac{\epsilon^{0.5}}{3\ell(d+3)^{1.5}}, \sqrt{\frac{2\epsilon}{5\ell\mu d}} \}$, we have
	\begin{flalign*}
		\mE[\ltwo{\nabla p_{\tau}(\tilde{w}_T)}^2]\leq \epsilon.
	\end{flalign*}
	The total sample complexity is given by
	\begin{flalign*}
		T\cdot (I \cdot B_2 + d\cdot B_1) = \mathcal{O}\left( d  \left(\kappa + \frac{1}{\epsilon}\right) \log\left( \frac{1}{\epsilon} \right) \right).
	\end{flalign*}
	
	\textbf{Extension to the finite-sum case: }ZO-iSARAH in \Cref{al:zoisarah} is also applicable to strongly-convex optimization in the finite-sum case, where the objective function takes the form given by
	\begin{flalign}
	\min_{w\in\mR^d}p(w)\triangleq \frac{1}{n}\sum_{i=1}^{n}P(w;\xi_i). \label{eq: 84}
	\end{flalign}
	To solve the problem in \cref{eq: 84}, we slightly modify \Cref{al:zoisarah} by replacing line 5 with the full gradient. Following steps similar to those from \cref{eq: 74} to \cref{eq: 85}, we have
	\begin{flalign*}
		\mE[\ltwo{\nabla p_{\tau}(\tilde{w}_T)}^2] \leq \frac{1}{2^T}\ltwo{\nabla p_{\tau}(\tilde{w}_0)}^2 + \frac{5d\ell^2\delta^2}{4} + \frac{11\tau^2}{8}\ell^2(d+3)^3 + \frac{\tau^2}{2}\ell\mu d.
	\end{flalign*}
	Letting $T=\log_2\frac{4\ltwo{\nabla p_{\tau}(\tilde{w}_0)}^2}{\epsilon}$, $\delta=\frac{\epsilon^{0.5}}{3\ell d^{0.5}}$, and $\tau=\min\{ \frac{\epsilon^{0.5}}{3\ell(d+3)^{1.5}}, \sqrt{\frac{\epsilon}{2\ell\mu d}} \}$, we have
	\begin{flalign*}
	\mE[\ltwo{\nabla p_{\tau}(\tilde{w}_T)}^2]\leq \epsilon.
	\end{flalign*}
	The total sample complexity is given by
	\begin{flalign}
	T\cdot (I \cdot B_2 + d\cdot n) = \mathcal{O}\left( d \left(\kappa + n\right) \log\left( \frac{1}{\epsilon} \right) \right).\label{eq: 86}
	\end{flalign}

Let $P(\cdot;\xi)=-F(x_0,\cdot;\xi)$. Then we can conclude that the sample complexity for the initialization of \Cref{al:zosredaboost} is given by $\mathcal{O}\left( d_2\kappa\log\left( \kappa \right) \right)$ in the online case, and is given by $\mathcal{O}\left(d_2 (\kappa+n)\log\left( \kappa \right) \right)$ in the finite-sum case.

\section{Proof of \Cref{thm2}}\label{sc: proofofthm2}
\subsection{Proof of Supporting Lemmas}
We define $\Delta^\prime_t=\mE[\ltwo{\nabla_x f_{\mu_1}(x_t,y_t)-v_t}^2] + \mE[\ltwo{\nabla_y f_{\mu_2}(x_t,y_t)-u_t}^2]$, $\wtDelta^\prime_{t,k}=\mE[\ltwo{\nabla_x f_{\mu_1}(\tx_{t,k},\ty_{t,k})-\tv_{t,k}}^2] + \mE[\ltwo{\nabla_y f_{\mu_2}(\tx_{t,k},\ty_{t,k})-\tu_{t,k}}^2]$, and $\delta^\prime_t=\mE[\ltwo{\nabla_y f_{\mu_2}(x_t,y_t)}^2]$. In this subsection, we establish the following lemmas to characterize the relationship between $\Delta_t$ and $\Delta_t^\prime$, and $\delta_t$ and $\delta^\prime_t$, and the recursive relationship of $\Delta^\prime_t$ and $\delta^\prime_t$, which are crucial for the analysis of \Cref{thm2}.
\begin{lemma}\label{lemma18}
	Suppose Assumption \ref{ass2} holds. Then, for any $0\leq t\leq T-1$, we have
	\begin{flalign*}
	\Delta_t \leq 2\Delta^\prime_t + \frac{\mu^2_1}{2}\ell^2(d_1+3)^3 + \frac{\mu^2_2}{2}\ell^2(d_2+3)^3,
	\end{flalign*}
	and
	\begin{flalign*}
	\delta_t \leq 2\delta^\prime_t + \frac{\mu^2_2}{2}\ell^2(d_2+3)^3.
	\end{flalign*}
\end{lemma}
\begin{proof}
	For the first inequality, we have
	\begin{flalign}
	\Delta_t&=\mE[\ltwo{\nabla_x f(x_t,y_t)-v_t}^2] + \mE[\ltwo{\nabla_y f(x_t,y_t)-u_t}^2]\nonumber\\
	&=\mE[\ltwo{\nabla_x f_{\mu_1}(x_t,y_t)-v_t + \nabla_x f(x_t,y_t) - \nabla_x f_{\mu_1}(x_t,y_t) }^2] \nonumber\\
	&\quad + \mE[\ltwo{\nabla_y f_{\mu_2}(x_t,y_t)-u_t + \nabla_y f(x_t,y_t) - \nabla_y f_{\mu_2}(x_t,y_t) }^2]\nonumber\\
	&\leq 2\mE[\ltwo{\nabla_x f_{\mu_1}(x_t,y_t)-v_t }^2] + 2\mE[\ltwo{\nabla_y f_{\mu_2}(x_t,y_t)-u_t  }^2] \nonumber\\
	&\quad + 2\mE[\ltwo{\nabla_x f(x_t,y_t) - \nabla_x f_{\mu_1}(x_t,y_t) }^2] + 2\mE[\ltwo{\nabla_y f(x_t,y_t) - \nabla_y f_{\mu_2}(x_t,y_t) }^2]\nonumber\\
	&\overset{(i)}{\leq} 2\Delta^\prime_t + \frac{\mu^2_1}{2}\ell^2(d_1+3)^3 + \frac{\mu^2_2}{2}\ell^2(d_2+3)^3,\nonumber
	\end{flalign}
	where $(i)$ follows from \Cref{lemma11}. For the second inequality, we have
	\begin{flalign}
	\delta_t&=\mE[\ltwo{\nabla_y f(x_t,y_t)}^2] = \mE[\ltwo{\nabla_y f_{\mu_2}(x_t,y_t) + \nabla_y f(x_t,y_t) - \nabla_y f_{\mu_2}(x_t,y_t)}^2] \nonumber\\
	&\leq 2\mE[\ltwo{\nabla_y f_{\mu_2}(x_t,y_t)}^2] + 2\mE[\ltwo{\nabla_y f(x_t,y_t) - \nabla_y f_{\mu_2}(x_t,y_t)}^2] \nonumber\\
	&\overset{(i)}{\leq} 2\delta^\prime_t + \frac{\mu^2 _2}{2}\ell^2(d_2+3)^3,\nonumber
	\end{flalign}
	where $(i)$ follows from \Cref{lemma11}.
\end{proof}
We provide the following two lemmas to characterize the relationship between $\delta^\prime_{t}$ and $\delta^\prime_{t-1}$ as well as that between $\Delta^\prime_{t}$ and $\Delta^\prime_{t-1}$.
\begin{lemma}\label{lemma7}
	Suppose Assumption \ref{ass2} holds. Then, we have
	\begin{flalign*}
	\Delta^\prime_t &\leq \left[ 1 + \frac{6\ell^2\beta^2}{1-b}\left(\frac{d_1+4}{S_{2,x}} + \frac{d_2+4}{S_{2,y}} \right) \right]\Delta^\prime_{t-1} + \frac{6\ell^2\beta^2}{1-b}\left(\frac{d_1+4}{S_{2,x}} + \frac{d_2+4}{S_{2,y}} \right)\delta^\prime_{t-1} \nonumber\\
	&\quad + 2\ell^2\alpha^2\left(\frac{d_1+4}{S_{2,x}} + \frac{d_2+4}{S_{2,y}} \right) \left( 1 + \frac{9\ell^2\beta^2}{1-b}\right)\mE[\ltwo{v_{t-1}}^2] + \pi_\Delta(d_1,d_2,\mu_1,\mu_2),
	\end{flalign*}
	where $b=1 - \frac{\beta\mu\ell}{2(\mu+\ell)}$ and
	\begin{flalign*}
	&\pi_\Delta(d_1,d_2,\mu_1,\mu_2)\nonumber\\
	&=\frac{2\ell^2\beta^2}{1-b} \left(\frac{d_1+4}{S_{2,x}} + \frac{d_2+4}{S_{2,y}} \right)\Bigg\{ 6\ell^2\left[ \frac{\mu^2_1(d_1+6)^3}{S_{2,x}} + \frac{\mu^2_2(d_2+6)^3}{S_{2,y}} \right] +  (m+1)\Big( \frac{2\mu^2_2\ell\kappa}{\beta}(d_2+3)^3  \nonumber\\
	&\quad + 7 \mu^2_2(d_2+6)^3\ell^2 \Big)  \Bigg\} + \frac{2(m+2) \mu^2_1(d_1+6)^3\ell^2}{S_{2,x}} +  \frac{2(m+2) \mu^2_2(d_2+6)^3\ell^2}{S_{2,y}}.
	\end{flalign*}
	Moreover, if we let $\beta=\frac{2}{13\ell}$, $m=104\kappa-1$, $S_{2,x}\geq 5600(d_1+4)$ and $S_{2,y}\geq 5600(d_2+4)$, then we have
	\begin{flalign*}
	\pi_\Delta(d_1,d_2,\mu_1,\mu_2) \leq \kappa^3\ell^2[\mu_1^2(d_1+6)^3 + \mu_2^2(d_2+6)^3].
	\end{flalign*}
\end{lemma}
\begin{proof}
	We proceed as follows:
	\begin{flalign}
	&\Delta^\prime_t\nonumber\\
	&=\wtDelta^\prime_{t-1,\bar{m}_{t-1}}\nonumber\\
	&=\mE\Big[\ltwo{\nabla_x f_{\mu_1}(\tx_{t-1,\tm_{t-1}},\ty_{t-1,\tm_{t-1}}) - \tv_{t-1,\tm_{t-1}} }^2\Big]\nonumber\\
	&\overset{(i)}{\leq} \mE\Big[\ltwo{\nabla_x f_{\mu_1}(\tx_{t-1,\tm_{t-1}-1},\ty_{t-1,\tm_{t-1}-1}) - \tv_{t-1,\tm_{t-1}-1} }^2\Big] \nonumber\\
	&\quad + \frac{1}{S_{2,x}} \mE\Big[ \ltwo{G_{\mu_1}(\tx_{t-1,\tm_{t-1}},\ty_{t-1,\tm_{t-1}},\nu_{i},\xi_{i}) - G_{\mu_1}(\tx_{t-1,\tm_{t-1}-1},\ty_{t-1,\tm_{t-1}-1},\nu_{i},\xi_{i})}^2 \Big] \nonumber\\
	&\overset{(ii)}{\leq} \mE\Big[\ltwo{\nabla_x f_{\mu_1}(\tx_{t-1,\tm_{t-1}-1},\ty_{t-1,\tm_{t-1}-1}) - \tv_{t-1,\tm_{t-1}-1} }^2\Big] \nonumber \\
	&\quad + \frac{1}{S_{2,x}}\left[2(d_1+4)\ell^2\beta^2\mE[\ltwo{\tu_{t-1,\tm_{t-1}-1}}^2] + 2\mu^2_1(d_1+6)^3\ell^2\right],\label{eq: 34}
	\end{flalign}
	where $(i)$ follows from \Cref{lemma3}, and $(ii)$ follows from \Cref{lemma6}. Applying \cref{eq: 34} recursively yields
	\begin{flalign}
	&\mE\Big[\ltwo{\nabla_x f_{\mu_1}(\tx_{t-1,\tm_{t-1}},\ty_{t-1,\tm_{t-1}}) - \tv_{t-1,\tm_{t-1}} }^2\Big]\nonumber\\
	&\leq \mE\Big[\ltwo{\nabla_x f_{\mu_1}(\tx_{t-1,0},\ty_{t-1,0}) - \tv_{t-1,0} }^2\Big] + \frac{2(d_1+4)\ell^2\beta^2}{S_{2,x}}\sum_{k=0}^{\tm_{t-1}-1}\mE[\ltwo{\tu_{t-1,k}}^2]  \nonumber \\
	&\quad + \frac{2\tm_{t-1}\mu^2_1(d_1+6)^3\ell^2}{S_{2,x}}\nonumber\\
	&\leq \mE\Big[\ltwo{\nabla_x f_{\mu_1}(\tx_{t-1,0},\ty_{t-1,0}) - \tv_{t-1,0} }^2\Big] + \frac{2(d_1+4)\ell^2\beta^2}{S_{2,x}}\sum_{k=0}^{m}\mE[\ltwo{\tu_{t-1,k}}^2] \nonumber \\
	&\quad + \frac{2(m+1)\mu^2_1(d_1+6)^3\ell^2}{S_{2,x}}.\label{eq: 35}
	\end{flalign}
	Similarly, we obtain
	\begin{flalign}
	&\mE\Big[\ltwo{\nabla_y f_{\mu_2}(\tx_{t-1,\tm_{t-1}},\ty_{t-1,\tm_{t-1}}) - \tu_{t-1,\tm_{t-1}} }^2\Big]\nonumber\\
	&\leq \mE\Big[\ltwo{\nabla_y f_{\mu_2}(\tx_{t-1,0},\ty_{t-1,0}) - \tu_{t-1,0} }^2\Big] + \frac{2(d_2+4)\ell^2\beta^2}{S_{2,y}}\sum_{k=0}^{m}\mE[\ltwo{\tu_{t-1,k}}^2] \nonumber \\
	&\quad + \frac{2(m+1)\mu^2_2(d_2+6)^3\ell^2}{S_{2,y}}.\label{eq: 36}
	\end{flalign}
	Combining \cref{eq: 35} and \cref{eq: 36} yields
	\begin{flalign}
	\Delta^\prime_t&\leq \wtDelta^\prime_{t-1,0} + \left(\frac{2(d_1+4)\ell^2\beta^2}{S_{2,x}} + \frac{2(d_2+4)\ell^2\beta^2}{S_{2,y}}\right)\sum_{k=0}^{m}\mE[\ltwo{\tu_{t-1,k}}^2] \nonumber\\
	&\quad+ \frac{2(m+1) \mu^2_1(d_1+6)^3\ell^2}{S_{2,x}} +  \frac{2(m+1) \mu^2_2(d_2+6)^3\ell^2}{S_{2,y}} .\label{eq: 37}
	\end{flalign}
	For $\wtDelta^\prime_{t-1,0}$, we obtain
	\begin{flalign}
	\wtDelta^\prime_{t-1,0}&=\mE[\ltwo{\nabla_x f_{\mu_1}(\tx_{t-1,0},\ty_{t-1,0})-\tv_{t-1,0}}^2] + \mE[\ltwo{\nabla_y f_{\mu_2}(\tx_{t-1,0},\ty_{t-1,0})-\tu_{t-1,0}}^2]\nonumber\\
	&\overset{(i)}{\leq} \mE[\ltwo{\nabla_x f_{\mu_1}(\tx_{t-1,-1},\ty_{t-1,-1})-\tv_{t-1,-1}}^2] + \mE[\ltwo{\nabla_y f_{\mu_2}(\tx_{t-1,-1},\ty_{t-1,-1})-\tu_{t-1,-1}}^2]\nonumber\\
	&\quad + \frac{1}{S_{2,x}}\mE[\ltwo{G(\tx_{t,0},\ty_{t,0},\nu_{i},\xi_{i}) - G(\tx_{t,-1},\ty_{t,-1},\nu_{\mcm_{i}},\xi_{i})}^2] \nonumber\\
	&\quad + \frac{1}{S_{2,y}}\mE[\ltwo{H(\tx_{t,0},\ty_{t,0},\nu_{i},\xi_{i}) - H(\tx_{t,-1},\ty_{t,-1},\nu_{\mcm_{i}},\xi_{i})}^2] \nonumber\\
	&\overset{(ii)}{\leq} \Delta^\prime_{t-1} + \left(\frac{2(d_1+4)\ell^2\alpha^2}{S_{2,x}} + \frac{2(d_2+4)\ell^2\alpha^2}{S_{2,y}}\right) \mE[\ltwo{v_{t-1}}^2] \nonumber\\
	&\quad + \frac{2\mu^2_1(d_1+6)^3\ell^2}{S_{2,x}} + \frac{2\mu^2_2(d_2+6)^3\ell^2}{S_{2,y}} ,\label{eq: 38}
	\end{flalign}
	where $(i)$ follows from \Cref{lemma3} and $(ii)$ follows from \Cref{lemma6}. Substituting \cref{eq: 38} into \cref{eq: 37} yields
	\begin{flalign}
	\Delta^\prime_t &\leq \Delta^\prime_{t-1} + \left(\frac{2(d_1+4)\ell^2\alpha^2}{S_{2,x}} + \frac{2(d_2+4)\ell^2\alpha^2}{S_{2,y}}\right)\mE[\ltwo{v_{t-1}}^2] \nonumber\\
	&\quad + \left(\frac{2(d_1+4)\ell^2\beta^2}{S_{2,x}} + \frac{2(d_2+4)\ell^2\beta^2}{S_{2,y}}\right)\sum_{k=0}^{m}\mE[\ltwo{\tu_{t-1,k}}^2]  \nonumber\\
	&\quad + \frac{2(m+2) \mu^2_1(d_1+6)^3\ell^2}{S_{2,x}} +  \frac{2(m+2) \mu^2_2(d_2+6)^3\ell^2}{S_{2,y}} \nonumber\\
	&\overset{(i)}{\leq} \Delta^\prime_{t-1} + \left(\frac{2(d_1+4)\ell^2\alpha^2}{S_{2,x}} + \frac{2(d_2+4)\ell^2\alpha^2}{S_{2,y}}\right)\mE[\ltwo{v_{t-1}}^2]  \nonumber\\
	&\quad + \frac{2(m+2) \mu^2_1(d_1+6)^3\ell^2}{S_{2,x}} +  \frac{2(m+2) \mu^2_2(d_2+6)^3\ell^2}{S_{2,y}}\nonumber\\
	&\quad + \frac{2\ell^2\beta^2}{1-b}\left(\frac{d_1+4}{S_{2,x}} + \frac{d_2+4}{S_{2,y}}\right)\left[  \mE[\ltwo{\tu_{t,0}}^2] + (m+1)\left( \frac{2\mu^2_2\ell\kappa}{\beta}(d_2+3)^3 + 7 \mu^2_2(d_2+6)^3\ell^2 \right) \right].\label{eq: 39}
	\end{flalign}
	where $(i)$ follows from \Cref{lemma14}. We next bound the term $\mE[\ltwo{\tu_{t-1,0}}^2]$ as follows:
	\begin{flalign}
	&\mE[\ltwo{\tu_{t-1,0}}^2]\nonumber\\
	&=\mE[\ltwo{\tu_{t-1,0} - \nabla_y f_{\mu_2}(x_t, y_{t-1}) + \nabla_y f_{\mu_2}(x_t, y_{t-1}) - \nabla_y f_{\mu_2}(x_{t-1}, y_{t-1}) + \nabla_y f_{\mu_2}(x_{t-1}, y_{t-1}) }^2]\nonumber\\
	&\leq 3\mE[\ltwo{\tu_{t-1,0} - \nabla_y f_{\mu_2}(x_t, y_{t-1})}^2] + 3\mE[\ltwo{\nabla_y f_{\mu_2}(x_t, y_{t-1}) - \nabla_y f_{\mu_2}(x_{t-1}, y_{t-1}) }^2] \nonumber\\
	&\quad + 3\mE[\ltwo{\nabla_y f_{\mu_2}(x_{t-1}, y_{t-1})}^2]\nonumber\\
	&\overset{(i)}{\leq} 3\mE[\ltwo{\tu_{t-1,0} - \nabla_y f_{\mu_2}(x_t, y_{t-1})}^2] + 3\ell^2\mE[\ltwo{x_t-x_{t-1}}^2] + 3\delta^\prime_{t-1} \nonumber\\
	&=3\mE[\ltwo{\tu_{t-1,0} - \nabla_y f_{\mu_2}(\tx_{t-1,0}, \ty_{t-1,0})}^2] + 3\alpha^2\ell^2\mE[\ltwo{v_{t-1}}^2] + 3\delta^\prime_{t-1}\nonumber\\
	&\leq 3\wtDelta^\prime_{t-1,0} + 3\alpha^2\ell^2\mE[\ltwo{v_{t-1}}^2] + 3\delta^\prime_{t-1}\nonumber\\
	&\overset{(ii)}{\leq} 3\Delta^\prime_{t-1} + 3\delta^\prime_{t-1} + \left[3+\frac{6(d_1+4)}{S_{2,x}} +\frac{6(d_2+4)}{S_{2,y}} \right]\alpha^2\ell^2\mE[\ltwo{v_{t-1}}^2] + 6\ell^2\left[ \frac{\mu^2_1(d_1+6)^3}{S_{2,x}} + \frac{\mu^2_2(d_2+6)^3}{S_{2,y}} \right]\nonumber\\
	&\overset{(iii)}{\leq} 3\Delta^\prime_{t-1} + 3\delta^\prime_{t-1} + 9\alpha^2\ell^2\mE[\ltwo{v_{t-1}}^2] + 6\ell^2\left[ \frac{\mu^2_1(d_1+6)^3}{S_{2,x}} + \frac{\mu^2_2(d_2+6)^3}{S_{2,y}} \right]\label{eq: 43}
	\end{flalign}
	where $(i)$ follows from \Cref{lemma9}, and $(ii)$ follows from \cref{eq: 38}, and $(iii)$ follows from the fact that $S_{2,x}\geq 2(d_1+4)$ and $S_{2,y}\geq 2(d_2+4)$. Substituting \cref{eq: 43} into \cref{eq: 39} yields
	\begin{flalign}
	\Delta^\prime_t &\leq \left[ 1 + \frac{6\ell^2\beta^2}{1-b}\left(\frac{d_1+4}{S_{2,x}} + \frac{d_2+4}{S_{2,y}} \right) \right]\Delta^\prime_{t-1} + \frac{6\ell^2\beta^2}{1-b}\left(\frac{d_1+4}{S_{2,x}} + \frac{d_2+4}{S_{2,y}} \right)\delta^\prime_{t-1} \nonumber\\
	&\quad + 2\ell^2\alpha^2\left(\frac{d_1+4}{S_{2,x}} + \frac{d_2+4}{S_{2,y}} \right) \left( 1 + \frac{9\ell^2\beta^2}{1-b}\right)\mE[\ltwo{v_{t-1}}^2] \nonumber\\
	&\quad + \frac{2\ell^2\beta^2}{1-b} \left(\frac{d_1+4}{S_{2,x}} + \frac{d_2+4}{S_{2,y}} \right)\Bigg\{ 6\ell^2\left[ \frac{\mu^2_1(d_1+6)^3}{S_{2,x}} + \frac{\mu^2_2(d_2+6)^3}{S_{2,y}} \right] +  (m+1)\Big( \frac{2\mu^2_2\ell\kappa}{\beta}(d_2+3)^3  \nonumber\\
	&\quad + 7 \mu^2_2(d_2+6)^3\ell^2 \Big)  \Bigg\} + \frac{2(m+2) \mu^2_1(d_1+6)^3\ell^2}{S_{2,x}} +  \frac{2(m+2) \mu^2_2(d_2+6)^3\ell^2}{S_{2,y}} \nonumber\\
	&\overset{(i)}{\leq} \left[ 1 + \frac{6\ell^2\beta^2}{1-b}\left(\frac{d_1+4}{S_{2,x}} + \frac{d_2+4}{S_{2,y}} \right) \right]\Delta^\prime_{t-1} + \frac{6\ell^2\beta^2}{1-b}\left(\frac{d_1+4}{S_{2,x}} + \frac{d_2+4}{S_{2,y}} \right)\delta^\prime_{t-1} \nonumber\\
	&\quad + 2\ell^2\alpha^2\left(\frac{d_1+4}{S_{2,x}} + \frac{d_2+4}{S_{2,y}} \right) \left( 1 + \frac{9\ell^2\beta^2}{1-b}\right)\mE[\ltwo{v_{t-1}}^2] + \pi_\Delta(d_1,d_2,\mu_1,\mu_2),
	\end{flalign}
	where $(i)$ follows from the definition of $\pi_\Delta$.
\end{proof}
\begin{lemma}\label{lemma20}
	Suppose Assumptions \ref{ass2}-\ref{ass3} hold. Let $S_{2,x}\geq 2d_1+8$ and $S_{2,y}\geq 2d_1+8$. Then, we have
	\begin{flalign*}
	\delta^\prime_t &\leq \left(\frac{4}{\beta\mu(m+1)} + \frac{3\ell\beta}{2-\ell\beta} \right)\delta^\prime_{t-1} + \frac{2+2\ell\beta}{2-\ell\beta}\Delta^\prime_{t-1} + \left( \frac{4\ell^2\alpha^2}{\beta\mu(m+1)} + 2\ell^2\alpha^2 + \frac{9\ell^3\beta\alpha^2}{2-\ell\beta} \right)\mE[\ltwo{v_{t-1}}^2]\nonumber\\
	&\quad + \pi_\delta(d_1,d_2,\mu_1,\mu_2),
	\end{flalign*}
	where
	\begin{flalign*}
	\pi_\delta(d_1,d_2,\mu_1,\mu_2) = \frac{2\ell^2(2+2\ell\beta)}{2-\ell\beta}\left( \frac{\mu^2_1(d_1+6)^3}{S_{2,x}} + \frac{\mu^2_2(d_2+6)^3}{S_{2,y}} \right) + \frac{2}{\beta(m+1)}\left(\frac{\mu_2^2}{4\mu}\ell^2(d_2+3)^3  + \mu_2^2\ell d_2 \right).
	\end{flalign*}
	Furthermore, if we let $\beta=\frac{2}{13\ell}$, $m=104\kappa-1$, then we have
	\begin{flalign*}
	\pi_\delta(d_1,d_2,\mu_1,\mu_2) = \frac{5}{2}\mu^2_1\ell^2(d_1+6)^3 + 3 \mu^2_2\ell^2(d_2+6)^3 + \frac{1}{8}\mu_2^2\mu\ell d_2.
	\end{flalign*}
\end{lemma}
\begin{proof}
	Using the result in \Cref{lemma17}, and recalling the definition in \Cref{sc: resultofzoconcavemaximizer} that $\nabla g_{t,\mu_2}(\ty_{t,\tm_t})=\nabla_yf(x_t,y_t)$ and $\nabla g_{t,\mu_2}(\ty_{t,0})=\nabla_yf_{\mu_2}(x_{t+1},y_t)$, we have
	\begin{flalign}
	\delta^\prime_{t+1} &\leq \frac{2}{\beta\mu(m+1)}\mE[\ltwo{\nabla_yf_{\mu_2}(x_{t+1},y_t)}^2] + \mE[\ltwo{\nabla g_{t,\mu_2}(\ty_{t,0}) - \tu_{t,0} }^2] + \frac{\ell\beta}{2-\ell\beta}\mE[\ltwo{\tu_{t,0}}^2] \nonumber\\
	&\quad  + \frac{2}{\beta(m+1)}\left(\frac{\mu_2^2}{4\mu}\ell^2(d_2+3)^3  + \mu_2^2\ell d_2 \right)\nonumber\\
	&\leq \frac{2}{\beta\mu(m+1)}\mE[\ltwo{\nabla_yf_{\mu_2}(x_{t+1},y_t)}^2] + \wtDelta^\prime_{t,0} + \frac{\ell\beta}{2-\ell\beta}\mE[\ltwo{\tu_{t,0}}^2] \nonumber\\
	&\quad  + \frac{2}{\beta(m+1)}\left(\frac{\mu_2^2}{4\mu}\ell^2(d_2+3)^3  + \mu_2^2\ell d_2 \right) \nonumber\\
	&\leq \frac{4}{\beta\mu(m+1)}\mE[\ltwo{\nabla_yf_{\mu_2}(x_{t+1},y_t) - \nabla_yf_{\mu_2}(x_t,y_t)}^2] + \frac{4}{\beta\mu(m+1)}\mE[\ltwo{\nabla_yf_{\mu_2}(x_t,y_t)}^2] \nonumber\\
	&\quad + \wtDelta^\prime_{t,0} + \frac{\ell\beta}{2-\ell\beta}\mE[\ltwo{\tu_{t,0}}^2] + \frac{2}{\beta(m+1)}\left(\frac{\mu_2^2}{4\mu}\ell^2(d_2+3)^3  + \mu_2^2\ell d_2 \right) \nonumber\\
	&\leq \frac{4\ell^2\alpha^2}{\beta\mu(m+1)}\mE[\ltwo{v_t}^2] + \frac{4}{\beta\mu(m+1)}\delta^\prime_t + \wtDelta^\prime_{t,0} + \frac{\ell\beta}{2-\ell\beta}\mE[\ltwo{\tu_{t,0}}^2] \nonumber\\
	&\quad + \frac{2}{\beta(m+1)}\left(\frac{\mu_2^2}{4\mu}\ell^2(d_2+3)^3  + \mu_2^2\ell d_2 \right) \nonumber\\
	&\overset{(i)}{\leq }\frac{4\ell^2\alpha^2}{\beta\mu(m+1)}\mE[\ltwo{v_t}^2] + \frac{4}{\beta\mu(m+1)}\delta^\prime_t \nonumber\\
	&\quad +  \Delta^\prime_t + 2\ell^2\alpha^2\mE[\ltwo{v_t}^2] + 2\ell^2\left( \frac{\mu^2_1(d_1+6)^3}{S_{2,x}} + \frac{\mu^2_2(d_2+6)^3}{S_{2,y}} \right)  \nonumber\\
	&\quad + \frac{\ell\beta}{2-\ell\beta}\left[ 3\Delta^\prime_t + 3\delta^\prime_t + 9\ell^2\alpha^2\mE[\ltwo{v_t}^2] + 6\ell^2\left( \frac{\mu^2_1(d_1+6)^3}{S_{2,x}} + \frac{\mu^2_2(d_2+6)^3}{S_{2,y}} \right) \right] \nonumber\\
	&\quad  + \frac{2}{\beta(m+1)}\left(\frac{\mu_2^2}{4\mu}\ell^2(d_2+3)^3  + \mu_2^2\ell d_2 \right) \nonumber\\
	&=\left(\frac{4}{\beta\mu(m+1)} + \frac{3\ell\beta}{2-\ell\beta} \right)\delta^\prime_t + \frac{2+2\ell\beta}{2-\ell\beta}\Delta^\prime_t + \left( \frac{4\ell^2\alpha^2}{\beta\mu(m+1)} + 2\ell^2\alpha^2 + \frac{9\ell^3\beta\alpha^2}{2-\ell\beta} \right)\mE[\ltwo{v_t}^2]\nonumber\\
	&\quad + \frac{2\ell^2(2+2\ell\beta)}{2-\ell\beta}\left( \frac{\mu^2_1(d_1+6)^3}{S_{2,x}} + \frac{\mu^2_2(d_2+6)^3}{S_{2,y}} \right) + \frac{2}{\beta(m+1)}\left(\frac{\mu_2^2}{4\mu}\ell^2(d_2+3)^3  + \mu_2^2\ell d_2 \right),\nonumber\\
	&\leq \left(\frac{4}{\beta\mu(m+1)} + \frac{3\ell\beta}{2-\ell\beta} \right)\delta^\prime_t + \frac{2+2\ell\beta}{2-\ell\beta}\Delta^\prime_t + \left( \frac{4\ell^2\alpha^2}{\beta\mu(m+1)} + 2\ell^2\alpha^2 + \frac{9\ell^3\beta\alpha^2}{2-\ell\beta} \right)\mE[\ltwo{v_t}^2]\nonumber\\
	&\quad + \pi_\delta(d_1,d_2,\mu_1,\mu_2),\label{eq: 48}
	\end{flalign}
	where $(i)$ follows from \cref{eq: 38} and \cref{eq: 43}, and from the fact that $S_{2,x}\geq 2d_1+8$ and $S_{2,y}\geq 2d_2+8$. The proof is completed by shifting the index in \cref{eq: 48} from $t$ to $t-1$.
\end{proof}

\subsection{Proof of \Cref{thm2}}
We first restate \Cref{thm2} as follows to include the specifics of the parameters.
\begin{theorem}[Restate of \Cref{thm2} with parameter specifics]\label{thm4}
	Let Assumptions \ref{ass1},\ref{ass2},\ref{ass5},and \ref{ass3} hold and apply ZO-VRGDA in \Cref{al:zosredaboost} to solve the problem in \cref{eq: 1} with the following parameters:
	\begin{flalign*}
	&\zeta=\frac{1}{\kappa}, \quad \alpha=\frac{1}{24(\kappa+1)\ell},\quad\beta=\frac{2}{13\ell},\quad q=\frac{2800\kappa}{13\epsilon(\kappa+1)},\\
	&m=104\kappa-1,\quad S_{2,x}=\frac{5600(d_1+4)\kappa}{\epsilon},\quad S_{2,y}=\frac{5600(d_2+4)\kappa}{\epsilon},\\
	&S_1=\frac{40320\sigma^2\kappa^2}{\epsilon^2},\quad T=\max\{ 1728(\kappa+1)\ell\frac{\Phi(x_0)-\Phi^*}{\epsilon^2}, \frac{810\kappa}{\epsilon^2} \},\\
	&\delta=\frac{\epsilon}{71\kappa\ell\sqrt{d_1+d_2}},\quad \mu_1= \frac{\epsilon}{71\kappa^{2.5}\ell(d_1+6)^{1.5}}, \quad \mu_2= \frac{\epsilon}{71\kappa^{2.5}\ell(d_2+6)^{1.5}}.
	\end{flalign*}
	\Cref{al:zosredaboost} outputs $\hat{x}$ satisfies that
	\begin{flalign*}
	\mE[\ltwo{\nabla\Phi(\hat{x})}]\leq \epsilon
	\end{flalign*}
	with at most $\mco((d_1+d_2)\kappa^{3} \epsilon^{-3})$ function queries.
\end{theorem}
\begin{proof}
By \Cref{lemma1}, the objective function $\Phi$ is $L$-smooth, which implies that
\begin{flalign}
\Phi(x_{t+1})&\leq \Phi(x_t)-\alpha\langle \nabla_x \Phi(x_t),v_t\rangle + \frac{L\alpha^2}{2}\ltwo{v_t}^2\nonumber\\
&=\Phi(x_t)-\alpha\langle \nabla_x \Phi(x_t)-v_t,v_t\rangle - \alpha\ltwo{v_t}^2 + \frac{L\alpha^2}{2}\ltwo{v_t}^2\nonumber\\
&\overset{(i)}{\leq}\Phi(x_t) + \frac{\alpha}{2}  \ltwo{\nabla_x \Phi(x_t)-v_t}^2 + \frac{\alpha}{2}\ltwo{v_t}^2  - \alpha\ltwo{v_t}^2 + \frac{L\alpha^2}{2}\ltwo{v_t}^2\nonumber\\
&\leq \Phi(x_t) + \frac{\alpha}{2}\ltwo{\nabla_x \Phi(x_t)-v_t}^2 - \left(\frac{\alpha}{2}-\frac{L\alpha^2}{2}\right)\ltwo{v_t}^2\nonumber\\
&\leq \Phi(x_t) + \alpha\ltwo{\nabla_x \Phi(x_t)-\nabla_xf(x_t,y_t)}^2 + \alpha\ltwo{\nabla_xf(x_t,y_t)-v_t}^2 - \left(\frac{\alpha}{2}-\frac{L\alpha^2}{2}\right)\ltwo{v_t}^2\nonumber\\
&\overset{(ii)}{\leq} \Phi(x_t) + \alpha\kappa^2\ltwo{\nabla_yf(x_t,y_t)}^2 + \alpha\ltwo{\nabla_xf(x_t,y_t)-v_t}^2 - \left(\frac{\alpha}{2}-\frac{L\alpha^2}{2}\right)\ltwo{v_t}^2,\label{eq: 16}
\end{flalign}
where $(i)$ follows from the fact that $(-1)\langle \nabla_x \Phi(x_t)-v_t,v_t\rangle\leq \frac{1}{2}\ltwo{\nabla_x \Phi(x_t)-v_t}^2+\frac{1}{2}\ltwo{v_t}^2$, and $(ii)$ follows from the fact that
\begin{flalign*}
\ltwo{\nabla_x \Phi(x_t)-\nabla_xf(x_t,y_t)}^2&=\ltwo{\nabla_x f(x_t,y^*(x_t))-\nabla_xf(x_t,y_t)}^2\leq \ell^2 \ltwo{y^*(x_t)-y_t}^2\\
&\overset{\cref{eq: 15}}{\leq} \frac{\ell^2}{\mu^2}\ltwo{\nabla_y f(x_t, y^*(x_t)) - \nabla_y f(x_t, y_t) }^2 = \kappa^2\ltwo{\nabla_y f(x_t, y_t)}^2.
\end{flalign*}
Taking expectation on both sides of \cref{eq: 16} yields
\begin{flalign}
\mE[\Phi(x_{t+1})]&\leq \mE[\Phi(x_t)] + \alpha\kappa^2\mE[\ltwo{\nabla_yf(x_t,y_t)}^2] + \alpha\mE[\ltwo{\nabla_xf(x_t,y_t)-v_t}^2] - \left(\frac{\alpha}{2}-\frac{L\alpha^2}{2}\right)\mE[\ltwo{v_t}^2]\nonumber\\
&\leq \mE[\Phi(x_t)] + \alpha\kappa^2\delta_t + \alpha\Delta_t - \left(\frac{\alpha}{2}-\frac{L\alpha^2}{2}\right)\mE[\ltwo{v_t}^2].\label{eq: 17}
\end{flalign}
Using the property in \Cref{lemma18}, we obtain the following
\begin{flalign}
\mE[\Phi(x_{t+1})] &\leq \mE[\Phi(x_t)] + \alpha\kappa^2\delta_t + \alpha\Delta_t - \left(\frac{\alpha}{2}-\frac{L\alpha^2}{2}\right)\mE[\ltwo{v_t}^2] \nonumber\\
&\overset{(i)}{\leq} \mE[\Phi(x_t)] + 2\alpha\kappa^2\delta^\prime_t + 2\alpha\Delta^\prime_t - \left(\frac{\alpha}{2}-\frac{L\alpha^2}{2}\right)\mE[\ltwo{v_t}^2] \nonumber\\
&\quad + \frac{\mu_2\alpha(\kappa^2+1)}{2}\ell^2(d_2+3)^3 + \frac{\mu_1\alpha}{2}\ell^2(d_1+3)^3. \label{eq: 50}
\end{flalign}
Rearranging \cref{eq: 50} and taking the summation over $t=\{0,1,\cdots,T-1\}$ yield
\begin{flalign}
	\left(\frac{\alpha}{2}-\frac{L\alpha^2}{2}\right)\sum_{t=0}^{T-1}\mE[\ltwo{v_t}^2] &\leq \Phi(x_0) - \mE[\Phi(x_T)] + 2\alpha\kappa^2\sum_{t=0}^{T-1}\delta^\prime_t + 2\alpha\sum_{t=0}^{T-1}\Delta^\prime_t \nonumber\\
	&\quad + \alpha T\pi(d_1,d_2,\mu_1,\mu_2).\label{eq: 51}
\end{flalign}
Note that in \cref{eq: 51} we define
\begin{flalign}
	\pi(d_1,d_2,\mu_1,\mu_2)=\frac{\mu^2_2(\kappa^2+1)}{2}\ell^2(d_2+3)^3 + \frac{\mu^2_1}{2}\ell^2(d_1+3)^3.\label{eq: 71}
\end{flalign} 

Then we proceed to prove \Cref{thm2}/\Cref{thm4} in the following three steps.

{\em \textbf{Step 1.} We establish the induction relationships for the tracking error and gradient estimation error with respect to the Gaussian smoothed function upon one outer-loop update for ZO-VRGDA. Namely, we develop the relationship between $\delta^\prime_{t}$ and $\delta^\prime_{t-1}$ as well as that between $\Delta^\prime_{t}$ and $\Delta^\prime_{t-1}$, which are captured in \Cref{lemma7} and \Cref{lemma20}.}

{\em \textbf{Step 2.} Based on Step 1, we provide the bounds on the inter-related accumulative errors $\sum_{t=0}^{T-1}\Delta^\prime_t$ and $\sum_{t=0}^{T-1}\delta^\prime_t$ over the entire execution of the algorithm.}

We first consider $\sum_{t=0}^{T-1}\Delta^\prime_t$, for any $(n_T-1)q\leq t^\prime<T-1$. Applying the inequality in \Cref{lemma7} recursively, we obtain the following bound
\begin{flalign}
\Delta^\prime_{t} &\leq \left[ 1 + \frac{6\ell^2\beta^2}{1-b} \left(\frac{d_1+4}{S_{2,x}} + \frac{d_2+4}{S_{2,y}} \right) \right]\Delta^\prime_{t-1} + \frac{6\ell^2\beta^2}{1-b}\left(\frac{d_1+4}{S_{2,x}} + \frac{d_2+4}{S_{2,y}}\right)\delta^\prime_{t-1} \nonumber\\
&\quad + 2\ell^2\alpha^2\left(\frac{d_1+4}{S_{2,x}} + \frac{d_2+4}{S_{2,y}}\right) \left( 1 + \frac{9\ell^2\beta^2}{1-b}\right)\mE[\ltwo{v_{T-2}}^2] + \pi_\Delta(d_1,d_2,\mu_1,\mu_2,S_2)\nonumber\\
&\leq \left[ 1 + \frac{6\ell^2\beta^2}{1-b} \left(\frac{d_1+4}{S_{2,x}} + \frac{d_2+4}{S_{2,y}} \right) \right]^{t-t^\prime}\Delta^\prime_{t^\prime} \nonumber\\
&\quad + \frac{6\ell^2\beta^2}{1-b}\left(\frac{d_1+4}{S_{2,x}} + \frac{d_2+4}{S_{2,y}}\right) \sum_{p=t^\prime}^{t-1} \left[ 1 + \frac{6\ell^2\beta^2}{1-b} \left(\frac{d_1+4}{S_{2,x}} + \frac{d_2+4}{S_{2,y}} \right) \right]^{p-t^\prime} \delta^\prime_p\nonumber\\
&\quad + 2\ell^2\alpha^2\left(\frac{d_1+4}{S_{2,x}} + \frac{d_2+4}{S_{2,y}}\right) \left( 1 + \frac{9\ell^2\beta^2}{1-b}\right) \sum_{p=t^\prime}^{t-1} \left[ 1 + \frac{6\ell^2\beta^2}{1-b} \left(\frac{d_1+4}{S_{2,x}} + \frac{d_2+4}{S_{2,y}} \right) \right]^{p-t^\prime} \mE[\ltwo{v_t}^2] \nonumber\\
&\quad + \pi_\Delta(d_1,d_2,\mu_1,\mu_2,S_2)\sum_{p=t^\prime}^{t-1} \left[ 1 + \frac{6\ell^2\beta^2}{1-b} \left(\frac{d_1+4}{S_{2,x}} + \frac{d_2+4}{S_{2,y}} \right) \right]^{p-t^\prime} \nonumber\\
&\overset{(i)}{\leq} 2\Delta^\prime_{t^\prime} +\frac{6\ell^2\beta^2}{1-b}\left(\frac{d_1+4}{S_{2,x}} + \frac{d_2+4}{S_{2,y}}\right) \sum_{p=t^\prime}^{t-1} \delta^\prime_t \nonumber\\
&\quad + 2\ell^2\alpha^2\left(\frac{d_1+4}{S_{2,x}} + \frac{d_2+4}{S_{2,y}}\right) \left( 1 + \frac{9\ell^2\beta^2}{1-b}\right) \sum_{p=t^\prime}^{t-1} \mE[\ltwo{v_t}^2] \nonumber\\
&\quad + 2\pi_\Delta(d_1,d_2,\mu_1,\mu_2,S_2),\label{eq: 52}
\end{flalign}
where $(i)$ follows from the fact that
\begin{flalign*}
&\left[ 1 + \frac{6\ell^2\beta^2}{1-b} \left(\frac{d_1+4}{S_{2,x}} + \frac{d_2+4}{S_{2,y}} \right) \right]^{p-t^\prime} \nonumber\\
&\leq \left[ 1 + \frac{6\ell^2\beta^2}{1-b} \left(\frac{d_1+4}{S_{2,x}} + \frac{d_2+4}{S_{2,y}} \right) \right]^q\\
& \overset{(ii)}{\leq} 1 + \frac{ \frac{6q\ell^2\beta^2}{1-b} \left(\frac{d_1+4}{S_{2,x}} + \frac{d_2+4}{S_{2,y}} \right) }{ 1 - \frac{6(q-1)\ell^2\beta^2}{1-b} \left(\frac{d_1+4}{S_{2,x}} + \frac{d_2+4}{S_{2,y}} \right) }\overset{(iii)}{\leq} 2,
\end{flalign*}
where $(ii)$ follows from the Bernoulli's inequality \cite{li2013some}
\begin{flalign}
(1+c)^r\leq 1 + \frac{rc}{1-(r-1)c}\quad \text{for\quad} c\in \Big[-1,\frac{1}{r-1} \Big), r>1,\label{eq: 53}
\end{flalign}
and $(iii)$ follows from the fact that $q= (1-b)\left(\frac{d_1+4}{S_{2,x}} + \frac{d_2+4}{S_{2,y}} \right)^{-1}$, $\beta=\frac{2}{13\ell}$, $\left(\frac{d_1+4}{S_{2,x}} + \frac{d_2+4}{S_{2,y}} \right)<1$, and $b=1-\frac{\beta\mu\ell}{2(\mu+\ell)}$, which further implies that
\begin{flalign*}
	&\frac{ \frac{6q\ell^2\beta^2}{1-b} \left(\frac{d_1+4}{S_{2,x}} + \frac{d_2+4}{S_{2,y}} \right) }{ 1 - \frac{6(q-1)\ell^2\beta^2}{1-b} \left(\frac{d_1+4}{S_{2,x}} + \frac{d_2+4}{S_{2,y}} \right) } \leq \frac{ \frac{6q\ell^2\beta^2}{1-b} \left(\frac{d_1+4}{S_{2,x}} + \frac{d_2+4}{S_{2,y}} \right) }{ 1 - \frac{6q\ell^2\beta^2}{1-b} \left(\frac{d_1+4}{S_{2,x}} + \frac{d_2+4}{S_{2,y}} \right) } = \frac{6\ell^2\beta^2}{1-6\ell^2\beta^2}<1.
\end{flalign*}
Letting $t^\prime=(n_T-1)q$ and taking summation of \cref{eq: 52} over $t=\{(n_T-1)q,\cdots,T-1\}$ yield
\begin{flalign}
\sum_{t=(n_T-1)q}^{T-1}\Delta^\prime_t&\leq 2(T-(n_T-1)q)\Delta^\prime_{(n_T-1)q} + \frac{6\ell^2\beta^2}{1-b}\left(\frac{d_1+4}{S_{2,x}} + \frac{d_2+4}{S_{2,y}}\right)\sum_{t=(n_T-1)q}^{T-1}\sum_{p=(n_T-1)q}^{t-1} \delta^\prime_p\nonumber\\
&\quad + 2\ell^2\alpha^2\left(\frac{d_1+4}{S_{2,x}} + \frac{d_2+4}{S_{2,y}}\right) \left( 1 + \frac{9\ell^2\beta^2}{1-b}\right) \sum_{t=(n_T-1)q}^{T-1}\sum_{p=(n_T-1)q}^{t-1} \mE[\ltwo{v_p}^2]\nonumber\\
&\quad + 2(T-(n_T-1)q) \pi_\Delta(d_1,d_2,\mu_1,\mu_2,S_2)\nonumber\\
&\overset{(i)}{\leq} 2(T-(n_T-1)q)\epsilon(S_1,\delta) + \frac{6q\ell^2\beta^2}{1-b}\left(\frac{d_1+4}{S_{2,x}} + \frac{d_2+4}{S_{2,y}}\right)\sum_{t=(n_T-1)q}^{T-2} \delta^\prime_t \nonumber\\
&\quad + 2q\ell^2\alpha^2\left(\frac{d_1+4}{S_{2,x}} + \frac{d_2+4}{S_{2,y}}\right) \left( 1 + \frac{9\ell^2\beta^2}{1-b}\right) \sum_{t=(n_T-1)q}^{T-2} \mE[\ltwo{v_t}^2]\nonumber\\
&\quad + 2(T-(n_T-1)q) \pi_\Delta(d_1,d_2,\mu_1,\mu_2)\nonumber\\
&=2(T-(n_T-1)q)\epsilon(S_1,\delta) + 6\ell^2\beta^2\sum_{t=(n_T-1)q}^{T-2} \delta^\prime_t \nonumber\\
&\quad + 2\ell^2\alpha^2(1-b) \left( 1 + \frac{9\ell^2\beta^2}{1-b}\right) \sum_{t=(n_T-1)q}^{T-2} \mE[\ltwo{v_t}^2]\nonumber\\
&\quad + 2(T-(n_T-1)q) \pi_\Delta(d_1,d_2,\mu_1,\mu_2)\nonumber\\
&\leq 2(T-(n_T-1)q)\epsilon(S_1,\delta) + 6\ell^2\beta^2\sum_{t=(n_T-1)q}^{T-2} \delta^\prime_t + 2\ell^2\alpha^2\left( 1 + 9\ell^2\beta^2 \right) \sum_{t=(n_T-1)q}^{T-2} \mE[\ltwo{v_t}^2]\nonumber\\
&\quad + 2(T-(n_T-1)q) \pi_\Delta(d_1,d_2,\mu_1,\mu_2)\nonumber\\
&\overset{(ii)}{\leq} 2(T-(n_T-1)q)\epsilon(S_1,\delta) + \frac{1}{7}\sum_{t=(n_T-1)q}^{T-2} \delta^\prime_t + 3\ell^2\alpha^2 \sum_{t=(n_T-1)q}^{T-2} \mE[\ltwo{v_t}^2]\nonumber\\
&\quad + 2(T-(n_T-1)q) \pi_\Delta(d_1,d_2,\mu_1,\mu_2), \label{eq: 54}
\end{flalign}
where $(i)$ follows from the fact that $\Delta^\prime_{(n_T-n)q}\leq \epsilon(S_1,\delta)$ for all $n\leq n_T$ (following from \Cref{lemma5}),
\begin{flalign*}
\sum_{t=(n_T-1)q}^{T-1}\sum_{p=(n_T-1)q}^{t-1} \delta^\prime_p\leq q\sum_{t=(n_T-1)q}^{T-2}\delta_t^\prime,
\end{flalign*}
and
\begin{flalign*}
\sum_{t=(n_T-1)q}^{T-1}\sum_{p=(n_T-1)q}^{t-1} \mE[\ltwo{v_t}^2]\leq q \sum_{t=(n_T-1)q}^{T-2} \mE[\ltwo{v_t}^2],
\end{flalign*}
and $(ii)$ follows because $\beta=\frac{2}{13\ell}$.
Applying steps similar to those in \cref{eq: 54} for iterations over $t=\{(n_T-n_t)q,\cdots,(n_T-n_t+1)q-1\}$ yields
\begin{flalign}
\sum_{t=(n_T-n_t)q}^{(n_T-n_t+1)q-1}\Delta^\prime_t &\leq 2q\epsilon(S_1,\delta) + \frac{1}{7} \sum_{t=(n_T-n_t)q}^{(n_T-n_t+1)q-1}\delta^\prime_t + 3\ell^2\alpha^2 \sum_{t=(n_T-n_t)q}^{(n_T-n_t+1)q-1} \mE[\ltwo{v_t}^2]\nonumber\\
&\quad + 2q \pi_\Delta(d_1,d_2,\mu_1,\mu_2). \label{eq: 55}
\end{flalign}
Taking summation of \cref{eq: 55} over $n=\{2,\cdots, n_T\}$ and combing with \cref{eq: 54} yield
\begin{flalign}
	\sum_{t=0}^{T-1}\Delta^\prime_t \leq 2T\epsilon(S_1,\delta) + \frac{1}{7} \sum_{t=0}^{T-1}\delta^\prime_t + 3\ell^2\alpha^2 \sum_{t=0}^{T-1} \mE[\ltwo{v_t}^2] + 2T \pi_\Delta(d_1,d_2,\mu_1,\mu_2).\label{eq: 56}
\end{flalign}
Then we consider the upper bound on $\sum_{t=0}^{T-1}\delta_t^\prime$. Since $m= \frac{16}{\mu\beta}-1$ and $\beta= \frac{2}{13\ell}$, \Cref{lemma20} implies
\begin{flalign}
\delta^\prime_t&\leq \frac{1}{2}\delta^\prime_{t-1} + \frac{5}{4}\Delta^\prime_{t-1} +  3\ell^2\alpha^2\mE[\ltwo{v_{t-1}}^2] + \pi_\delta(d_1,d_2,\mu_1,\mu_2).\label{eq: 57}
\end{flalign}
Applying \cref{eq: 57} recursively from $t$ to $0$ yields
\begin{flalign}
\delta^\prime_{t}\leq \frac{1}{2^t}\delta^\prime_0 + \frac{5}{4}\sum_{p=0}^{t-1}\frac{1}{2^p}\Delta^\prime_p + 3\ell^2\alpha^2\sum_{p=0}^{t-1}\frac{1}{2^p}\mE[\ltwo{v_p}^2] + \pi_\delta(d_1,d_2,\mu_1,\mu_2)\sum_{p=0}^{t-1}\frac{1}{2^p}.\label{eq: 58}
\end{flalign}
Taking the summation of \cref{eq: 58} over $t=\{0,1,\cdots,T-1\}$ yields
\begin{flalign}
\sum_{t=0}^{T-1}\delta^\prime_t&\leq \delta^\prime_0\sum_{t=0}^{T-1}\frac{1}{2^t} + \frac{5}{4}\sum_{t=0}^{T-1}\sum_{p=0}^{t-1}\frac{1}{2^p}\Delta^\prime_p + 3\ell^2\alpha^2\sum_{t=0}^{T-1}\sum_{p=0}^{t-1}\frac{1}{2^p}\mE[\ltwo{v_p}^2] \nonumber\\
&\quad + \pi_\delta(d_1,d_2,\mu_1,\mu_2) \sum_{t=0}^{T-1}\sum_{p=0}^{t-1} \frac{1}{2^p}\nonumber\\
&\leq 2\delta^\prime_0 + \frac{5}{2}\sum_{t=0}^{T-2}\Delta^\prime_t + 6\ell^2\alpha^2\sum_{t=0}^{T-2}\mE[\ltwo{v_t}^2] + 2T\pi_\delta(d_1,d_2,\mu_1,\mu_2).\label{eq: 59}
\end{flalign}

{We then decouple the bounds on $\sum_{t=0}^{T-1}\Delta^\prime_t$ and $\sum_{t=0}^{T-1}\delta^\prime_t$ in Step 2 from each other, and establish their separate relationships with the accumulative gradient estimators $\sum_{i=0}^{T-1}\mE[\ltwo{v_t}^2]$.}

Substituting \cref{eq: 59} into \cref{eq: 56} yields
\begin{flalign*}
\sum_{t=0}^{T-1}\Delta^\prime_t &\leq 2T\epsilon(S_1,\delta) + \frac{2}{7}\delta^\prime_0 + 4\alpha^2\ell^2\sum_{t=0}^{T-2} \mE[\ltwo{v_t}^2] + \frac{5}{14} \sum_{t=0}^{T-2}\Delta^\prime_t \nonumber\\
&\quad + 2T\pi_\Delta(d_1,d_2,\mu_1,\mu_2) + \frac{2}{7}T\pi_\delta(d_1,d_2,\mu_1,\mu_2),
\end{flalign*}
which implies
\begin{flalign}
\sum_{t=0}^{T-1}\Delta^\prime_t &\leq 4T\epsilon(S_1,\delta) + \frac{1}{2}\delta^\prime_0+ 7\alpha^2 \ell^2\sum_{t=0}^{T-2}\mE[\ltwo{v_t}^2] \nonumber\\
&\quad + \frac{1}{2}T\pi_\Delta(d_1,d_2,\mu_1,\mu_2) + 4T\pi_\delta(d_1,d_2,\mu_1,\mu_2). \label{eq: 60}
\end{flalign}
Substituting \cref{eq: 60} into \cref{eq: 59} yields
\begin{flalign}
	\sum_{t=0}^{T-1}\delta^\prime_t &\leq 10T\epsilon(S_1,\delta) + 4\delta^\prime_0 + 24\alpha^2\ell^2\sum_{t=0}^{T-2} \mE[\ltwo{v_t}^2] \nonumber\\
	&\quad + 10T\pi_\Delta(d_1,d_2,\mu_1,\mu_2) + 4T\pi_\delta(d_1,d_2,\mu_1,\mu_2). \label{eq: 61}
\end{flalign}

{\em \textbf{Step 3.} We bound $\sum_{i=0}^{T-1}\mE[\ltwo{v_t}^2]$, and further cancel  out the impact of $\sum_{t=0}^{T-1}\Delta^\prime_t$ and $\sum_{t=0}^{T-1}\delta^\prime_t$ by exploiting Step 2. Then, we obtain the convergence rate of $\mE[\ltwo{\nabla\Phi(\hat{x})}^2]$.}

Substituting \cref{eq: 60} and \cref{eq: 61} into \cref{eq: 51} yields
\begin{flalign}
	&\left(\frac{\alpha}{2}-\frac{L\alpha^2}{2}\right)\sum_{t=0}^{T-1}\mE[\ltwo{v_t}^2] \nonumber\\
	&\leq \Phi(x_0) - \mE[\Phi(x_T)] + (20\kappa^2 + 8)\alpha T\epsilon(S_1,\delta) + (8\kappa^2+1)\alpha\delta^\prime_0 + (48\kappa^2 + 14)\alpha^3\ell^2\sum_{t=0}^{T-1}\mE[\ltwo{v_t}^2] \nonumber\\
	&\quad + (20\kappa^2+1)\alpha T\pi_\Delta(d_1,d_2,\mu_1,\mu_2) + (8\kappa^2+8)\alpha T\pi_\delta(d_1,d_2,\mu_1,\mu_2) + \alpha T\pi(d_1,d_2,\mu_1,\mu_2) \nonumber\\
	&\overset{(i)}{\leq} \Phi(x_0) - \mE[\Phi(x_T)] + 28\kappa^2\alpha T\epsilon(S_1,\delta) + 9\kappa^2\alpha\delta^\prime_0 + 62\alpha^3L^2\sum_{t=0}^{T-1}\mE[\ltwo{v_t}^2]\nonumber\\
	&\quad + 21\kappa^2\alpha T\pi_\Delta(d_1,d_2,\mu_1,\mu_2) + 16\kappa^2\alpha T\pi_\delta(d_1,d_2,\mu_1,\mu_2) + \alpha T\pi(d_1,d_2,\mu_1,\mu_2), \label{eq: 62}
\end{flalign}
where $(i)$ follows from the fact that $L=(1+\kappa)\ell$ and $\kappa>1$. Rearranging \cref{eq: 62}, we have
\begin{flalign}
	&\left(\frac{\alpha}{2}-\frac{L\alpha^2}{2} - 62L^2\alpha^3 \right)\sum_{t=0}^{T-1}\mE[\ltwo{v_t}^2] \nonumber\\
	&\leq \Phi(x_0) - \mE[\Phi(x_T)] + 28\kappa^2\alpha T\epsilon(S_1,\delta) + 9\kappa^2\alpha\delta^\prime_0 \nonumber\\
	&\quad + 21\kappa^2\alpha T\pi_\Delta(d_1,d_2,\mu_1,\mu_2) + 16\kappa^2\alpha T\pi_\delta(d_1,d_2,\mu_1,\mu_2) + \alpha T\pi(d_1,d_2,\mu_1,\mu_2).\label{eq: 63}
\end{flalign}
Since $\alpha=\frac{1}{24L}$, we obtain
\begin{flalign}
	\frac{\alpha}{2}-\frac{L\alpha^2}{2} - 62L^2\alpha^3 = \frac{214}{13824L}\geq \frac{1}{72L}.\label{eq: 64}
\end{flalign}
Substituting \cref{eq: 64} into \cref{eq: 63} and applying Assumption \ref{ass1} yield
\begin{flalign}
	\sum_{t=0}^{T-1}\mE[\ltwo{v_t}^2]&\leq 72L(\Phi(x_0) - \Phi^*) + 84\kappa^2T\epsilon(S_1,\delta) + 27\kappa^2\delta^\prime_0 \nonumber\\
	&\quad + 63\kappa^2T\pi_\Delta(d_1,d_2,\mu_1,\mu_2) + 48\kappa^2T\pi_\delta(d_1,d_2,\mu_1,\mu_2) \nonumber\\
	&\quad + 3T\pi(d_1,d_2,\mu_1,\mu_2). \label{eq: 65}
\end{flalign}

{We then establish the convergence bound on $\mE[\ltwo{\nabla\Phi(\hat{x})}]$ based on the bounds on its estimators $\sum_{i=0}^{T-1}\mE[\ltwo{v_t}^2]$ and the two error bounds $\sum_{t=0}^{T-1}\Delta^\prime_t$, and $\sum_{t=0}^{T-1}\delta^\prime_t$. }

Observe that
\begin{flalign}
\sum_{t=0}^{T-1}\mE[\ltwo{\nabla\Phi(x_t)}^2]&\leq \sum_{t=0}^{T-1}\mE[\ltwo{\nabla\Phi(x_t)-\nabla_x f(x_t,y_t) + \nabla_x f(x_t,y_t) - v_t + v_t }^2]\nonumber\\
&\leq 3\sum_{t=0}^{T-1}\left(\mE[\ltwo{\nabla\Phi(x_t)-\nabla_x f(x_t,y_t)}^2] + \mE[\ltwo{\nabla_x f(x_t,y_t) - v_t}^2] + \mE[\ltwo{v_t}^2]\right)\nonumber\\
&\leq 3\sum_{t=0}^{T-1}\left(\kappa^2\mE[\ltwo{\nabla_y f(x_t,y_t)}^2] + \mE[\ltwo{\nabla_x f(x_t,y_t) - v_t}^2] + \mE[\ltwo{v_t}^2]\right)\nonumber\\
&\leq 3\kappa^2\sum_{t=0}^{T-1}\delta_t + 3\sum_{t=0}^{T-1}\Delta_t + 3\sum_{t=0}^{T-1}\mE[\ltwo{v_t}^2]\nonumber\\
&\overset{(i)}{\leq} 6\kappa^2\sum_{t=0}^{T-1}\delta^\prime_t + 6\sum_{t=0}^{T-1}\Delta^\prime_t + 3\sum_{t=0}^{T-1}\mE[\ltwo{v_t}^2] + 3T\pi(d_1,d_2,\mu_1,\mu_2) \label{eq: 66}
\end{flalign}
where $(i)$ follows from \Cref{lemma18}. Substituting \cref{eq: 60}, \cref{eq: 61} and \cref{eq: 65} into \cref{eq: 66} yields
\begin{flalign}
	&\sum_{t=0}^{T-1}\mE[\ltwo{\nabla\Phi(x_t)}^2] \nonumber\\
	&\leq (60\kappa^2 + 24)T\epsilon(S_1,\delta) + (24\kappa^2 + 3)\delta^\prime_0 + (60\kappa^2 + 3)T\pi_\Delta(d_1,d_2,\mu_1,\mu_2)  \nonumber\\
	&\quad + (24\kappa^2 + 24)T\pi_\delta(d_1,d_2,\mu_1,\mu_2) + (144\kappa^2\alpha^2\ell^2 + 42\alpha^2\ell^2 + 3) \sum_{t=0}^{T-1}\mE[\ltwo{v_t}^2] \nonumber\\
	&\quad + 3T\pi(d_1,d_2,\mu_1,\mu_2)\nonumber\\
	&\overset{(i)}{\leq} 84\kappa^2T\epsilon(S_1,\delta) + 27\kappa^2\delta^\prime_0 + 63\kappa^2T\pi_\Delta(d_1,d_2,\mu_1,\mu_2) + 48\kappa^2 T\pi_\delta(d_1,d_2,\mu_1,\mu_2)\nonumber\\
	&\quad + 4\sum_{t=0}^{T-1}\mE[\ltwo{v_t}^2] + 3T\pi(d_1,d_2,\mu_1,\mu_2) \nonumber\\
	&\overset{(ii)}{\leq} 288L(\Phi(x_0)-\Phi^*) + 420\kappa^2T\epsilon(S_1,\delta) + 135\kappa^2\delta^\prime_0 + 315\kappa^2T\pi_\Delta(d_1,d_2,\mu_1,\mu_2)\nonumber\\
	&\quad + 240\kappa^2T\pi_\delta(d_1,d_2,\mu_1,\mu_2) + 15T\pi(d_1,d_2,\mu_1,\mu_2).\label{eq: 67}
\end{flalign}
where $(i)$ follows from the fact that $\kappa>1$, $L=(\kappa+1)\ell$ and $\alpha=\frac{1}{24L}$, and $(ii)$ follows from \cref{eq: 65}. Recall $L=(1+\kappa)\ell$. Then, \cref{eq: 67} implies that
\begin{flalign}
	&\mE[\ltwo{\nabla\Phi(\hat{x})}^2]\nonumber\\
	&\leq 288(\kappa+1)\ell\frac{\Phi(x_0)-\Phi^*}{T} + 420\kappa^2\epsilon(S_1,\delta) + \frac{135\kappa^2\delta^\prime_0}{T}\nonumber\\
	&\quad + 315\kappa^2\pi_\Delta(d_1,d_2,\mu_1,\mu_2) + 240\kappa^2\pi_\delta(d_1,d_2,\mu_1,\mu_2) + 15\pi(d_1,d_2,\mu_1,\mu_2).\label{eq: 68}
\end{flalign}
Recalling \Cref{lemma19}, we have
\begin{flalign*}
	\epsilon(S_1,\delta) \leq \frac{(d_1+d_2)\ell^2\delta^2}{2} + \frac{4\sigma^2}{S_1} + \frac{\mu_1^2}{2}\ell^2(d_1+3)^3 + \frac{\mu_2^2}{2}\ell^2(d_2+3)^3.
\end{flalign*}
If we let $\delta^\prime_0\leq\frac{1}{\kappa}$, $T=\max\{ 1728(\kappa+1)\ell\frac{\Phi(x_0)-\Phi^*}{\epsilon^2}, \frac{810\kappa}{\epsilon^2} \}$, $S_1=\frac{40320\sigma^2\kappa^2}{\epsilon^2}$, and further let $\delta=\frac{\epsilon}{71\kappa\ell\sqrt{d_1+d_2}}$, $\mu_1= \frac{\epsilon}{71\kappa^{2.5}\ell(d_1+6)^{1.5}}$ and $\mu_2=  \frac{\epsilon}{71\kappa^{2.5}\ell(d_2+6)^{1.5}}$, according to the definition of $\epsilon(S_1,\delta)$ (\Cref{lemma19}), $\pi_\Delta(d_1,d_2,\mu_1,\mu_2)$ (\Cref{lemma7}), $\pi_\delta(d_1,d_2,\mu_1,\mu_2)$ (\Cref{lemma20}) and $\pi(d_1,d_2,\mu_1,\mu_2)$ (\cref{eq: 71}), then we have $420\kappa^2\epsilon(S_1,\delta)\leq \frac{\epsilon^2}{6}$, and
\begin{flalign*}
	315\kappa^2\pi_\Delta(d_1,d_2,\mu_1,\mu_2) + 240\kappa^2 \pi_\delta(d_1,d_2,\mu_1,\mu_2) + 15\pi(d_1,d_2,\mu_1,\mu_2)\leq \frac{\epsilon^2}{2},
\end{flalign*}
which implies
\begin{flalign*}
	\mE[\ltwo{\nabla\Phi(\hat{x})}]\leq\sqrt{\mE[\ltwo{\nabla\Phi(\hat{x})}^2]}\leq \epsilon.
\end{flalign*}
We also let $S_{2,x}=\frac{5600(d_1+4)\kappa}{\epsilon}$, $S_{2,y}=\frac{5600(d_2+4)\kappa}{\epsilon}$ and $q=\frac{2800\kappa}{13\epsilon(\kappa+1)}$. Then, the total sample complexity is given by
\begin{flalign}
&T\cdot (S_{2,x} + S_{2,y}) \cdot m + \left\lceil\frac{T}{q}\right\rceil\cdot S_1 \cdot (d_1+d_2) + T_0 \nonumber\\
&\leq \Theta\left(\frac{\kappa}{\epsilon^2} \cdot \frac{(d_1+d_2)\kappa}{\epsilon} \cdot \kappa \right) + \Theta\left( \frac{\kappa}{\epsilon} \cdot \frac{\kappa^2}{\epsilon^2} \cdot (d_1+d_2) \right) + \Theta\left( d_2 \kappa \log(\kappa) \right)\nonumber\\
&=\mco\left( \frac{(d_1+d_2)\kappa^{3}}{\epsilon^3}\right),\nonumber
\end{flalign}
which completes the proof.
\end{proof}

\section{Proof of \Cref{cor: 0thfinitesum}}\label{sc: 0stfinitesum}
In the finite-sum case, recall that
\begin{flalign*}
f(x,y)\triangleq \frac{1}{n}\sum_{i=1}^{n}F(x,y;\xi_i).
\end{flalign*}
Here we modify \Cref{al:zosredaboost} by replacing the mini-batch update used in line 6 and 7 of \Cref{al:zosredaboost} with the following update using all samples:
\begin{flalign*}
	v_t=\frac{1}{n}\sum_{i=1}^{n}\sum_{j=1}^{d_1}\frac{F(x_t+\delta e_j, y_t, \xi_i) - F(x_t-\delta e_j, y_t, \xi_i)}{2\delta} e_j,\\
	u_t=\frac{1}{n}\sum_{i=1}^{n}\sum_{j=1}^{d_2}\frac{F(x_t, y_t+\delta e_j, \xi_i) - F(x_t, y_t-\delta e_j, \xi_i)}{2\delta} e_j,
\end{flalign*}
where $e_j$ denotes the $j$-th canonical unit basis vector. In this case, if $\text{mod}(k,q)=0$, then we have
\begin{flalign}
\epsilon(S_1,\delta) \leq \frac{(d_1+d_2)\ell^2\delta^2}{2} + \frac{\mu_1^2}{2}\ell^2(d_1+3)^3 + \frac{\mu_2^2}{2}\ell^2(d_2+3)^3.\label{eq: 87}
\end{flalign}
\textbf{Case 1: $n\geq \kappa^2$ }

Substituting \cref{eq: 87} into \cref{eq: 68}, it can be checked easily that under the same parameter settings for $\delta^\prime_0$, $T$, $\delta$, $\mu_1$ and $\mu_2$ in \Cref{thm2}, we have
\begin{flalign*}
\mE[\ltwo{\nabla\Phi(\hat{x})}]\leq\sqrt{\mE[\ltwo{\nabla\Phi(\hat{x})}^2]}\leq \epsilon.
\end{flalign*}
Then, let $S_{2,x}=5600(d_1+4)\kappa\sqrt{n}$, $S_{2,y}=5600(d_2+4)\kappa\sqrt{n}$ and $q=\frac{2800\kappa\sqrt{n}}{13(\kappa+1)}$. Recalling the sample complexity result of ZO-iSARAH in the finite-sum case in \Cref{sc: zoisarah}, we have $T_0=\mathcal{O}\left( d_2(\kappa+n)\log\left( \kappa \right) \right)$. The total sample complexity is given by
\begin{flalign}
&T\cdot (S_{2,x} + S_{2,y}) \cdot m + \left\lceil\frac{T}{q}\right\rceil\cdot S_1 \cdot (d_1+d_2) + T_0 \nonumber\\
&\leq \Theta\left(\frac{\kappa}{\epsilon^2} \cdot (d_1+d_2)\sqrt{n} \cdot \kappa \right) + \Theta\left( \left\lceil\frac{\kappa^{2}}{\sqrt{n}\epsilon^2} \right\rceil \cdot n \cdot (d_1+d_2) \right) + \Theta\left( d_2(\kappa+n) \kappa \log(\kappa) \right)\nonumber\\
&=\mco\left( (d_1+d_2)(\sqrt{n}\kappa^{2}\epsilon^{-2} + n) \right) + \mco( d_2(\kappa^2+\kappa n)  \log(\kappa)).\nonumber
\end{flalign}

\textbf{Case 2: $n\leq \kappa^2$ }

In this case, we let $S_{2,x}=56(d_1+4)+420$, $S_{2,y}=56(d_2+4)+420$ and $q=1$. Then we have
\begin{flalign}
	\Delta^\prime_t\leq \epsilon_\Delta = \frac{(d_1+d_2)\ell^2\delta^2}{2} + \frac{\mu_1^2}{2}\ell^2(d_1+3)^3 + \frac{\mu_2^2}{2}\ell^2(d_2+3)^3, \quad \text{for all\quad} 0\leq t\leq T-1.\label{eq: 88}
\end{flalign}
Given the value of $S_{2,x}$ and $S_{2,y}$, it can be checked that the proofs of \Cref{lemma14} and \Cref{lemma20} still hold. Following from the steps similar to those from \cref{eq: 51} to \cref{eq: 59}, we obtain
\begin{flalign}
\sum_{t=0}^{T-1}\delta^\prime_t \leq 2\delta^\prime_0 + \frac{5}{2}T\epsilon_\Delta + 6\ell^2\alpha^2\sum_{t=0}^{T-2}\mE[\ltwo{v_t}^2] + 2T\pi_\delta(d_1,d_2,\mu_1,\mu_2).\label{eq: 89}
\end{flalign}
Substituting \cref{eq: 88} and \cref{eq: 89} into \cref{eq: 51} yields
\begin{flalign}
\left(\frac{\alpha}{2}-\frac{L\alpha^2}{2}\right)&\sum_{t=0}^{T-1}\mE[\ltwo{v_t}^2] \nonumber\\
&\overset{(i)}{\leq} \Phi(x_0) - \mE[\Phi(x_T)] + 4\alpha\kappa^2\delta^\prime_0 + 7\alpha\kappa^2T\epsilon_\Delta + 12L^2\alpha^3\sum_{t=0}^{T-2}\mE[\ltwo{v_t}^2]  \nonumber\\
&\quad + 4\alpha\kappa^2T\pi_\delta(d_1,d_2,\mu_1,\mu_2) + \alpha T\pi(d_1,d_2,\mu_1,\mu_2),\label{eq: 90}
\end{flalign}
where $(i)$ follows because $L=(1+\kappa)\ell$. Rearranging \cref{eq: 90} yields
\begin{flalign}
&\left(\frac{\alpha}{2}-\frac{L\alpha^2}{2} - 12L^2\alpha^3 \right)\sum_{t=0}^{T-1}\mE[\ltwo{v_t}^2] \nonumber\\
&\leq \Phi(x_0) - \mE[\Phi(x_T)] + 4\alpha\kappa^2\delta^\prime_0 + 7\alpha\kappa^2T\epsilon_\Delta + 4\alpha\kappa^2T\pi_\delta(d_1,d_2,\mu_1,\mu_2) + \alpha T\pi(d_1,d_2,\mu_1,\mu_2).\label{eq: 91}
\end{flalign}
Letting $\alpha=\frac{1}{8L}$, we obtain
\begin{flalign}
	\frac{\alpha}{2}-\frac{L\alpha^2}{2} - 12L^2\alpha^3=\frac{1}{32L}.\label{eq: 92}
\end{flalign}
Substituting \cref{eq: 91} into \cref{eq: 92} and applying Assumption \ref{ass1} yield
\begin{flalign}
\sum_{t=0}^{T-1}\mE[\ltwo{v_t}^2] &\leq 32L (\Phi(x_0) - \Phi^*) + 16\kappa^2\delta^\prime_0 + 28\kappa^2T\epsilon_\Delta + 16\kappa^2T\pi_\delta(d_1,d_2,\mu_1,\mu_2) \nonumber\\
&\quad + 4 T\pi(d_1,d_2,\mu_1,\mu_2).\label{eq: 93}
\end{flalign}
Substituting \cref{eq: 93} and \cref{eq: 88} into \cref{eq: 66} yields
\begin{flalign}
&\sum_{t=0}^{T-1}\mE[\ltwo{\nabla\Phi(x_t)}^2]\nonumber\\
&\leq 6\kappa^2\sum_{t=0}^{T-1}\delta^\prime_t + 6T\epsilon_\Delta + 3\sum_{t=0}^{T-1}\mE[\ltwo{v_t}^2] + 3T\pi(d_1,d_2,\mu_1,\mu_2)\nonumber\\
&\leq 12\kappa^2\delta^\prime_0 + 21\kappa^2T\epsilon_\Delta + 4\sum_{t=0}^{T-1}\mE[\ltwo{v_t}^2] + 12\kappa^2T\pi_\delta(d_1,d_2,\mu_1,\mu_2) + 3T\pi(d_1,d_2,\mu_1,\mu_2)\nonumber\\
&\leq 128L(\Phi(x_0) - \Phi^*) + 76\kappa^2\delta^\prime_0 + 133\kappa^2T\epsilon_\Delta + 76\kappa^2T\pi_\delta(d_1,d_2,\mu_1,\mu_2) + 19T\pi(d_1,d_2,\mu_1,\mu_2).\label{eq: 94}
\end{flalign}
Recall that $L=(1+\kappa)\ell$. Then, \cref{eq: 94} implies
 \begin{flalign}
 \mE[\ltwo{\nabla\Phi(\hat{x})}^2]&\leq 128(\kappa+1)\ell\frac{\Phi(x_0)-\Phi^*}{T} + 133\kappa^2\epsilon_\Delta + \frac{76\kappa^2\delta^\prime_0}{T} + 76\kappa^2\pi_\delta(d_1,d_2,\mu_1,\mu_2) \nonumber\\
 &\quad + 19\pi(d_1,d_2,\mu_1,\mu_2).\nonumber
 \end{flalign}
If we let $\delta^\prime_0\leq \frac{1}{\kappa}$, $T=\max \{ 640(\kappa+1)\ell\frac{\Phi(x_0)-\Phi^*}{\epsilon^2}, \frac{380\kappa}{\epsilon^2} \}$, and let $\mu_1$, $\mu_2$ and $\delta$ follow the same setting in \Cref{thm2}, then we have
\begin{flalign*}
\mE[\ltwo{\nabla\Phi(\hat{x})}]\leq\sqrt{\mE[\ltwo{\nabla\Phi(\hat{x})}^2]}\leq \epsilon.
\end{flalign*}
Recall the sample complexity result of ZO-iSARAH in the finite-sum case in \Cref{sc: zoisarah}. Then, we have $T_0=\mathcal{O}\left( d_2(\kappa+n)\log\left( \kappa \right) \right)$. The total sample complexity is given by
\begin{flalign}
&T\cdot (S_{2,x} + S_{2,y}) \cdot m + \left\lceil\frac{T}{q}\right\rceil\cdot S_1 \cdot (d_1+d_2) + T_0 \nonumber\\
&\leq \Theta\left(\frac{\kappa}{\epsilon^2} \cdot (d_1+d_2) \cdot \kappa \right) + \Theta\left( \left\lceil\frac{\kappa}{\epsilon^2}\right\rceil \cdot n \cdot (d_1+d_2) \right) + \Theta\left( d_2(\kappa+n) \log(\kappa) \right)\nonumber\\
&=\mco\left(  (d_1+d_2)(\kappa^2 + \kappa n)\epsilon^{-2} \right).\nonumber
\end{flalign}

\end{document}